\documentclass{article}
\usepackage{iclr2027_conference,times}

\usepackage{amsmath,amsfonts,bm}

\def\eqref#1{equation~\ref{#1}}

\def\1{\bm{1}}

\DeclareMathAlphabet{\mathsfit}{\encodingdefault}{\sfdefault}{m}{sl}
\SetMathAlphabet{\mathsfit}{bold}{\encodingdefault}{\sfdefault}{bx}{n}

\usepackage{hyperref}
\usepackage{url}
\usepackage{xspace}
\usepackage{nicefrac}
\usepackage{booktabs}
\usepackage{multirow}
\usepackage{xcolor}
\usepackage{colortbl}
\usepackage{adjustbox}
\usepackage{wrapfig}

\newcommand{\mcerr}[1]{{%
  \fontsize{6}{7}\selectfont
  \color{black!70}%
  \ensuremath{\,\pm}#1%
}}

\newcommand{\mcnum}[2]{%
  \makebox[2.3em][r]{#1}\mcerr{#2}%
}

\newcommand{\mcbest}[2]{%
  \mcnum{\textbf{#1}}{#2}%
}

\newcommand{\mcsecond}[2]{%
  \mcnum{\underline{#1}}{#2}%
}

\RequirePackage{titletoc}
\RequirePackage{placeins}
\definecolor{MemCoTOCBlue}{RGB}{0,66,130}
\titlecontents{memcosection}[1.7em]
  {\addvspace{4.5pt}\bfseries}
  {\contentslabel{1.7em}}{}
  {\hfill\contentspage}
\titlecontents{memcosubsection}[3.8em]
  {\addvspace{1pt}}
  {\contentslabel{2.5em}}{}
  {\titlerule*[0.6pc]{.}\contentspage}
\titlecontents{memcosubsubsection}[5.7em]
  {\addvspace{0.3pt}}
  {\contentslabel{3.5em}}{}
  {\titlerule*[0.6pc]{.}\contentspage}
\newcommand{\MemCoAppendixContents}{%
  \clearpage
  \startcontents[memcoappendix]%
  \begingroup
  \hypersetup{linkcolor=MemCoTOCBlue,linktoc=all}%
  \pdfbookmark[0]{Appendix Contents}{memco.appendix.contents}%
  \section*{Appendix Contents}
  \small
  \printcontents[memcoappendix]{memco}{1}{\setcounter{tocdepth}{2}}%
  \endgroup
  \clearpage
}

\usepackage{amsmath}
\usepackage{amssymb}
\usepackage{amsthm}
\theoremstyle{plain}
\newtheorem{theorem}{Theorem}[section]
\newtheorem{proposition}[theorem]{Proposition}

\newtheorem{corollary}[theorem]{Corollary}
\theoremstyle{definition}

\theoremstyle{remark}

\usepackage{array}
\usepackage{tabularx}

\providecommand{\modelname}{\texttt{MemCo}\xspace}

\hypersetup{
  pdftitle={MemCo: Memory-Centric Collaboration for Generalizing LLM Agents to Unseen Environments},
  pdfauthor={Xinting Liao, Siyan Liu, Rabab K. Ward, Holger R. Roth, Xiaoxiao Li}
}

\makeatletter
\newcommand{\MemCoArxivTitle}{%
  \begingroup
  \setlength{\parindent}{0pt}%
  \setlength{\parskip}{0pt}%
  \vspace*{-0.43in}%
  \noindent\makebox[\textwidth][c]{%
    \begin{minipage}{6.5in}
      \hspace*{0.13in}%
      \includegraphics[trim=30 20 30 15,clip,height=0.66in]{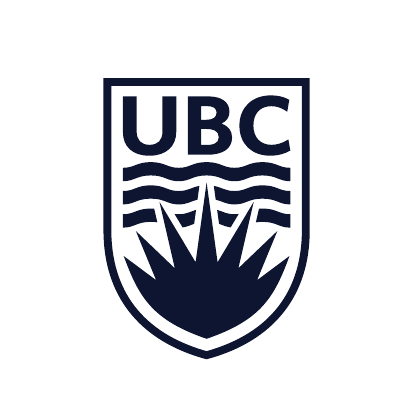}%
      \hspace{0.58in}%
      \includegraphics[height=0.7in]{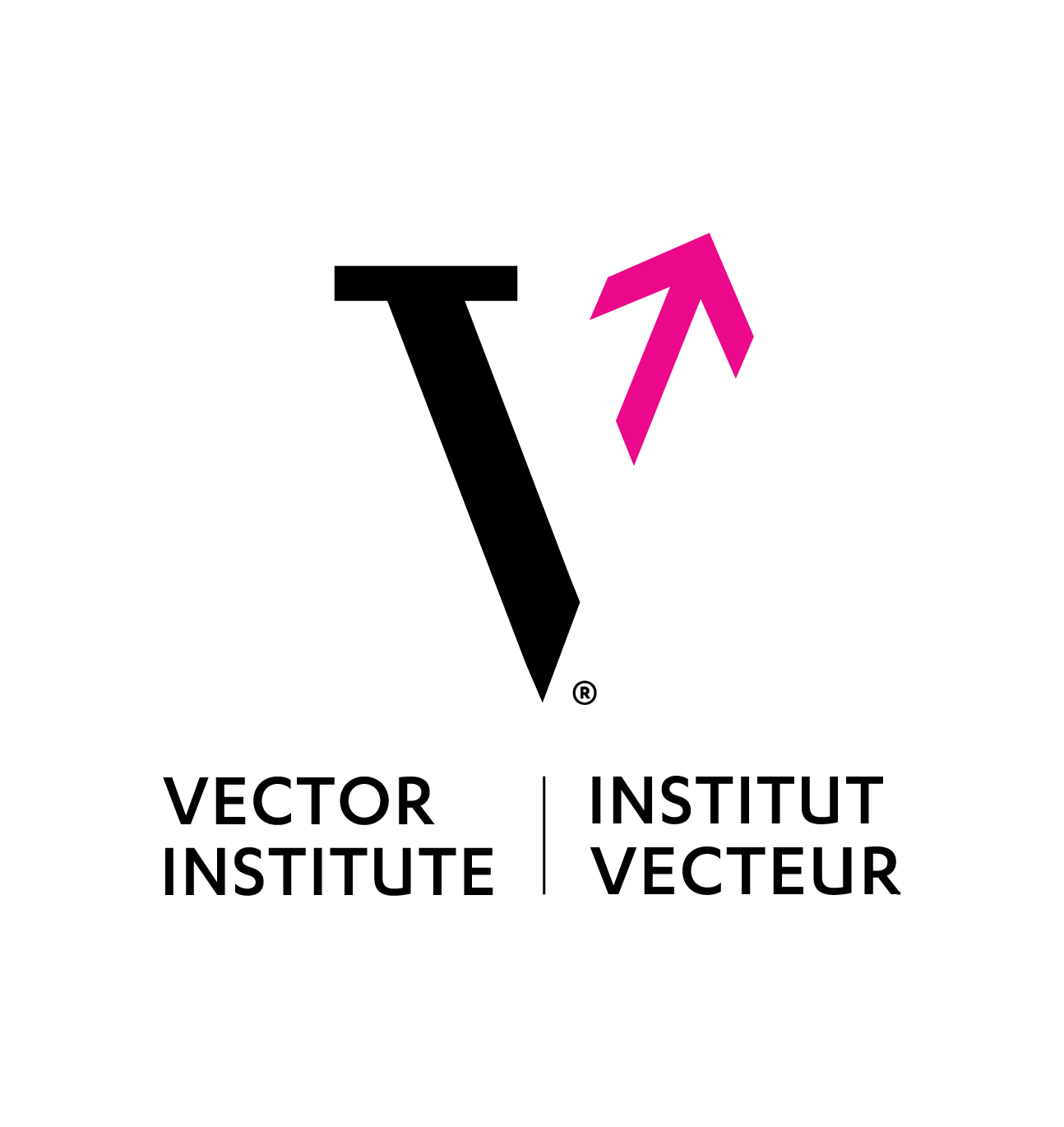}%
      \hspace{0.38in}%
      \includegraphics[height=0.67in]{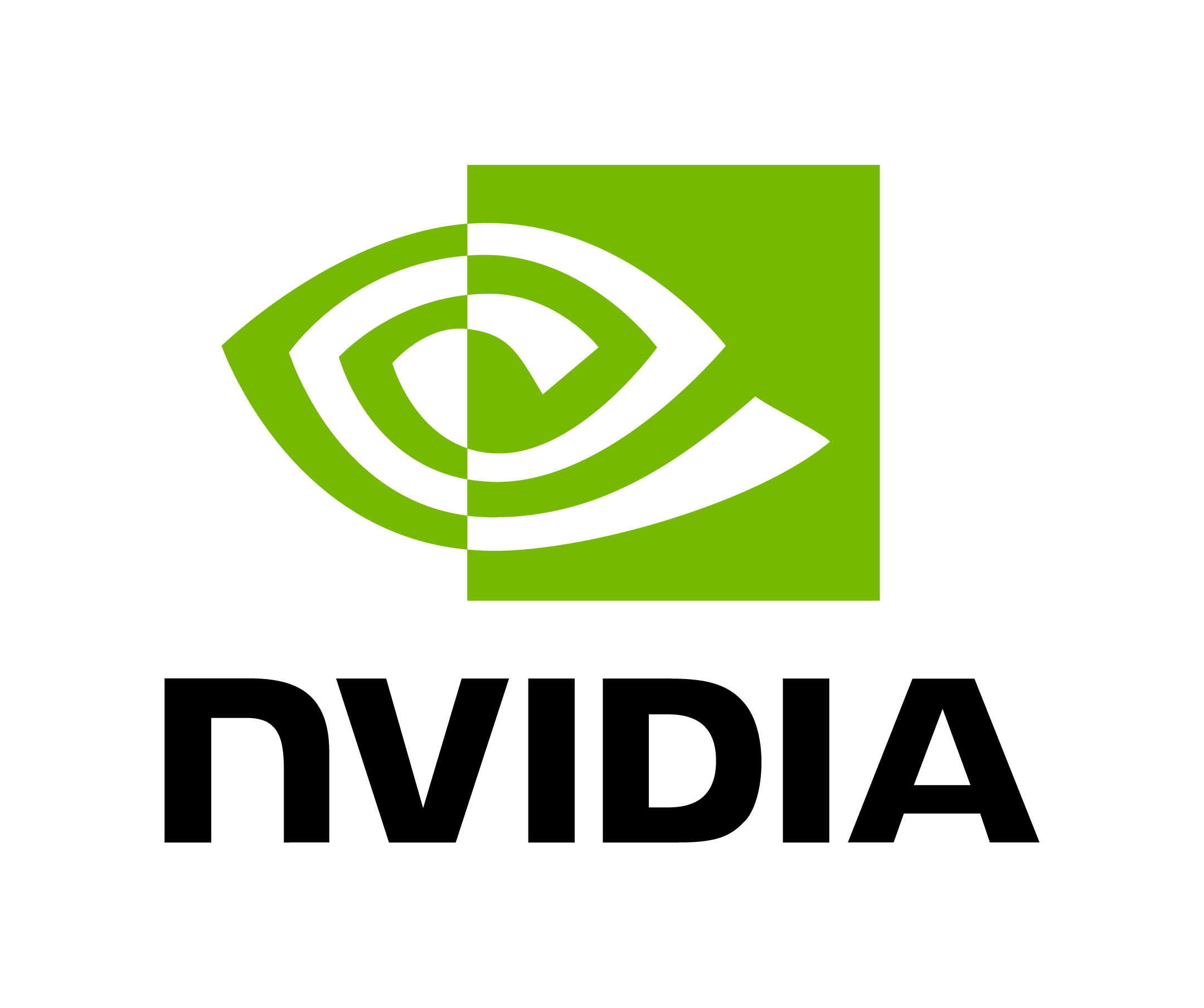}%
      \par
      \hrule height 0.35pt
    \end{minipage}%
  }\par
  \vspace{0.75in}%
  {\centering
    {\fontsize{25}{30}\selectfont
      \makebox[\textwidth][c]{%
        \textbf{\texttt{MemCo}:}\hspace{0.5em}%
        \scalebox{1.227}[1]{Memory-Centric Collaboration}%
      }\\
      \makebox[\textwidth][c]{\scalebox{1.227}[1]{for Generalizing LLM Agents to}}\\
      \makebox[\textwidth][c]{\scalebox{1.227}[1]{Unseen Environments}}\par
    }
    \vspace{0.23in}%
    \makebox[\textwidth][c]{%
      \resizebox{0.98\textwidth}{!}{\bfseries
        Xinting Liao\textsuperscript{1,2,$\dagger$},
        Siyan Liu\textsuperscript{1,$\dagger$},
        Rabab K. Ward\textsuperscript{1},
        Holger R. Roth\textsuperscript{3,*},
        Xiaoxiao Li\textsuperscript{1,2,*}%
      }%
    }\par
    \vspace{0.19in}%
    {\large
      \textsuperscript{1}The University of British Columbia, Vancouver, BC, Canada\par
      \textsuperscript{2}Vector Institute, Toronto, ON, Canada\par
      \textsuperscript{3}NVIDIA, Santa Clara, CA, USA\par
    }
    \vspace{0.23in}%
    {\normalsize \textsuperscript{$\dagger$}Equal contribution.\par}
    {\large
      \textsuperscript{*}Corresponding authors:
      \href{mailto:holger@nvidia.com}{holger@nvidia.com},
      \href{mailto:xiaoxiao.li@ece.ubc.ca}{xiaoxiao.li@ece.ubc.ca}\par
    }
  }%
  \endgroup
}
\makeatother

\renewenvironment{abstract}{%
  \vspace{0.264in}%
  \begin{center}\bfseries Abstract\end{center}%
  \vspace{-0.87em}%
  \begin{list}{}{%
    \setlength{\leftmargin}{0.375in}%
    \setlength{\rightmargin}{0.375in}%
  }%
  \item\hspace*{12.5pt}\relax
}{%
  \end{list}%
  \vspace{1ex}%
}

\iclrfinalcopy

\begin{document}
\pagestyle{plain}
\MemCoArxivTitle

\begin{abstract}
Large language model (LLM) agents increasingly operate in interactive environments, where they need to make sequential decisions through observation, action, and feedback.
Although memory can help agents reuse experience, existing
work designs memory in isolation, where collecting enough trajectories to populate it is expensive.
Existing shared-memory approaches mitigate isolated experience by pooling episodic memories across tasks and environments.
However, retrieving shared memory is challenged by the granularity, where retrieved memories can be either too specific to preserve current grounding or too coarse to support the next action.
In this work, we propose \modelname, a memory-centric collaboration framework for generalizing LLM agents to unseen interactive environments.
It maintains complementary local and global memory spaces, preserving environment-specific details locally while promoting transferable workflows induced from local trajectories to global memory.
During online interaction, \modelname routes relevant local and global memories in terms of the agent's current state and decision phase, enabling agents to reuse the experience of other agents without blindly transferring environment-specific details.
Experiments on interactive decision-making benchmarks show that \modelname improves task success and reduces redundant exploration compared with isolate-memory and shared-memory baselines.
Our codes are available at \url{https://github.com/SYannL/nvdamas}.
\end{abstract}
\clearpage

\section{Introduction}

Large language model (LLM) agents are increasingly applied to interact with
environments, where they make decisions based on the current state and execute
actions toward the task goal.
Retrieving episodic memories allows agents to reuse interaction experience at
inference time~\citep{cao2025remember,fang2025memp,liang2026learning}.
Methods such as Reflexion~\citep{shinn2023reflexion},
MemRL~\citep{zhang2026memrl}, and A-Mem~\citep{xu2025amem}
improve how individual agents organize and reuse prior experience.
However, each agent still needs to collect useful trajectories through its own
interactions, while limited experience may provide incomplete or unreliable
guidance~\citep{kwon2025embodied}.

When multiple agents operate independently in different environments,
experience acquired by one agent provides a natural reference for others.
As illustrated in Fig.~\ref{fig:teaser}(a), agents can preserve experience
from their own environments while sharing reusable experience that may support
other agents, including those executing tasks in unseen environments, i.e.,
environments they have not previously interacted with.
This motivates us to ask:
\textit{how can one agent's experience help other agents and vice versa?}
Existing methods such as G-Memory~\citep{zhang2025g},
PlugMem~\citep{yang2026plugmem}, and MCMA~\citep{liang2026learning}
organize or transfer interaction experience across tasks or agents,
demonstrating the potential of sharing experience.
As shown in Fig.~\ref{fig:teaser}(b,c), G-Memory improves over
in-context learning (ICL) and agents with isolated memory, yet still leaves
room for improving task success and interaction efficiency, particularly on
held-out layouts.
Yet shared experience is not inherently useful across environments or
decision contexts.
This raises a crucial challenge:
\textit{how can shared experience be organized into useful memory and
effectively used to support another agent's current decision?}

To address this challenge, we propose \modelname{}, a
\texttt{Mem}ory-centric \texttt{Co}llaboration framework with complementary
local and global memory spaces.
As shown in Fig.~\ref{fig:main}, \modelname{} contains three main modules.
In each agent, \textit{local memory construction} converts interaction
trajectories into a local memory graph, preserving environment-specific
context while extracting experience patterns, actionable knowledge, and
failure reflections.
To form useful shared memory, \textit{progressive global memory collaboration}
abstracts selected local records while retaining their applicability conditions,
aggregates evidence across agents, and promotes sufficiently supported records
to global memory.
To translate memory into effective decision guidance,
\textit{query-guided memory support} retrieves local and global memories
according to the current decision context, filters incompatible guidance, and
grounds shared experience in the current environment before composing memory
support for the next decision.
Thus, \modelname{} enables agents to benefit from experience accumulated by
others without blindly transferring source-environment-specific information,
while preserving the local grounding required for task execution in unseen
environments.

\begin{figure*}[t]
    \centering
    \includegraphics[width=0.9\textwidth]
    {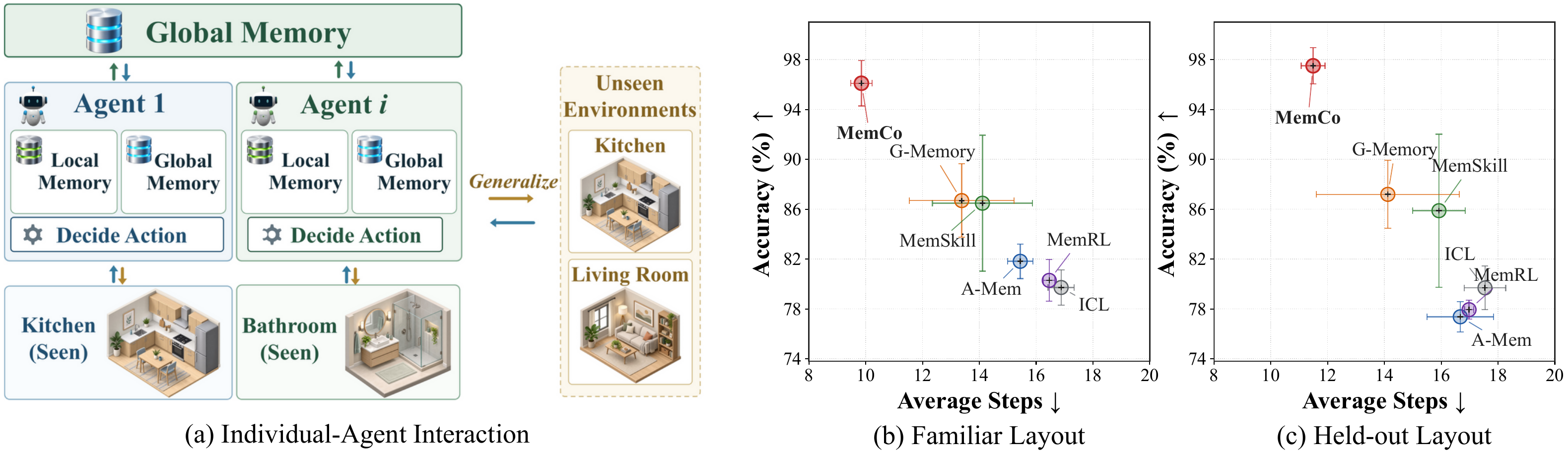}
    \vspace{-8pt}
    \caption{
    (a) Memory-centric collaboration among independently acting agents:
    each agent preserves environment-specific local experience while
    contributing reusable experience to global memory for use across environments.
    (b) and (c) ALFWorld performance on familiar and held-out layouts,
    respectively, with Qwen3-32B.
    Markers and error bars denote the mean and standard deviation.
    }
    \label{fig:teaser}
    \vspace{-8pt}
\end{figure*}

Our main contributions are:
(1) We introduce memory-centric collaboration for agents interacting with environments, where experience accumulated across different environments can be
shared to reduce repeated exploration and support task execution in unseen
environments. 
(2) We develop \modelname{} with complementary local and global memory,
turning shared experience into useful memory through evidence-aware global
promotion and into decision guidance related to execution state.
(3) Experiments across four benchmarks
(ALFWorld, PDDL, FEVER, and ScienceWorld) and multiple LLM backbones show
improved task performance with generally fewer interaction steps.

\section{Related Work}

\noindent\textbf{Individual Experience and Agent Memory.}
Early agent paradigms such as ReAct~\citep{yao2022react} rely on the current
interaction context without persistent experience across trials.
Long-term memory systems subsequently support persistent storage and retrieval
across interactions
~\citep{park2023generative,packer2023memgpt,zhong2024memorybank,
chhikara2025mem0,wang2025mirix,kwon2025embodied}.
For interactive agents, prior work further turns trajectories into reflections,
skills, procedural knowledge, or adaptive memory
~\citep{shinn2023reflexion,wang2023voyager,zhu2023ghost,ding2026skillgen,
fang2025memp,cao2025remember,wu2025meta,xu2025amem,nan2025nemori,
zhang2026memrl,zhang2026memskill}.
These methods mainly improve how experience accumulated within an individual
agent or memory system is represented and reused.

\noindent\textbf{Agent Collaboration through Shared Experience.}
Beyond communication, agent selection, or execution-time orchestration
~\citep{wu2024autogen,jia2025agentstore}, agents can collaborate by reusing
experience acquired in prior interactions.
G-Memory~\citep{zhang2025g} organizes cross-trial collaboration memory within
multi-agent systems, while MATM~\citep{kim2026multiagenttransactivememory}
supports population-level sharing of producer-agent trajectories.
Agent KB~\citep{tang2025agentkbleveragingcrossdomain} enables structured
experience reuse across heterogeneous agent frameworks, whereas
MemCollab~\citep{chang2026memcollabcrossmodelmemorycollaboration} and
FedAgentKE~\citep{li2026fedagentkefederatedsemanticknowledge} construct
shared knowledge across different models or agent frameworks.
Collectively, these works demonstrate the value of sharing agent experience.
For agents interacting with different environments, however, shared experience
is not uniformly useful: its value can depend on the receiving environment and
current decision.
\modelname{} addresses this setting by preserving local grounding while sharing
reusable experience across agents.

\noindent\textbf{Constructing and Using Shared Memory.}
Prior work studies how reusable experience can be represented as transferable
skills, rules, workflows, or structured memories
~\citep{ding2026skillgen,fang2025memp,cao2025remember,wu2025meta,
xiangflowsearcher,yang2026plugmem,liang2026learning,xia2026memora},
as well as how stored memory can be selectively used at inference time
~\citep{liu2025rcr,zhao2026shardmemo}.
\modelname{} jointly addresses both stages through evidence-aware global
promotion and query-guided local and global memory use, while retaining the local
grounding needed for execution.
A more comprehensive discussion is provided in
Appendix~\ref{appendix:related_work}.

\section{Methodology}
\label{sec:method}

\subsection{Problem Formulation}

\begin{figure*}[t]
\centering
\includegraphics[width=\linewidth]{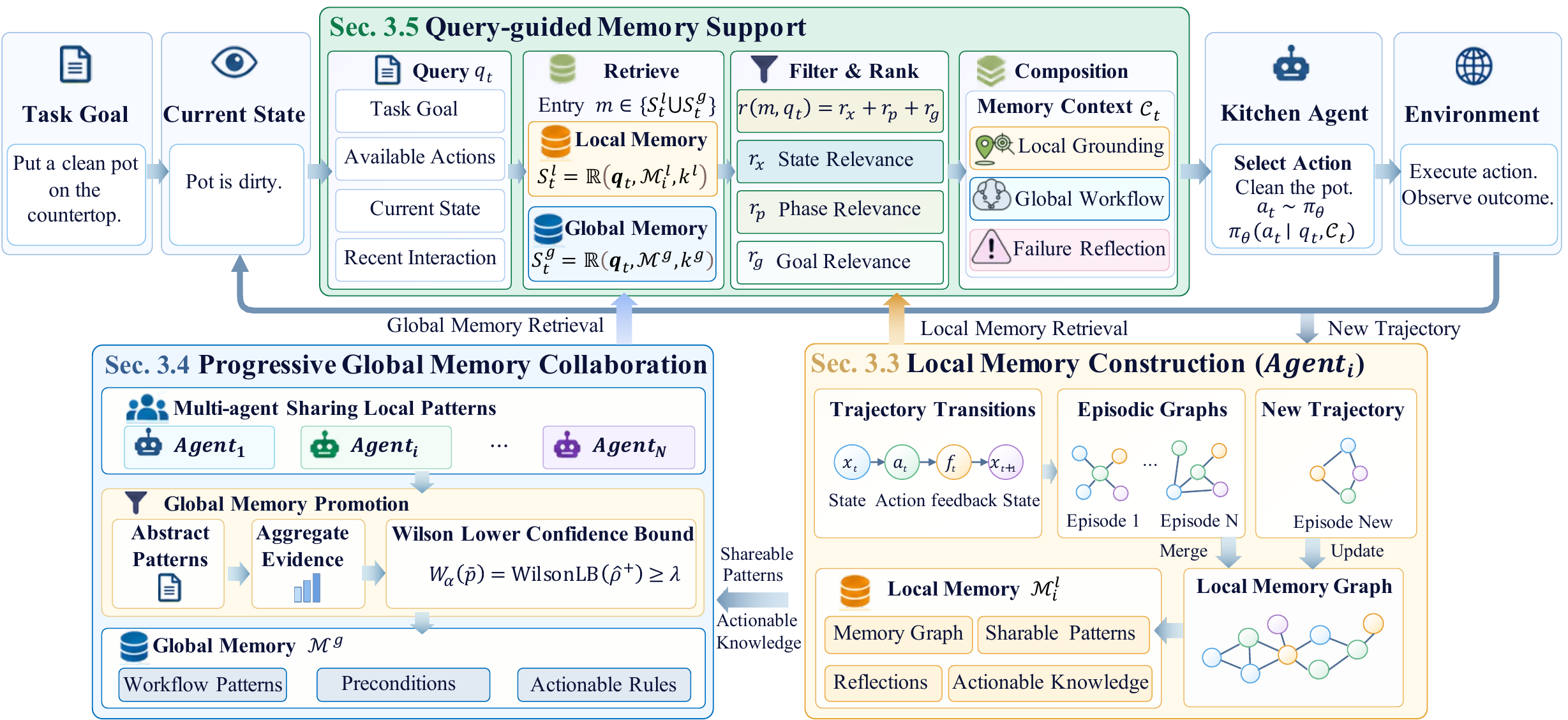}
\caption{Overview of \modelname.
\textit{Local memory construction} organizes environment-specific interaction
experience and extracts reusable records.
\textit{Progressive global memory collaboration} aggregates evidence across
agents and promotes supported reusable records to global memory.
\textit{Query-guided memory support} grounds and composes local and global
memory for the current decision.}
\vspace{-0.3cm}
\label{fig:main}
\end{figure*}

Among $N$ agents, each agent $i$ independently interacts with an environment
using an LLM-based policy model $\pi_{\theta}$.
At step $t$, the agent receives a task goal
$\boldsymbol{g}$, a state observation $\boldsymbol{x}_t$, and an available
action set $\mathcal{A}_t$.
It selects an action $\boldsymbol{a}_t \in \mathcal{A}_t$, and the environment
returns feedback $\boldsymbol{f}_t$ indicating state progress, pause, failure,
or success.
An episode trajectory is
$\tau=
\{(\boldsymbol{x}_t, \boldsymbol{a}_t, \boldsymbol{f}_t,
\boldsymbol{x}_{t+1})\}_{t=1}^{T}$.
Agent $i$ accumulates experience from environments it has interacted with, but
is expected to execute tasks in \emph{unseen environments}, i.e., environments
it has not interacted with before evaluation.
The next action is selected using context $\mathcal C_t$ composed from memory
available to the agent:
\begin{equation}
\boldsymbol a_t\sim\pi_\theta
(\boldsymbol a_t\mid\boldsymbol g,\boldsymbol x_t,\mathcal A_t,\mathcal C_t).
\label{eq:memco-action}
\end{equation}

We aim to organize experience accumulated across agents into reusable memory
and selectively use it to support independent task execution in seen and unseen
environments.

\subsection{Framework Overview}

As in Fig.~\ref{fig:main}, \modelname{} maintains local memory
$\mathcal{M}_i^l$ for environment-specific experience and global memory
$\mathcal{M}^g$ for reusable experience shared across agents.
The framework addresses two stages: constructing memory from accumulated
experience and using it for the current decision.
For memory construction, \textit{local memory construction} organizes each
agent's episodes into a graph and extracts patterns, actionable knowledge, and
reflections, while \textit{progressive global memory collaboration} abstracts
selected shareable records, aggregates evidence across agents, and promotes
supported records to global memory.
For decision-time use, \textit{query-guided memory support} selectively
retrieves and composes relevant local and global memories into context
$\mathcal C_t$ for the current decision.

At step $t$, agent $i$ selects an action according to
Eq.~\ref{eq:memco-action} and uses environment feedback to form the next query,
repeating this process until the episode terminates.
During memory acquisition, completed episodes update local memory and provide
selected records and evidence for global promotion.
Subsequent queries can access the updated memories, enabling agents to benefit
from one another's experience while executing their own tasks independently.
The algorithm details are in
Appendix~\ref{sec:implementation-algorithms}.

\subsection{Local Memory Construction}
\label{sec:3.3}

Local memory construction organizes an agent's interaction episodes
into a graph that connects related state--action transitions and their
feedback across episodes.
The graph preserves the context in which individual experiences occur
and exposes recurring structures across episodes, from which
\modelname{} induces shareable records together with their
supporting evidence.
This provides environment-specific memory for the agent itself and
candidate records for cross-agent sharing.

For each episode $\tau_n$ of agent $i$, \modelname{} constructs an episodic
graph $G_{\tau_n}=(V_{\tau_n},E_{\tau_n})$ whose nodes
$v_t^x$, $v_t^a$, and $v_t^f$ store normalized observations, parsed
operations, and environment feedback, respectively.
Directed edges connect them as temporal transitions:
$v_t^x \rightarrow v_t^a \rightarrow v_t^f \rightarrow v_{t+1}^x$.
Across $N_i$ episodes of the same agent, nodes are matched when their
types and normalized contents agree, and edges are matched when their
relations and endpoints agree.
Merging these elements gives local memory graph
$G_i^l \gets \operatorname{LocalMerge}(\{G_{\tau_n}\}_{n=1}^{N_i})$
as detailed in Algo.~\ref{alg:memco-local-graph-merge-details}.
The graph and the record sets
$\mathcal P_i^l,\mathcal E_i^l,$ and $\mathcal R_i^l$
derived from the interaction episodes together form the local memory:
\begin{equation}
\mathcal M_i^l=
\left[
G_i^l,\mathcal P_i^l,\mathcal E_i^l,\mathcal R_i^l
\right],
\label{eq:local-memory}
\end{equation}
where $\mathcal P_i^l$ contains pattern records of workflows, action
preconditions, repairs, and failures,
$\mathcal E_i^l$ contains actionable-knowledge records, and
$\mathcal R_i^l$ contains concrete cases and reflection records.
We consolidate records of the same type and normalized content,
retaining occurrence-level statistics locally and per-episode annotations for
promotion.

A local record is shareable when
$p \in \mathcal{P}_i^l \cup \mathcal{E}_i^l$.
For each such record $p$, \modelname{} records supporting,
contradicting, and no-progress evidence from its observed occurrences.
Before global promotion, annotations associated with the same record,
source scope, and episode are deterministically reduced to at most one
signed vote.
The resulting $n_i^+(p)$ and $n_i^-(p)$ are the episode-reduced
supporting and contradicting counts used for promotion, while
$n_i^0(p)$ records episode-level stall markers and is kept separate
from the signed Wilson sample.
Detailed evidence assignment and reduction rules are provided in
Appendix~\ref{app:memco-local-details}.

\modelname{} updates the graph after each completed episode
$\tau_{\mathrm{new}}$ as
$G_i^l \gets \operatorname{LocalMerge}(\{G_i^l,G_{\tau_{\mathrm{new}}}\})$
and merges the newly extracted records and evidence into the existing
$\mathcal{P}_i^l$, $\mathcal{E}_i^l$, and $\mathcal{R}_i^l$.
Both $G_i^l$ and $\mathcal R_i^l$ remain local because they retain
environment-specific context.
Selected shareable records from $\mathcal P_i^l$ and $\mathcal E_i^l$,
together with their episode-reduced promotion counts, are shared as they
become available.

\subsection{Progressive Global Memory Collaboration}
\label{sec:3.4}

Progressive global memory collaboration turns selected local experience into
reusable memory shared across agents.
Each agent abstracts and shares selected local records with their
episode-reduced promotion counts, while keeping fine-grained trajectories,
local memory graphs, and instance-specific reflections local.
\modelname{} aggregates evidence across agents and promotes supported abstract
records into global memory for subsequent queries.

For a shareable local record
$p\in\mathcal P_i^l\cup\mathcal E_i^l$,
\modelname{} constructs its abstract form
$\bar p=\operatorname{Abstract}(p)$ by replacing concrete entities where
supported by the benchmark-specific task-role mapping and removing episode-
and scene-specific anchors
(Algo.~\ref{alg:memco-abstract-details}).
This preserves the record type, applicability conditions,
action sequence or relation, and outcome.
Records without removable environment-specific anchors can retain their
structured form, while records without a transferable global form remain local.
For example, a kettle's current location remains local, while a
fill-then-heat procedure can be shared using object and tool roles.

Abstract records across agents share the same abstract key when their
record types, applicability conditions, action sequences or relations,
and outcomes match exactly
(Appendix~\ref{app:memco-merge-example}).
This key determines which local records contribute to the same abstract record.
The episode-reduced counts of local records mapping to the same
abstract record are aggregated as
{
\small
\begin{equation}
n^{\omega}(\bar p)
=
\sum_{\substack{
i\in [N],p':
\operatorname{Abstract}(p')=\bar p}}
n_i^{\omega}(p'),
\qquad
\omega\in\{+,-\}.
\label{eq:promotion-count-aggregation}
\end{equation}
}

Before promotion, fixed type-specific filters exclude incomplete or
uninformative abstract records.
Promotion further depends on how consistently the aggregated evidence
supports each retained record.
Let $s(\bar p)=n^+(\bar p)+n^-(\bar p)$ denote the signed episode-level
evidence size and 
$\rho^+(\bar p)=\nicefrac{n^+(\bar p)}{\max\{s(\bar p),1\}}$
its empirical supporting rate for promotion.

However, finite trajectories may provide sparse evidence, so a high
empirical supporting rate can arise from only a few episode-level votes.
We therefore use the one-sided Wilson lower confidence bound of $\rho^+(\bar p)$, i.e., $W(\bar p)$, as a
finite-evidence promotion score that accounts for both the empirical
supporting rate and evidence size~\citep{wilson1927,brown2001}.
Its statistical interpretation and assumptions are discussed in
Appendix~\ref{app:wilson-analysis}.
Let $z=\Phi^{-1}(1-\alpha)$, where $\alpha\in(0,1/2)$ is the one-sided
tail probability and $\Phi$ is the standard normal cumulative
distribution function.
We set $W(\bar p)=0$ if $s(\bar p)=0$; otherwise,
\begin{equation}
W(\bar p)=
\frac{
s(\bar p)\rho^+(\bar p)+\frac{z^2}{2}
-z\sqrt{
s(\bar p)\rho^+(\bar p)(1-\rho^+(\bar p))
+\frac{z^2}{4}
}}
{s(\bar p)+z^2}.
\label{eq:wilson}
\end{equation}

The following result characterizes the finite-evidence correction.

\begin{theorem}[Wilson finite-evidence correction]
\label{thm:wilson-error-bound}
For $s(\bar p)>0$ and fixed $z>0$, the correction
$\epsilon_W(\bar p)=\rho^+(\bar p)-W(\bar p)$ satisfies
\begin{equation}
0 \leq \epsilon_W(\bar p)
=
z\sqrt{\frac{W(\bar p)(1-W(\bar p))}{s(\bar p)}}
\leq
\frac{z}{2\sqrt{s(\bar p)}} .
\label{eq:wilson-correction}
\end{equation}
Thus, Wilson promotion discounts empirical support under finite evidence,
with the maximal correction vanishing uniformly as
$O(s(\bar p)^{-1/2})$.
\end{theorem}

The proof and exact threshold and sample-size implications are given in
Appendix~\ref{app:memco-theory}.

After the fixed type-specific filters above, global memory is initialized
as $\mathcal M^g=\varnothing$ and updated by merging abstract records whose
Wilson score passes threshold $\lambda$:
\begin{equation}
\mathcal M^g\gets
\operatorname{GlobalMerge}
\left(
\mathcal M^g,
\{\bar p:W(\bar p)\geq\lambda\}
\right),
\quad 0<\lambda<1.
\label{eq:global-exact-merge}
\end{equation}
$\operatorname{GlobalMerge}(\cdot)$ maintains one global record per abstract
key, inserting new records or replacing the stored evidence with the latest
aggregated counts
(Algo.~\ref{alg:memco-global-merge-details}).
New local records or updated evidence trigger reaggregation and reassessment
of the affected abstract records.
Existing global records that no longer satisfy the preceding filters or
threshold are removed.
These updates become available to subsequent queries, allowing each agent to
use experience contributed by other agents while executing its own task
independently.

\subsection{Query-guided Memory Support}
\label{sec:3.5}

To translate stored experience into support for the next decision,
\modelname{} combines environment-specific local memory with reusable global
memory accumulated across agents.
The selected memory must be relevant to the current decision and grounded in
the environment the agent is currently interacting with.
At step $t$, the query
$\boldsymbol q_t=\mathcal Q(
    \boldsymbol g,
    \boldsymbol x_t,
    \mathcal A_t,
    \tau_{<t}
)$
captures the goal, current observation, available actions, and task
progress inferred from the preceding interaction history $\tau_{<t}$.
This query guides both memory selection and composition.

We retrieve two sets of memory entries from local and global sources separately:
\begin{equation}
\mathcal S_t^l
=
\mathbb R
\left(
\boldsymbol q_t,\mathcal M_i^l,k^l
\right),
\qquad
\mathcal S_t^g
=
\mathbb R
\left(
\boldsymbol q_t,\mathcal M^g,k^g
\right),
\label{eq:two-streams}
\end{equation}
where $\mathbb R(\cdot)$ is specified in
Algo.~\ref{alg:memco-bounded-selection}.
The local output $\mathcal S_t^l$ is selected from records in
$\mathcal P_i^l$, $\mathcal E_i^l$, and $\mathcal R_i^l$, together with
bounded transition entries from query-matched states in $G_i^l$.
These entries provide bounded local structural context for assessing workflow
or rule applicability.
The global output $\mathcal S_t^g$ is selected from the promoted records in
$\mathcal M^g$.
Both outputs obey their respective cardinality caps.

Before selection, each memory entry $m$ in the source-specific retrieval pool
is scored along three complementary dimensions: state, task phase, and goal:
\begin{equation}
r(m,\boldsymbol q_t)
=
r_x(m,\boldsymbol q_t)
+
r_p(m,\boldsymbol q_t)
+
r_g(m,\boldsymbol q_t),
\label{eq:memco-relevance}
\end{equation}
where $r_x$ measures consistency with the current state and available
actions, $r_p$ measures compatibility with task progress, and $r_g$
measures the alignment of object, destination, or tool roles with
the goal. 

Selection distinguishes applicability constraints from ranking
preferences.
Entries are filtered when their required conditions contradict the
current state or their actions cannot be matched to $\mathcal A_t$
when such matching is required.
The remaining entries are ranked using joint relevance, stored
evidence, and action compatibility.
We reduce priority when stale local information and
other soft conflicts occur.
Selection retains at most $k^l$ local and $k^g$ global entries.
When a retrieval pool exceeds its budget, strong matches along
individual relevance dimensions are retained before the remaining
positions are filled by rank, preserving complementary support
within the limit
(Algo.~\ref{alg:memco-bounded-selection}).

Composition connects shared knowledge to its local use by grounding
global records against $\boldsymbol x_t$ and $\mathcal A_t$,
with local evidence supplying environment-specific context where
available.
For example, a shared clean-before-place workflow specifies action
order, while local evidence identifies a usable basin for cleaning
the current object
(Appendix~\ref{app:memco-retrieval-example}).
The selected entries are organized by their roles in the current
decision:
\textit{grounding} provides evidence for interpreting the current state,
graph relations, and concrete action references;
\textit{workflow guidance} provides plans, procedures, and preconditions;
and \textit{reflection} provides failure patterns, repairs, and blocked
actions.
%
For each support usage role 
$h\in\mathcal H=
\{\mathrm{ground},\mathrm{workflow},\mathrm{reflection}\}$,
let $\mathcal U_t^h$ contain the entries assigned to that role.
Role-aware composition and query-guided textual rendering
produce the final memory context:
\begin{align}
\{\mathcal U_t^h\}_{h\in\mathcal H}
&=
\operatorname{Compose}
(\mathcal S_t^l,\mathcal S_t^g
\mid\boldsymbol q_t,\mathcal A_t),
\label{eq:memory-composition}\\
\mathcal C_t
&=
\operatorname{Render}
(\boldsymbol q_t,\{\mathcal U_t^h\}_{h\in\mathcal H})
=
[c_t^{\mathrm{ground}},
c_t^{\mathrm{workflow}},
c_t^{\mathrm{reflection}}],
\label{eq:adaptive-context}
\end{align}
where $c_t^h$ is the final textual context, functions
 $\operatorname{Render}(\cdot)$ and $\operatorname{Compose}(\cdot)$ are defined in Algo.~\ref{alg:compose}-\ref{alg:render}.
%
The resulting $\mathcal C_t$ conditions next-action prediction
in Eq.~\ref{eq:memco-action}.
After each action, new observations and feedback update the history
and the next query, allowing memory support to follow task progress.

\section{Experiments}

\noindent\textbf{Benchmarks.} We evaluate \modelname on four benchmarks that necessitate multi-step decision-making across episodes: (1) \textbf{ALFWorld} \cite{shridhar2020alfworld} aligns high-level language instructions with low-level robotic actions, challenging agents to reason and act within simulated environments; (2) \textbf{PDDL} \cite{ma2024agentboard} evaluates symbolic planning by requiring agents to generate valid action sequences that achieve specified goals under domain constraints; (3) \textbf{FEVER} \cite{thorne2018fever} evaluates large-scale fact verification, requiring agents to retrieve and reason over external evidence to classify claims as \textsc{Supports}, \textsc{Refutes}, or \textsc{Not Enough Info} (\textsc{NEI}); and (4) \textbf{ScienceWorld} \cite{wang2022scienceworld} is a comprehensive interactive benchmark for elementary science experiments. 

\noindent\textbf{Baselines.} We compare \modelname against several representative baselines: (1) \textbf{In-Context Learning} (ICL), using direct in-context prompting without a persistent memory bank; (2) \textbf{A-Mem} \cite{xu2025amem}, a specialized long-term memory architecture; (3) \textbf{MemRL} \cite{zhang2026memrl}, which utilizes reinforcement learning to enable self-evolution via a value-based decision process; (4) \textbf{MemSkill} \cite{zhang2026memskill}, which reframes memory management as a set of learnable natural-language skills; and (5) \textbf{G-Memory} \cite{zhang2025g}, a hierarchical graph-based memory for multi-agent systems that incorporates interactions and insights.

\noindent\textbf{Implementation Details.}
We implement \modelname using AutoGen~\cite{wu2024autogen} and evaluate open- and closed-source models. We use $\alpha=0.05$ and $\lambda=0.35$ for promotion across all benchmarks. Open-source model inference is conducted on a single NVIDIA RTX PRO 6000 Blackwell GPU. Both local and global retrievers use deterministic bounded selection with facet coverage and rank-based filling, i.e., $k^{\ell}=k^{g}=5$. Unless otherwise specified, we report accuracy (Acc., \%) and the average number of interaction steps during evaluation. For each benchmark, memory is constructed from designated source partitions and evaluated on disjoint tasks, with no evaluation trajectories used for memory construction. We use one agent per source domain: four agents for ALFWorld and PDDL, two for FEVER, and ten for ScienceWorld. Each agent is evaluated on the merged test set using its own frozen local memory and the shared global memory, covering both its acquisition domain and other environments. ALFWorld's Familiar and Held-out splits additionally assess layout generalization. Full partitioning and construction details
are in Appendix~\ref{appendix:benchmark}.

\subsection{Experimental Analysis}
\begin{table*}[t]
\centering
\caption{Main results on FEVER, PDDL, and ALFWorld (mean $\pm$ standard
deviation over three runs). \textbf{Bold} and \underline{underlined} indicate the best and second-best results within each backbone, respectively.}
\label{tab:memory-comparison-simplified}

\begingroup
\fontsize{7.5}{9}\selectfont
\setlength{\tabcolsep}{2pt}
\renewcommand{\arraystretch}{1.12}

\begin{adjustbox}{max width=0.92\textwidth}
\begin{tabular}{ll*{8}{c}}
\toprule
\multirow{2}{*}{\textbf{Backbone}}
& \multirow{2}{*}{\textbf{Method}}
& \multicolumn{2}{c}{\textbf{FEVER}}
& \multicolumn{2}{c}{\textbf{PDDL}}
& \multicolumn{4}{c}{\textbf{ALFWorld}} \\

\cmidrule(lr){3-4}
\cmidrule(lr){5-6}
\cmidrule(lr){7-8}
\cmidrule(lr){9-10}

& & Acc. $\uparrow$ & Steps $\downarrow$
  & Acc. $\uparrow$ & Steps $\downarrow$
  & Familiar Acc. $\uparrow$ & Familiar Steps $\downarrow$
& Held-out Acc. $\uparrow$ & Held-out Steps $\downarrow$ \\
\midrule

\multirow{6}{*}{\textbf{Qwen3-4B}}
& ICL
& \mcnum{56.78}{1.47} & \mcnum{4.23}{0.52}
& \mcnum{44.44}{3.47} & \mcnum{29.19}{0.14}
& \mcnum{60.28}{4.73} & \mcnum{20.43}{0.18}
& \mcnum{46.31}{3.79} & \mcnum{22.45}{0.74} \\

& A-Mem
& \mcnum{54.24}{2.24} & \mcnum{4.94}{1.99}
& \mcnum{37.22}{8.42} & \mcnum{29.17}{0.31}
& \mcnum{58.66}{8.24} & \mcnum{19.88}{0.93}
& \mcnum{45.60}{3.57} & \mcnum{22.14}{0.32} \\

& MemSkill
& \mcsecond{59.89}{4.35} & \mcsecond{3.41}{0.22}
& \mcnum{40.97}{3.49} & \mcnum{29.31}{0.07}
& \mcnum{61.90}{6.66} & \mcnum{20.00}{1.07}
& \mcsecond{57.74}{2.72} & \mcsecond{20.28}{1.17} \\

& MemRL
& \mcnum{55.93}{1.70} & \mcnum{4.29}{0.31}
& \mcsecond{46.58}{5.18} & \mcbest{29.03}{0.10}
& \mcnum{59.31}{4.88} & \mcnum{20.60}{0.22}
& \mcnum{49.41}{0.54} & \mcnum{22.16}{0.43} \\

& G-Memory
& \mcnum{55.50}{8.99} & \mcnum{5.44}{4.37}
& \mcnum{39.38}{0.88} & \mcnum{29.18}{0.18}
& \mcsecond{62.23}{2.31} & \mcsecond{18.73}{0.61}
& \mcnum{53.81}{1.15} & \mcnum{20.57}{1.08} \\

\rowcolor{black!4}
& \textbf{\modelname}
& \mcbest{60.17}{1.20} & \mcbest{3.38}{0.48}
& \mcbest{53.01}{1.47} & \mcsecond{29.16}{0.21}
& \mcbest{84.21}{2.28} & \mcbest{13.33}{0.64}
& \mcbest{82.84}{2.65} & \mcbest{13.77}{0.71} \\

\midrule[0.35pt]

\multirow{6}{*}{\textbf{GPT-4o-mini}}
& ICL
& \mcsecond{63.00}{3.42} & \mcsecond{6.33}{1.10}
& \mcnum{57.78}{2.55} & \mcnum{28.97}{0.12}
& \mcnum{28.57}{3.08} & \mcnum{25.86}{0.32}
& \mcnum{35.71}{2.77} & \mcnum{25.53}{0.74} \\

& A-Mem
& \mcnum{61.30}{4.90} & \mcnum{7.02}{0.87}
& \mcnum{52.64}{2.77} & \mcnum{29.36}{0.12}
& \mcnum{23.05}{4.82} & \mcnum{26.62}{0.69}
& \mcnum{35.00}{2.40} & \mcnum{26.04}{0.75} \\

& MemSkill
& \mcnum{62.71}{3.06} & \mcnum{7.76}{0.21}
& \mcbest{59.03}{4.46} & \mcsecond{28.89}{0.17}
& \mcnum{16.92}{6.06} & \mcnum{27.75}{1.42}
& \mcnum{28.93}{4.43} & \mcnum{26.70}{1.05} \\

& MemRL
& \mcnum{62.71}{2.93} & \mcnum{6.57}{2.11}
& \mcnum{55.56}{2.77} & \mcnum{28.98}{0.04}
& \mcnum{28.24}{3.27} & \mcnum{25.74}{0.21}
& \mcnum{34.41}{0.65} & \mcnum{26.10}{0.30} \\

& G-Memory
& \mcnum{62.99}{6.36} & \mcnum{7.95}{0.26}
& \mcsecond{58.89}{2.30} & \mcbest{28.41}{0.19}
& \mcbest{42.21}{2.63} & \mcsecond{24.05}{1.23}
& \mcsecond{49.29}{1.94} & \mcsecond{24.23}{1.80} \\

\rowcolor{black!4}
& \textbf{\modelname}
& \mcbest{70.90}{1.76} & \mcbest{2.63}{0.18}
& \mcnum{55.56}{1.46} & \mcnum{29.15}{0.02}
& \mcsecond{40.26}{3.10} & \mcbest{22.69}{0.55}
& \mcbest{50.00}{2.86} & \mcbest{21.75}{0.61} \\

\midrule[0.35pt]

\multirow{6}{*}{\textbf{Qwen3-32B}}
& ICL
& \mcnum{66.67}{1.96} & \mcnum{5.66}{2.24}
& \mcnum{76.25}{0.42} & \mcnum{26.19}{0.13}
& \mcnum{79.72}{1.42} & \mcnum{16.88}{0.46}
& \mcnum{79.69}{1.75} & \mcnum{17.54}{0.73} \\

& A-Mem
& \mcnum{68.08}{1.76} & \mcsecond{4.40}{1.93}
& \mcnum{54.58}{2.09} & \mcsecond{25.21}{0.41}
& \mcnum{81.82}{1.39} & \mcnum{15.44}{0.45}
& \mcnum{77.37}{1.22} & \mcnum{16.67}{1.17} \\

& MemSkill
& \mcnum{67.52}{3.43} & \mcnum{4.92}{1.63}
& \mcnum{70.28}{2.84} & \mcnum{26.56}{0.50}
& \mcnum{86.48}{5.45} & \mcnum{14.11}{1.76}
& \mcnum{85.88}{6.13} & \mcnum{15.92}{0.92} \\

& MemRL
& \mcnum{67.52}{5.38} & \mcnum{5.48}{1.65}
& \mcsecond{77.78}{1.05} & \mcnum{26.02}{0.10}
& \mcnum{80.29}{1.66} & \mcnum{16.46}{0.20}
& \mcnum{77.94}{0.75} & \mcnum{16.98}{0.16} \\

& G-Memory
& \mcsecond{68.64}{3.43} & \mcnum{6.49}{1.18}
& \mcnum{68.47}{2.10} & \mcbest{24.92}{0.43}
& \mcsecond{86.69}{2.94} & \mcsecond{13.38}{1.84}
& \mcsecond{87.19}{2.73} & \mcsecond{14.12}{2.52} \\

\rowcolor{black!4}
& \textbf{\modelname}
& \mcbest{71.19}{2.05} & \mcbest{3.24}{0.52}
& \mcbest{78.89}{1.58} & \mcnum{26.20}{0.10}
& \mcbest{96.62}{1.82} & \mcbest{9.71}{0.37}
& \mcbest{97.46}{1.45} & \mcbest{11.23}{0.42} \\

\bottomrule
\end{tabular}
\end{adjustbox}

\par\vspace{3pt}
\endgroup
\vspace{-10pt}
\end{table*}

\noindent\textbf{Overall Results.}
Tab.~\ref{tab:memory-comparison-simplified} presents the main comparison across three benchmarks and three LLM backbones. For ALFWorld, Familiar and Held-out denote the official \texttt{valid\_seen} and \texttt{valid\_unseen} layout splits, respectively. Counting FEVER, PDDL, and ALFWorld for each backbone, \modelname achieves the highest accuracy in 10 out of the 12 evaluation settings and requires the fewest interaction steps in 9 of them. The performance gains are frequently accompanied by shorter trajectories, indicating that \modelname improves task outcomes without relying on additional environment exploration.

The largest gains occur on ALFWorld: with Qwen3-4B,
\modelname improves held-out layout accuracy by 43.47\% and reduces
interaction steps by 32.10\% relative to the runner-up baseline
(MemSkill). Gains with Qwen3-32B show that benefits extend beyond
weaker policy models. Held-out results further support transfer
across layouts under this cross-environment evaluation protocol.
\modelname also leads FEVER in both accuracy and interaction
efficiency under all three backbones. With GPT-4o-mini, it improves
accuracy by 12.54\% and reduces steps by 58.45\%
relative to the runner-up baseline.
It also improves PDDL accuracy with both Qwen models.
These PDDL gains do not always coincide with the fewest interaction
steps, showing that success and interaction efficiency are not always
aligned. GPT-4o-mini reveals model dependence, but \modelname still
achieves the best held-out-layout accuracy and shortest ALFWorld
trajectories. Furthermore, exact two-sided Wilcoxon signed-rank
tests on the 12 matched configuration means indicate improvements
over each baseline in accuracy and steps after Holm correction,
under the stated block-level assumptions.
Full statistics and effect sizes appear in
Appendix~\ref{appendix:statistical-significance}.

\noindent\textbf{Robustness across Additional LLM Backbones.}
Tab.~\ref{tab:fever-backbone-combined} examines whether the improvements
persist across additional LLM families. Across them, \modelname consistently achieves the highest FEVER accuracy, improving over the runner-up baseline by 10.81\%, 32.82\%, and 6.98\%, respectively. These gains generally coincide with shorter trajectories. On DeepSeek-v4-flash, \modelname improves accuracy while requiring nearly the same number of steps as ICL. However, G-Memory did not complete successfully in this setting.
Together with Tab.~\ref{tab:memory-comparison-simplified}, these results indicate that the benefit of \modelname generalizes across model families rather than depending on a particular backbone or model scale.

\begin{table*}[t]
\centering
\begingroup

\captionsetup{skip=3pt}
\caption{Generalization across LLM backbones and to more complex tasks.}
\label{tab:fever-scienceworld}

\fontsize{7}{8.2}\selectfont
\renewcommand{\arraystretch}{1.02}
\setlength{\aboverulesep}{0.25ex}
\setlength{\belowrulesep}{0.25ex}

\captionsetup[subtable]{
    font=scriptsize,
    labelfont=bf,
    textfont=bf,
    skip=1.5pt
}

\begin{subtable}[t]{0.58\textwidth}
\centering
\caption{FEVER across different LLM backbones.}
\label{tab:fever-backbone-combined}

\setlength{\tabcolsep}{1.4pt}
\begin{tabularx}{\linewidth}{
    @{}l*{6}{>{\centering\arraybackslash}X}@{}
}
\toprule

\multirow{2}{*}{\textbf{Method}}
& \multicolumn{2}{c}{\textbf{Claude-Haiku-4.5}}
& \multicolumn{2}{c}{\textbf{Gemini-2.5-flash-lite}}
& \multicolumn{2}{c}{\textbf{DeepSeek-v4-flash}} \\

\cmidrule(lr){2-3}
\cmidrule(lr){4-5}
\cmidrule(lr){6-7}

& Acc. $\uparrow$ & \mbox{Steps\,$\downarrow$}
& Acc. $\uparrow$ & \mbox{Steps\,$\downarrow$}
& Acc. $\uparrow$ & \mbox{Steps\,$\downarrow$} \\

\midrule

ICL
& 54.24 & 10.44
& 49.15 & 6.24
& \underline{72.88} & \textbf{3.02} \\

A-Mem
& \underline{62.71} & 9.02
& 45.76 & 5.98
& 63.56 & 4.35 \\

MemSkill
& 60.17 & \underline{7.46}
& 53.39 & 6.05
& 68.65 & 3.54 \\

MemRL
& 54.24 & 7.69
& \underline{54.24} & \underline{3.47}
& 67.80 & 3.67 \\

G-Memory
& 57.63 & 7.95
& 44.92 & 5.15
& -- & -- \\

\rowcolor{black!4}
\textbf{\modelname{}}
& \textbf{69.49} & \textbf{6.38}
& \textbf{72.04} & \textbf{2.42}
& \textbf{77.97} & \underline{3.04} \\

\bottomrule
\end{tabularx}
\end{subtable}%
\hfill
\begin{subtable}[t]{0.40\textwidth}
\centering
\caption{ScienceWorld with Qwen3-32B.}
\label{tab:scienceworld}

\setlength{\tabcolsep}{1.2pt}
\begin{tabularx}{\linewidth}{
    @{}l*{4}{>{\centering\arraybackslash}X}@{}
}
\toprule

\multirow{2}{*}{\textbf{Method}}
& \multicolumn{4}{c}{\textbf{Qwen3-32B}} \\

\cmidrule(lr){2-5}

& \mbox{Success\,$\uparrow$}
& \mbox{Progress\,$\uparrow$}
& \mbox{Score\,$\uparrow$}
& \mbox{Steps\,$\downarrow$} \\

\midrule

ICL
& 13.80 & 79.80 & 25.75 & 10.77 \\

A-Mem
& 17.33 & 82.17 & 25.81 & 10.87 \\

MemSkill
& 13.67 & 81.50 & 28.13 & 11.79 \\

MemRL
& 8.33 & 77.00 & 22.25 & \textbf{5.21} \\

G-Memory
& \underline{21.67} & \textbf{91.33}
& \underline{30.02} & 9.67 \\

\rowcolor{black!4}
\textbf{\modelname{}}
& \textbf{30.47} & \underline{82.90}
& \textbf{34.39} & \underline{9.57} \\

\bottomrule
\end{tabularx}
\end{subtable}

\endgroup
\end{table*}

\noindent\textbf{Generalization to a More Complex Interactive Environment.}
On ScienceWorld, Tab.~\ref{tab:scienceworld} shows that \modelname achieves the highest final success and score, and the auxiliary metrics further reveal a distinction between intermediate exploration and successful task completion. G-Memory achieves higher intermediate progress, while MemRL terminates in fewer steps, yet neither converts these advantages into comparable final success. These results show that intermediate progress and fewer steps alone do not ensure successful completion, e.g., \modelname achieves better final outcomes with a step count close to G-Memory's. 
%

\subsection{Mechanism and Sensitivity Analysis}

Following the memory lifecycle introduced in Section~\ref {sec:method}, we conduct controlled studies of \modelname's complementary memory scopes, evidence-aware promotion, and query-guided memory use. Unless otherwise specified, all variants use the same frozen local-memory snapshot and differ only in the component under evaluation. These results represent controlled within-snapshot comparisons rather than independent repetitions of the main experiments.

\noindent\textbf{Complementary Local and Global Memory.} Fig.~\ref{fig:lambda-component-ablation}(c)-(d) shows that Local-only performs well on familiar layouts but degrades on held-out layouts, whereas Global-only narrows this gap. Held-out accuracy increases from $89.64\%$ with Local-only to $97.14\%$ with Global-only and $97.86\%$ with the full system. Retaining local evidence also improves familiar-layout accuracy over Global-only, from $96.10\%$ to $97.40\%$. Their adaptive combination achieves the highest overall accuracy with shorter trajectories than either isolated scope, supporting complementary roles for local grounding and transferable global guidance with abstract rules and workflows.

\begin{figure*}[t]
  \centering
  \includegraphics[width=\textwidth]{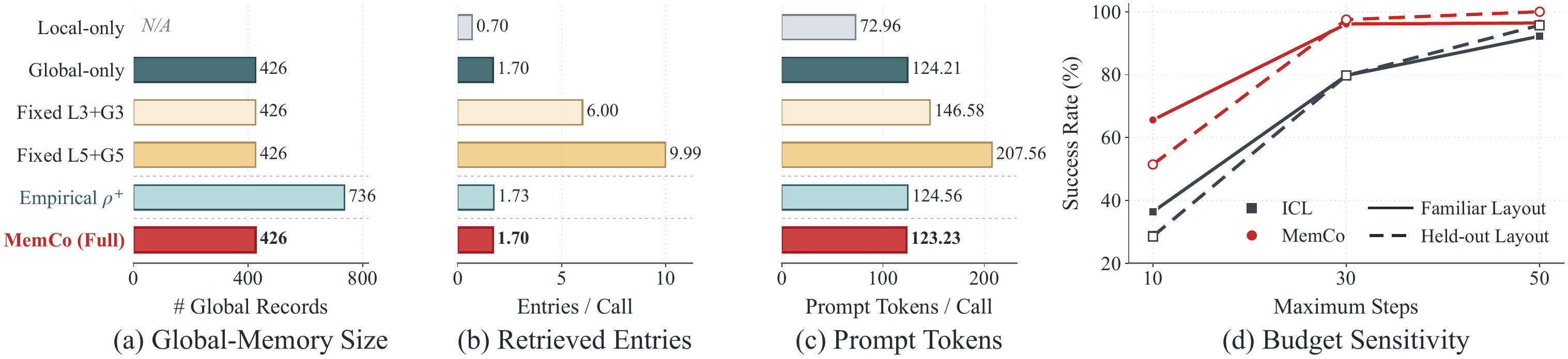}
  \caption{\textbf{Memory efficiency and interaction-budget sensitivity on ALFWorld with Qwen3-32B.}
  Panels (a)--(c) report the global-memory size, average retrieved memory entries, and memory-prompt tokens per retrieval call. Panel (d) compares \modelname and ICL under different interaction budgets.}
  \label{fig:memory-cost-budget-sensitivity}
  \vspace{-15pt}
\end{figure*}

\begin{figure*}[t]
  \centering
  \includegraphics[width=0.8\textwidth]
  {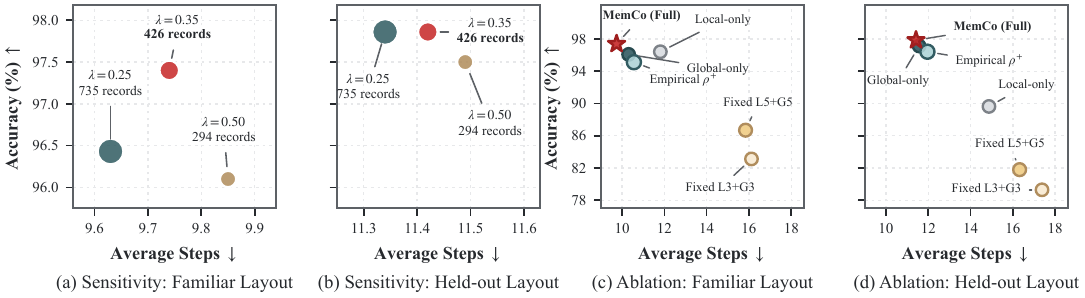}
  \caption{\textbf{Sensitivity and component ablations on ALFWorld with Qwen3-32B.} Panels (a)--(b) report familiar-layout and held-out-layout performance under different Wilson promotion thresholds $\lambda$; bubble area represents the number of promoted global-memory records. Panels (c)--(d) compare memory-scope, fixed-concatenation, empirical-promotion, and full \modelname variants.}
  \label{fig:lambda-component-ablation}
   \vspace{-8pt}
\end{figure*}

\noindent\textbf{Finite-Evidence-Aware Progressive Collaboration.}
We replace the Wilson lower confidence bound promotion rule $W(\bar p)$ with its counterpart $\rho^{+}(\bar p)$ while keeping $\lambda=0.35$ and the frozen local evidence fixed. As shown in Fig.~\ref{fig:memory-cost-budget-sensitivity}(a)-(c) and Fig.~\ref{fig:lambda-component-ablation}(c)-(d), Wilson promotion reduces the global bank from 736 to 426 records and improves familiar- and held-out-layout accuracy by 2.27 and 1.43 percentage points, respectively, with fewer interaction steps. Because the two variants inject nearly the same number of entries per decision, the difference is not due to prompt size alone.
 This result is consistent with Wilson calibration favoring candidates supported by more reliable finite evidence.

\noindent\textbf{Efficiency under Restricted Interaction Budgets.}
As shown in Fig.~\ref{fig:memory-cost-budget-sensitivity}(d), \modelname improves when the interaction budget increases from 10 to 30 steps, but gains only marginally from 30 to 50 steps on both splits. This early saturation suggests that a 30-step budget is largely sufficient in this setting for \modelname to exploit its retrieved experience effectively. In comparison, ICL improves more gradually as the budget increases and requires more interaction opportunities to approach \modelname's performance. By varying step limits, \modelname at 30 steps already exceeds ICL at 50 steps on both splits, supporting improved task completion under smaller execution budgets.

\noindent\textbf{Sensitivity to the Promotion Threshold.} Fig.~\ref{fig:lambda-component-ablation}(a)-(b) evaluates the robustness of the promotion threshold by varying $\lambda$ while keeping the same frozen local-memory snapshot and reconstructing only the global memory. Increasing $\lambda$ reduces the global memory from 735 to 294 records, while accuracy and trajectory length remain insensitive. 
Threshold $\lambda=0.35$ reduces the number of global
memory records by 42.0\% compared with $\lambda=0.25$, while
maintaining similar accuracy and interaction steps.

\noindent\textbf{Query-Guided Memory Support.} The fixed-concatenation variants in Fig.~\ref{fig:lambda-component-ablation}(c)--(d) test whether \modelname benefits merely from exposing the policy model to more stored experience. As shown in Fig.~\ref{fig:lambda-component-ablation}(c)--(d) and Fig.~\ref{fig:memory-cost-budget-sensitivity}(b)--(c), the results show the opposite: L5+G5 trails full \modelname by 10.71 and 16.07 percentage points on the Familiar and Held-out Layout splits, respectively, despite injecting more memory. These results indicate that the way retrieved records are selected and composed matters; simply concatenating more records does not recover the benefit of adaptive composition. The full system uses less context while achieving higher accuracy and shorter trajectories. This is consistent with less targeted memory leading irrelevant guidance into state-specific decision making.

Fig.~\ref{fig:relevance-retrieval-ablation}(a) examines the state, phase, and goal relevance facets. Removing any one facet or retaining only a single facet lowers accuracy and increases interaction steps on both splits. The all-zero relevance control in Appendix~\ref{appendix:extended-results} also underperforms the full method.
These within-snapshot comparisons support the complementary, incremental contributions of the three facets to memory selection.
Fig.~\ref{fig:relevance-retrieval-ablation}(b) examines an orthogonal
design choice: how memory relevance is scored. It replaces
\modelname's explicit state, phase, and goal matching with cosine
similarity or BM25.
Both alternatives underperform the full method in accuracy and efficiency, and neither consistently dominates across splits. 
These comparisons support explicit state, phase, and goal matching over generic semantic or lexical similarity for selecting memory entries in the evaluated setting.

\begin{figure*}[t]
  \centering
  \includegraphics[width=\textwidth]
  {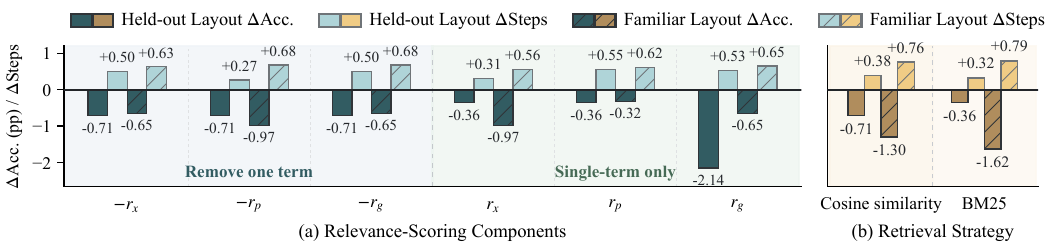}
\caption{\textbf{Ablation of relevance scoring and retrieval on ALFWorld with Qwen3-32B.} Bars show changes from full \modelname, with $\Delta$Acc. measured in percentage points and $\Delta$Steps in interactions. In (a), for each relevance facet $d\in\{x,p,g\}$, $-r_d$ removes term $r_d$, whereas $r_d$ retains only that term. Panel (b) replaces \modelname's retrieval strategy with cosine similarity or BM25.}
  \label{fig:relevance-retrieval-ablation}
  
 \vspace{-10pt}
\end{figure*}

Qualitative trajectory analysis in Appendix~\ref{appendix:case-study} further illustrates how this mechanism combines transferable global workflows with environment-specific local grounding, while excluding incompatible source-environment guidance in the shown cross-environment trajectories.

\FloatBarrier
\section{Case Study: Adaptive Memory Support and Cross-Environment Generalization}

\begin{figure}[H]
  \centering
  \includegraphics[width=0.94\linewidth]{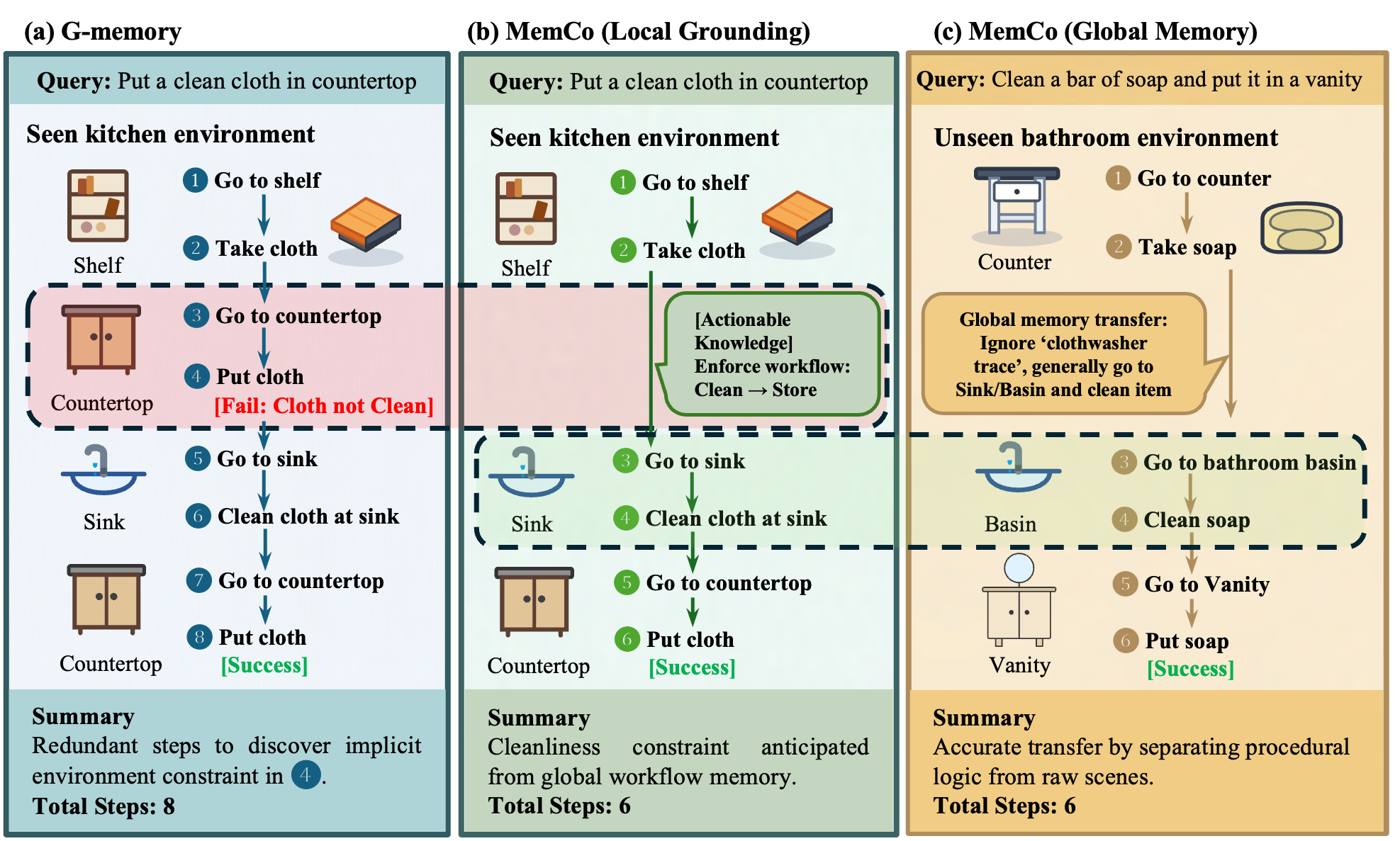}
  \caption{\textbf{Case study showcasing adaptive memory support and cross-environment generalization.} (a) In the shown G-Memory trajectory, the agent attempts placement before satisfying the cleanliness requirement. (b) \modelname{} uses local actionable knowledge to avoid the invalid placement in the same environment. (c) In an unseen environment, \modelname{} combines a reusable global workflow with current-environment grounding.}
  \label{fig:casestudy-main}
\end{figure}

Figure~\ref{fig:casestudy-main} illustrates how local situational grounding and global procedural workflows interact during online decision making. For the task ``Put a clean cloth in countertop'', G-Memory attempts placement before satisfying the cleanliness requirement, while \modelname{} retrieves the clean-before-store ordering together with local actionable knowledge and avoids the invalid move. In the unseen bathroom environment, query-guided selection excludes incompatible kitchen-specific guidance and grounds the transferable workflow in the available basin and vanity, completing the task in six steps.

\section{Conclusion}

In this work, we introduced MemCo, a memory-centric collaboration framework for generalizing LLM agents to unseen interactive environments. \modelname constructs local memory from each agent's interaction trajectories, progressively promotes reusable abstract records to shared global memory, and adaptively selects and composes local and global evidence into a decision-relevant memory context during online interaction. Experiments across four heterogeneous benchmarks show that \modelname improves task success and generally reduces redundant exploration compared with memory-free, isolated-memory, and shared-memory baselines.



\section*{Limitations}

First, our framework relies on initial interaction traces to bootstrap the global workflow memory. While a complete lack of early experience limits global guidance, this cold-start characteristic is common to all retrieval-based experience sharing systems. Second, we evaluate \modelname exclusively on text-based and symbolic interactive benchmarks. Extending this hierarchical memory architecture to multimodal environments with visual observations remains an open direction for future work.

\bibliography{references}
\bibliographystyle{iclr2027_conference}

\newpage
\appendix
\makeatletter
\@ifundefined{corollary}{\newtheorem{corollary}{Corollary}[section]}{}
\makeatother
\MemCoAppendixContents

\section{Extended Related Work}
\label{appendix:related_work}

\setcounter{figure}{0}
\renewcommand{\thefigure}{A.\arabic{figure}}
\renewcommand{\theHfigure}{\modelnameappendix.\arabic{figure}}

\noindent\textbf{Individual Experience and Agent Memory.}
Generative Agents~\citep{park2023generative} stores observation streams and
derives higher-level reflections, while MemGPT~\citep{packer2023memgpt}
manages information beyond the immediate context window.
MemoryBank~\citep{zhong2024memorybank}, Mem0~\citep{chhikara2025mem0},
and MIRIX~\citep{wang2025mirix} study long-term memory for conversational,
personalized, or multimodal interactions, while
MEMENTO~\citep{kwon2025embodied} evaluates memory utilization for
personalized embodied assistance.
Long contexts themselves do not guarantee effective reuse, since relevant
information may become difficult to access and irrelevant context can distract
reasoning~\citep{liu2024lost,shi2023large}, motivating explicit storage and
retrieval of decision-relevant experience.
For interactive agents, a closely related line of work turns environment
interaction into reusable decision support.
Reflexion~\citep{shinn2023reflexion} stores verbal reflections from previous
trials; Voyager~\citep{wang2023voyager} and GITM~\citep{zhu2023ghost}
accumulate reusable skills or knowledge for open-world interaction; and
SkillGen~\citep{ding2026skillgen} extracts domain skills for sequential
decision making.
Memp~\citep{fang2025memp} studies procedural-memory construction, retrieval,
and update, while ReMe~\citep{cao2025remember} distills and refines procedural
experience.
Meta-Policy Reflexion~\citep{wu2025meta} consolidates reflections into reusable
rules, whereas A-Mem~\citep{xu2025amem}, Nemori~\citep{nan2025nemori},
MemRL~\citep{zhang2026memrl}, and MemSkill~\citep{zhang2026memskill}
improve memory organization, episodic reuse, or memory-management policies.
These methods primarily improve how experience accumulated within an individual
agent or memory system is represented and reused.
When multiple agents interact with different environments, a complementary
question is how experience acquired by one agent can become useful to others.

\noindent\textbf{Agent Collaboration through Shared Experience.}
Multi-agent collaboration is commonly realized through communication, agent
selection, or execution-time orchestration, e.g.,
AutoGen~\citep{wu2024autogen} and AgentStore~\citep{jia2025agentstore}.
Experience sharing provides a complementary form of collaboration in which
experience acquired in prior interactions can support other agents later.
Recent work realizes such sharing through different representations.
G-Memory~\citep{zhang2025g} organizes prior collaboration within a multi-agent
system through hierarchical interaction, query, and insight memory.
MATM~\citep{kim2026multiagenttransactivememory} supports population-level
storage and retrieval of producer-agent trajectories for later use by consumer
agents.
Agent KB~\citep{tang2025agentkbleveragingcrossdomain} aggregates trajectories
from heterogeneous agent frameworks into a structured knowledge base for
cross-domain planning and feedback.
MemCollab~\citep{chang2026memcollabcrossmodelmemorycollaboration} distills
shared constraints from trajectories generated by different backbone models,
while FedAgentKE~\citep{li2026fedagentkefederatedsemanticknowledge}
evolves transferable semantic knowledge across heterogeneous agents through
federated distillation and aggregation.
Collectively, these works demonstrate that experience can be reused beyond
its original interaction and, in different settings, across agents, models,
or agent frameworks. \modelname{} builds on this direction for independently
acting agents that accumulate experience in different environments. It keeps
environment-specific experience local while sharing reusable records across
agents, so that cross-agent experience can be used without collapsing the
context in which it was acquired.

\noindent\textbf{Constructing and Using Shared Memory.}
A first question is how reusable experience should be represented and retained.
Procedural-memory methods show that trajectories can be distilled into more
compact skills, scripts, rules, or workflows
~\citep{ding2026skillgen,fang2025memp,cao2025remember,wu2025meta}.
FlowSearcher~\citep{xiangflowsearcher} synthesizes memory-guided workflows from
historical experience, while PlugMem~\citep{yang2026plugmem} organizes
episodic experience into a knowledge-centric memory graph.
MCMA~\citep{liang2026learning} learns memory structures at different
abstraction levels for selective transfer under distribution shift, and
Memora~\citep{xia2026memora} explicitly balances abstraction against
specificity.
Together, these studies highlight that transferable memory benefits from
capturing reusable structure while retaining information needed for later
application. A second question is how stored memory should be used for the
current decision. RCR-Router~\citep{liu2025rcr} dynamically selects structured
memory subsets according to agent role and task stage, while
ShardMemo~\citep{zhao2026shardmemo} uses masked mixture-of-experts routing
over sharded memory. These works highlight that the utility of stored memory
also depends on context-aware selection at inference time.
\modelname{} connects these two stages in a single memory-centric collaboration
process: environment-specific evidence is retained locally, reusable records
are promoted using evidence accumulated across agents, and local and global
memory are selected and grounded according to the receiving agent's current
state, progress, goal, and available actions. These mechanisms jointly define
the proposed collaboration framework rather than treating shared-memory
construction and decision-time use as separate problems.
\section{Algorithms and Implementation}
\label{app:algorithms-implementation}

\subsection{Algorithmic Overview}
\label{sec:implementation-algorithms}

To improve reproducibility, this section summarizes the concrete procedures
used by \modelname{}.
Algorithm~\ref{alg:local-memory} constructs local graph memory and extracts
pattern records from interaction trajectories.
Algorithm~\ref{alg:memco-local-graph-merge-details} specifies the structural
merge used to consolidate one agent's episodic graphs.
Algorithm~\ref{alg:promotion-merge} abstracts the shareable pattern and
actionable records in $\mathcal P_i^l\cup\mathcal E_i^l$, aggregates their
evidence across agents, and rebuilds the Wilson-promoted global snapshot.
Algorithm~\ref{alg:memco-global-merge-details} separately specifies the
exact-key update of promoted records in global memory.
Algorithm~\ref{alg:routing-current} retrieves and composes local and global
memory according to the current decision context.

Following the main text, local memory is
$\mathcal M_i^l=(G_i^l,\mathcal P_i^l,\mathcal E_i^l,\mathcal R_i^l)$.
Here $p\in\mathcal P_i^l\cup\mathcal E_i^l$ denotes a shareable local
pattern or actionable record and $\bar p=\operatorname{Abstract}(p)$ its
abstract form. The local graph $G_i^l$ and concrete cases
and reflections $\mathcal R_i^l$ remain local. At decision time, $m$ denotes
a memory entry, including a record or a bounded transition entry derived from
$G_i^l$.
We use $N$ for the number of agents and $N_i$ for agent $i$'s episodes.
The counters $c_i^\omega(p)$ and $n_i^\omega(p)$ denote occurrence-level
and episode-reduced evidence, respectively; $c_t^h$ denotes rendered text
for support role $h$.

Table~\ref{tab:memco-memory-types} provides a unified reference for the
memory components, their record types, local examples, and possible global
counterparts. Appendix~\ref{sec:memco_module_examples} distinguishes these
stored types from the support roles used during retrieval and composition.

\begin{algorithm*}[!htbp]
\caption{Local Memory Construction}
\label{alg:local-memory}
\small
\begin{algorithmic}[1]
\STATE \textbf{Input:} Agent $i$'s interaction trajectories
$\{\tau_n\}_{n=1}^{N_i}$, where
$\tau_n=\{(\boldsymbol{x}_t,\boldsymbol{a}_t,\boldsymbol{f}_t,
\boldsymbol{x}_{t+1})\}_{t=1}^{T_n}$
\STATE \textbf{Return:} Local memory
$\mathcal{M}_i^l=(G_i^l,\mathcal{P}_i^l,\mathcal{E}_i^l,\mathcal{R}_i^l)$

\FOR{each trajectory $\tau_n$}
    \STATE Initialize an episodic graph
    $G_{\tau_n}=(V_{\tau_n},E_{\tau_n})$.
    \FOR{each transition
    $(\boldsymbol{x}_t,\boldsymbol{a}_t,\boldsymbol{f}_t,
    \boldsymbol{x}_{t+1})\in\tau_n$}
        \STATE At the method level, associate state nodes $v_t^x,v_{t+1}^x$,
        action node $v_t^a$, and feedback node $v_t^f$ with the normalized
        observation, parsed operation, and observed outcome.
        \STATE Represent the interaction as
        $v_t^x\rightarrow v_t^a\rightarrow v_t^f\rightarrow v_{t+1}^x$.
        In the serialized graph, materialize the
        $v_t^x\rightarrow v_t^a\rightarrow v_{t+1}^x$ backbone and attach
        the information represented by $v_t^f$ to the corresponding transition.
        \STATE Store the attached feedback through transition statistics and,
        when present, typed state-delta, subgoal, and failure relations together
        with the record evidence ledger.
    \ENDFOR
\ENDFOR

\STATE Preserve the observed transition order within each episodic graph so
that record induction follows the executed trajectory, while the merged graph
accumulates recurring structure across episodes.
\STATE Construct the local graph as
$G_i^l\leftarrow\operatorname{LocalMerge}
(\{G_{\tau_n}\}_{n=1}^{N_i})$ using
Algorithm~\ref{alg:memco-local-graph-merge-details}.
\STATE Initialize $\mathcal P_i^l\leftarrow\emptyset$,
$\mathcal E_i^l\leftarrow\emptyset$, and
$\mathcal R_i^l\leftarrow\emptyset$.

\FOR{each trajectory $\tau_n$ and its episodic graph $G_{\tau_n}$}
    \STATE Induce workflows, preconditions, repairs, and failure-pattern
    records from the observed episode structure and add them to
    $\mathcal P_i^l$.
    \STATE Extract state--action--outcome relations into
    $\mathcal E_i^l$, and retain concrete examples, failure diagnoses, and
    recoveries in $\mathcal R_i^l$.
\ENDFOR

\STATE Within $\mathcal P_i^l$ and $\mathcal E_i^l$, consolidate records
with the same local record key and union their occurrence references.

\FOR{each shareable local record
$p\in\mathcal{P}_i^l\cup\mathcal E_i^l$}
    \STATE Initialize occurrence counters
    $c_i^+(p),c_i^-(p),c_i^0(p)$
    to zero and an episode ledger
    $\mathcal J_i(p)\leftarrow\emptyset$.
    \FOR{each observed occurrence $j$ of $p$}
        \STATE Assign signed evidence $e_j(p)\in\{+,-\}$ according to
        whether the observed feedback and transition support or contradict
        the relation encoded by $p$.
        \IF{$e_j(p)=+$}
            \STATE $c_i^+(p)\leftarrow c_i^+(p)+1$.
        \ELSE
            \STATE $c_i^-(p)\leftarrow c_i^-(p)+1$.
        \ENDIF
        \IF{occurrence $j$ has no observable state or task progress}
            \STATE $c_i^0(p)\leftarrow c_i^0(p)+1$.
        \ENDIF
        \STATE Add the signed annotation and optional stall marker to
        $\mathcal J_i(p)
        [\operatorname{scope}(j),\operatorname{episode}(j)]$.
    \ENDFOR
    \STATE Reduce each annotated episode once: contradiction takes
    precedence for ordinary relations, while a stall marker supports an
    anti-pattern. Store the resulting promotion counts as
    $n_i^+(p),n_i^-(p),n_i^0(p)$.
\ENDFOR

\STATE \textbf{Return}
$\mathcal{M}_i^l=
(G_i^l,\mathcal{P}_i^l,\mathcal{E}_i^l,\mathcal{R}_i^l)$
according to Eq.~\ref{eq:local-memory}.
\end{algorithmic}
\end{algorithm*}
\begin{algorithm*}[!htbp]
\caption{Progressive Memory Collaboration and Global Promotion}
\label{alg:promotion-merge}
\small
\begin{algorithmic}[1]
\STATE \textbf{Input:} Shareable local records
$\{\mathcal{P}_i^l\cup\mathcal{E}_i^l\}_{i=1}^{N}$,
promotion threshold $\lambda$, and one-sided tail probability $\alpha$
\STATE \textbf{Return:} Rebuilt global memory $\mathcal{M}^g$
\STATE Initialize an empty global memory $\mathcal M^g\leftarrow\emptyset$.
\FOR{each $p\in\mathcal P_i^l\cup\mathcal E_i^l$ from each agent $i$}
    \STATE Construct its abstract form
    $\bar p\leftarrow\operatorname{Abstract}(p)$ using
    Algorithm~\ref{alg:memco-abstract-details}; keep $p$ local if no
    transferable abstract form is defined.
    \STATE Compute its deterministic global record key $\kappa_g(\bar p)$.
    \STATE Merge $p$ into the aggregate for $\kappa_g(\bar p)$: sum its
    occurrence statistics, union its source provenance, and register its
    episode-reduced counts $n_i^+(p)$, $n_i^-(p)$, and $n_i^0(p)$.
\ENDFOR
\FOR{each abstract-record aggregate $\bar p_{\kappa}$}
    \STATE Aggregate the episode-reduced counts according to
    Eq.~\ref{eq:promotion-count-aggregation} to obtain
    $n^+(\bar p_{\kappa})$, $n^-(\bar p_{\kappa})$, and diagnostic
    $n^0(\bar p_{\kappa})$; occurrence-level counters are not used in the
    Wilson sample.
    \STATE Compute the signed evidence size
    $s\leftarrow n^+(\bar p_{\kappa})+n^-(\bar p_{\kappa})$ and
    $\rho^+\leftarrow n^+(\bar p_{\kappa})/\max(s,1)$.
    \STATE Set $z\leftarrow\Phi^{-1}(1-\alpha)$ and set $W\leftarrow0$ if
    $s=0$; otherwise compute the one-sided Wilson lower bound using
    Eq.~\ref{eq:wilson}.
    \IF{$\bar p_{\kappa}$ passes its record-type structural check,
    passes the applicable low-information check, and $W\geq\lambda$}
        \STATE Update the rebuilt global memory as
        $\mathcal M^g\leftarrow\operatorname{GlobalMerge}
        (\mathcal M^g,\{\bar p_{\kappa}\})$ using
        Algorithm~\ref{alg:memco-global-merge-details}.
    \ENDIF
\ENDFOR
\STATE Atomically publish the complete rebuilt snapshot for subsequent
queries; do not change the memory context of a decision already in progress.
\STATE \textbf{Return} $\mathcal{M}^g$.
\end{algorithmic}
\end{algorithm*}

\begin{algorithm*}[!htbp]
\caption{Query-guided Local--Global Memory Support}
\label{alg:routing-current}
\small
\begin{algorithmic}[1]
\STATE \textbf{Input:} Task goal $\boldsymbol g$,
observation $\boldsymbol x_t$, available actions $\mathcal A_t$,
interaction history $\tau_{<t}$, local memory $\mathcal M_i^l$,
global memory $\mathcal M^g$, and retrieval caps $k^l,k^g$
\STATE \textbf{Return:} Memory context
$\mathcal C_t=
[c_t^{\mathrm{ground}},
c_t^{\mathrm{workflow}},
c_t^{\mathrm{reflection}}]$

\STATE Construct the current decision query
$\boldsymbol q_t
\leftarrow
\mathcal Q(
\boldsymbol g,\boldsymbol x_t,\mathcal A_t,\tau_{<t})$
following the query definition in Section~\ref{sec:3.5}.

\STATE Retrieve local memory using
$\mathcal S_t^l
\leftarrow
\mathbb R(
\boldsymbol q_t,\mathcal M_i^l,k^l)$.

\STATE Retrieve global memory using the same function:
$\mathcal S_t^g
\leftarrow
\mathbb R(
\boldsymbol q_t,\mathcal M^g,k^g)$.

\STATE Use Algorithm~\ref{alg:memco-bounded-selection} for both retrieval
calls, consistent with Eq.~\ref{eq:two-streams}.

\STATE Define
$\mathcal H=
\{\mathrm{ground},\mathrm{workflow},\mathrm{reflection}\}$.

\STATE Compose role-specific support as
$\{\mathcal U_t^h\}_{h\in\mathcal H}
\leftarrow
\operatorname{Compose}\!\left(
\mathcal S_t^l,\mathcal S_t^g
\mid\boldsymbol q_t,\mathcal A_t
\right)$
using Algorithm~\ref{alg:compose}.

\STATE Render the selected support as
$\mathcal C_t
\leftarrow
\operatorname{Render}\!\left(
\boldsymbol q_t,
\{\mathcal U_t^h\}_{h\in\mathcal H}
\right)$
using Algorithm~\ref{alg:render}.

\STATE \textbf{Return} $\mathcal C_t$.
\end{algorithmic}
\end{algorithm*}

\FloatBarrier

\subsection{Local Graph Representation and Record Evidence}
\label{app:memco-local-details}

\subsubsection{Trajectory and Episodic Graph Representation}

Given an interaction episode
\begin{equation}
\tau=
\{(\boldsymbol x_t,\boldsymbol a_t,\boldsymbol f_t,
\boldsymbol x_{t+1})\}_{t=1}^{T},
\end{equation}
\modelname{} normalizes the interaction records and constructs an episodic
graph
\begin{equation}
G_\tau=(V_\tau,E_\tau).
\end{equation}

At the method level, state nodes $v_t^x$, action nodes $v_t^a$, and feedback
nodes $v_t^f$ represent the normalized observation, parsed operation, and
environment feedback, respectively.
Directed edges connect them according to the temporal interaction:
$v_t^x\rightarrow v_t^a\rightarrow v_t^f\rightarrow v_{t+1}^x$.
The feedback node records the observed outcome associated with the executed
action, including state progress, failure, or task completion information.

In the serialized implementation, $v_t^f$ is not materialized as a separate
stored node for every transition.
Instead, its information is attached to the corresponding transition through
edge statistics and, when present, typed state-delta, subgoal, and failure
relations together with the record evidence ledger.
The stored graph backbone is therefore
$v_t^x\rightarrow v_t^a\rightarrow v_{t+1}^x$,
while preserving the feedback information represented by $v_t^f$ in the
paper-level episodic graph.
The structural merge below operates on this stored realization of the same
episodic graph; we retain the notation $G_\tau$ because the feedback semantics
above are preserved in its attached transition fields.

Record induction remains anchored to the observed order of each completed
episode and its temporary episodic graph. The persistent merged graph described
next is a structural index: it does not retain per-edge episode membership and
is not used to reconstruct an episode-consistent trajectory after merging.
\subsubsection{Structural Matching and Local Graph Merge}

To connect related experience across episodes, \modelname{} compares graph
elements using structural equality after their contents have been normalized. Two nodes are considered equivalent only when both their node
types and normalized contents match, while two edges are equivalent only when
they share the same relation and equivalent endpoints. Matched elements
accumulate their statistics. In the serialized representation, normalized
feedback associated with a transition is compared through its attached
transition payload and typed relations rather than through a separately stored
feedback node. The local memory records the completed episode identifiers at
store level, but not a membership list on every merged node or edge.
For each agent $i$, equivalent graph elements across its episodic trajectories
are aligned to construct
\begin{equation}
G_i^l=\operatorname{LocalMerge}
\left(\{G_{\tau_n}\}_{n=1}^{N_i}\right).
\label{eq:memco-app-graph-merge}
\end{equation}
The resulting graph therefore provides a connected aggregate structural index
across trajectories rather than a collection of isolated action records. An
adjacent path in this merged index is not asserted to be a newly observed
single-episode trajectory.
This operation acts on graph nodes and edges before abstraction; it is not the
global exact-key record update in
Algorithm~\ref{alg:memco-global-merge-details}.

\begin{algorithm*}[!htbp]
\caption{$\operatorname{LocalMerge}$: Structural Merge of Episodic Graphs}
\label{alg:memco-local-graph-merge-details}
\small
\begin{algorithmic}[1]
\STATE \textbf{Input:} Episodic graphs
$\{G_{\tau_n}=(V_{\tau_n},E_{\tau_n})\}_{n=1}^{N_i}$ for agent $i$
\STATE \textbf{Return:} Merged local graph $G_i^l=(V(G_i^l),E(G_i^l))$
\STATE Initialize $V(G_i^l)\leftarrow\emptyset$, $E(G_i^l)\leftarrow\emptyset$,
and an empty map from episodic node identifiers to matching local nodes.
\FOR{each episodic graph $G_{\tau_n}$}
    \FOR{each node $v\in V_{\tau_n}$}
        \STATE Form the local node key
        $\kappa_v(v)=(\operatorname{type}(v),
        \operatorname{Normalize}(\operatorname{content}(v)))$.
        \IF{no node in $V(G_i^l)$ has key $\kappa_v(v)$}
            \STATE Insert a local copy of $v$ into $V(G_i^l)$.
        \ELSE
            \STATE Reuse the matching local node and accumulate node
            statistics.
        \ENDIF
        \STATE Map the episodic identifier of $v$ to that local node.
    \ENDFOR
    \FOR{each edge $e=(u,\eta,v)\in E_{\tau_n}$}
        \STATE Resolve $u$ and $v$ to their matching local nodes and form
        $\kappa_e(e)=(\kappa_v(u),\eta,\kappa_v(v))$.
        \IF{no edge in $E(G_i^l)$ has key $\kappa_e(e)$}
            \STATE Insert a local edge with relation $\eta$ into $E(G_i^l)$.
        \ELSE
            \STATE Reuse the matching edge and accumulate edge statistics;
            retain the payload of the stored signature entry.
        \ENDIF
    \ENDFOR
\ENDFOR
\STATE \textbf{Return} $G_i^l=(V(G_i^l),E(G_i^l))$.
\end{algorithmic}
\end{algorithm*}

\subsubsection{Type-Specific Record Extraction}
\label{app:memco-memory-type-definitions}

The local graph exposes recurring structures across episodes, from which
\modelname{} induces reusable records. Record induction remains anchored to the
observed order of the contributing episodes, so it does not mine arbitrary
synthetic paths from the merged graph. Recurring records with the same
normalized content are consolidated so that repeated observations contribute
evidence to the same local entry.

Specifically, a \textit{workflow} pattern captures an ordered sequence of
progress phases and associated operations; a \textit{precondition} pattern
records a temporal transition associated with successful progress; a
\textit{repair} pattern records a successful action that resolves an observed
failure; and a \textit{failure} pattern records an action associated with a
failure or no-progress outcome. Here, a precondition is an experience-derived
condition for progress, not a claim that a logically necessary or sufficient
operator precondition has been proved.

The other local components in Eq.~\ref{eq:local-memory} retain complementary
information. The graph $G_i^l$ preserves connected state--action--feedback
experience across episodes. Actionable knowledge $\mathcal E_i^l$ stores
state--action--outcome relations that can ground an executable action, while
concrete cases and reflections $\mathcal R_i^l$ retain episode-specific
examples, failure diagnoses, and recoveries. These components may describe
related transitions but serve different purposes: patterns summarize recurring
structure, actionable knowledge supports a concrete decision, and
$\mathcal R_i^l$ keeps instance-level context.

Extracted records with the same record type and normalized content are
consolidated into a single entry in $\mathcal P_i^l$; actionable records in
$\mathcal E_i^l$ are consolidated analogously by their local normalized keys.
This record consolidation is distinct from graph-element merging and from
global promotion. Only selected records in
$\mathcal P_i^l\cup\mathcal E_i^l$ are considered for sharing; $G_i^l$ and
$\mathcal R_i^l$ remain local. Table~\ref{tab:memco-memory-types} summarizes
the corresponding memory types and sharing scope.
\subsubsection{Record Evidence and Local Counts}
\label{app:memco-evidence-counts}

Evidence is defined relative to the relation encoded by each shareable record
$p\in\mathcal P_i^l\cup\mathcal E_i^l$. A record may be observed multiple times
within one episode or across different episodes. For its $j$-th occurrence,
we assign a signed evidence label
\begin{equation}
e_j(p)\in\{+,-\},
\label{eq:memco-app-evidence-sign}
\end{equation}
where $e_j(p)=+$ means that the observed transition supports the relation
encoded by $p$, and $e_j(p)=-$ means that it contradicts that relation.

\paragraph{Occurrence-level counters.}
The occurrence counters summarize all matched observations of the local record:
\begin{equation}
c_i^\omega(p)=\sum_j\mathbb I[e_j(p)=\omega],
\qquad \omega\in\{+,-\}.
\label{eq:memco-app-signed-counts}
\end{equation}
No-progress behavior is tracked as an auxiliary, potentially overlapping
statistic,
\begin{equation}
c_i^0(p)=
\sum_j
\mathbb I[\operatorname{NoProgress}(j,p)].
\label{eq:memco-app-stall-count}
\end{equation}
A no-progress occurrence does not by itself determine the sign $e_j(p)$: the
same observed outcome can contradict a workflow that claims progress while
supporting a failure or anti-pattern whose relation predicts no progress.
The counters $c_i^+(p)$, $c_i^-(p)$, and $c_i^0(p)$ are retained as
occurrence-level evidence for local record statistics and retrieval-side
ranking. They are \emph{not} the sample counts used by the Wilson promotion
rule.

\paragraph{Episode-reduced promotion counters.}
For global promotion, repeated occurrences of the same local record within one
episode are first reduced to at most one signed vote. Let
$\mathcal J_i(p)$ be the evidence ledger keyed by the local record's source
scope and episode identifier. A deterministic, relation-aware reducer maps the
annotations under each episode key $\xi$ to
$\widetilde e_\xi(p)\in\{+,-\}$ or to no vote. For ordinary progress relations,
a contradiction takes precedence over support within the same episode. For a
failure/no-progress relation, an observed stall supports the asserted
anti-pattern rather than being treated as successful execution. An unannotated
episode contributes no signed vote. The promotion counters are
\begin{equation}
n_i^\omega(p)=\sum_{\xi\in\mathcal J_i(p)}
\mathbb I[\widetilde e_\xi(p)=\omega],\qquad \omega\in\{+,-\}.
\label{eq:memco-app-episode-counts}
\end{equation}
The diagnostic stall count is
\begin{equation}
n_i^0(p)=\sum_{\xi\in\mathcal J_i(p)}
\mathbb I[z_\xi(p)=1],
\label{eq:memco-app-episode-stalls}
\end{equation}
where $z_\xi(p)$ indicates that episode key $\xi$ contains a no-progress
marker. Like $c_i^0(p)$, $n_i^0(p)$ may overlap the signed counts. It is never
added to the Wilson sample size.

The two levels therefore have distinct roles:
$$
\{e_j(p),\operatorname{NoProgress}(j,p)\}_j
\;\rightarrow\;
(c_i^+,c_i^-,c_i^0,\mathcal J_i)
\;\rightarrow\;
(n_i^+,n_i^-,n_i^0)
\;\rightarrow\;
(n^+(\bar p),n^-(\bar p))
\;\rightarrow\;
\rho^+(\bar p),\,W(\bar p).
$$
The first counters summarize occurrences and may affect retrieval ranking.
For promotion, annotations associated with the same local record, source
scope, and episode are first reduced to at most one signed vote, producing
$n_i^+(p)$ and $n_i^-(p)$. Equation~\ref{eq:promotion-count-aggregation}
then aggregates these episode-reduced counts over local records that map to
the same abstract record. Wilson uses only
$s(\bar p)=n^+(\bar p)+n^-(\bar p)$ and
$\rho^+(\bar p)=n^+(\bar p)/\max\{s(\bar p),1\}$.
Neither the occurrence counters $c_i^\omega$ nor the diagnostic count $n^0$
enters $s(\bar p)$ or the Wilson formula. Consequently, repeated occurrences
of the same local record within one episode do not directly multiply the
Wilson evidence.

\begin{table*}[htbp]
\centering
\small
\begin{tabularx}{\linewidth}{@{}p{0.23\linewidth}XX@{}}
\toprule
\textbf{Record type} & \textbf{Supporting evidence} &
\textbf{Contradicting evidence}\\
\midrule
Workflow & The observed sequence makes the claimed progress and reaches the
recorded outcome. & The sequence is applicable but fails to make the claimed
progress or reaches an incompatible outcome.\\
Precondition / actionable relation & The condition--action relation is observed
with compatible progress or outcome. & The same relation is applicable but the
observed outcome is incompatible with the stored relation.\\
Repair & The repair follows the relevant failure context and restores progress.
& The proposed repair is applicable but does not resolve the failure context.\\
Failure / anti-pattern & The observed action reproduces the stored failure or
no-progress relation. & The action instead produces compatible progress.\\
\bottomrule
\end{tabularx}
\caption{Record-level evidence used before episode reduction. Evidence is
assigned relative to the relation encoded by the record; the same environment
outcome can therefore support one relation while contradicting another.}
\label{tab:memco-evidence-reduction}
\end{table*}

\subsubsection{Deterministic Evidence Assignment}
\label{app:memco-evidence-assignment}

Evidence signs are assigned from normalized interaction signals rather than by
an additional LLM call. Each environment interface maps its native observations
and outcomes to the common information required by the method, including action
admissibility, state or task progress, failure, and terminal success. These
normalized signals are interpreted relative to the relation encoded by the
record to produce $e_j(p)$ and the optional no-progress marker. The evidence
semantics, episode reduction, aggregation, and Wilson promotion rule are shared
across evaluation environments; only the interface that exposes the normalized
state, action, and feedback fields depends on the environment.

\FloatBarrier

\subsection{Progressive Memory Collaboration Details}
\label{app:memco-collaboration-details}
\label{app:abstraction-details}

\subsubsection{Role-Level Abstraction}
\label{app:memco-abstraction-example}
\label{app:memco-role-abstraction}

Before global aggregation, \modelname{} applies
Algorithm~\ref{alg:memco-abstract-details} only to selected records
$p\in\mathcal P_i^l\cup\mathcal E_i^l$. The abstraction preserves the relation
that makes a record useful---its type, applicability conditions, action order or
state--action relation, and outcome---while removing details that identify only
the source episode or scene. When a concrete entity has a supported semantic
role in the task representation, the stored entity can be replaced by that
role; otherwise it remains an applicability constraint that must be matched
before use. Records for which no transferable abstract form can be constructed
remain local. Neither $G_i^l$ nor $\mathcal R_i^l$ is promoted as a global copy.

\paragraph{Type preservation and shared-record fields.}
Global memory is therefore a collection of promoted abstractions of
$\mathcal P_i^l$ and $\mathcal E_i^l$, not a duplicate of the complete local
memory. For a selected record $p$, $\bar p=\operatorname{Abstract}(p)$ retains
the fields that define the transferable relation, while source-episode and
scene-specific anchors are removed. A promoted entry additionally carries its
aggregated evidence, promotion score, and source provenance as specified in
Algorithm~\ref{alg:promotion-merge}. Abstraction preserves operational
conditions rather than turning a conditional relation into an unconditional
rule. Consequently, a shared record is still checked and grounded against the
receiving agent's current decision context before it is used.

Each abstracted pattern or actionable record is assigned the deterministic
global key
\begin{equation}
\kappa_g(\bar p)
=\operatorname{StoredRecordID}
\bigl(\operatorname{type}(\bar p),
\operatorname{sig}(\operatorname{relation}(\bar p))\bigr),
\label{eq:memco-abstract-key}
\end{equation}
where $\operatorname{relation}(\bar p)$ denotes the retained applicability
condition, ordered actions or state--action relation, and outcome fields for
the record. Evidence counts, display-only summaries, source sets, and removed
local anchors do not participate in the identifier. Thus records are merged
only when their transferable relation matches exactly after abstraction.

Consider a schematic household episode in which an agent takes a kettle from a
counter, fills it at a sink, and heats it on a stove burner. The local
trajectory is first converted into a workflow record and is then
abstracted into a role-level representation before promotion.
Table~\ref{tab:memco-abstraction-example} shows which information is retained
and which local anchors are removed.

\begin{table*}[htbp]
\centering
\begin{tabularx}{\linewidth}{@{}>{\raggedright\arraybackslash}p{0.24\linewidth}>{\raggedright\arraybackslash}X@{}}
\toprule
\textbf{Field} & \textbf{Value}\\
\midrule
Local trajectory
& go to counter 2; take kettle 1; go to sink 1; fill kettle 1; go to
stoveburner 1; heat kettle 1.\\
Record type
& workflow.\\
Condition
& target\_object is fillable and heatable.\\
Action sequence
& find(source = water); fill(target\_object, water);
move(target\_object, heating\_device); heat(target\_object).\\
Outcome
& target\_object becomes heated.\\
Removed local details
& kettle 1; counter 2; sink 1; stoveburner 1; local graph-node ids; raw
observation strings.\\
Global workflow fields: metadata
& type = workflow; condition = fillable-heatable target\_object;
outcome = heated target\_object.\\
Global workflow fields: actions
& action 1 = find water source; action 2 = fill target; action 3 = move to
heating device; action 4 = heat target.\\
\bottomrule
\end{tabularx}
\caption{Schematic household abstraction example. The action strings
illustrate the existing abstraction fields rather than a verbatim benchmark
rollout. Concrete entities are removed, while the workflow type, abstract
condition, action order, and outcome relation are retained.}

\label{tab:memco-abstraction-example}
\label{tab:abstraction-example}
\end{table*}

\begin{algorithm*}[!htbp]
\caption{$\operatorname{Abstract}$: {Role-Level Abstraction}}
\label{alg:memco-abstract-details}
\small
\begin{algorithmic}[1]
\STATE {\textbf{Input:} Shareable local record
$p\in\mathcal P_i^l\cup\mathcal E_i^l$ and its record-type abstraction
schema}
\STATE {\textbf{Return:} Abstract form $\bar p$ with attached
occurrence and episode evidence, or $\varnothing$ when $p$ remains local}
\STATE {Initialize $\bar p$ from the transferable pattern or actionable
relation fields of $p$.}
\IF{no transferable abstract form is defined for $p$}
    \STATE {\textbf{Return} $\varnothing$.}
\ENDIF
\STATE {Preserve the record type, applicability condition, action order, and
outcome relation defined by $p$.}
\STATE {Replace a stored entity by a task role when the task representation
supports that mapping; otherwise retain it as an applicability constraint.}
\STATE {Remove the designated episode, scene, layout, raw-observation, and
local graph-reference fields for this record type.}
\STATE {Preserve the retained operational relation, action order,
occurrence statistics, source sets, and source-scoped episode ledger.}
\STATE {Compute $\kappa_g(\bar p)$ using
Eq.~\ref{eq:memco-abstract-key}.}
\STATE {\textbf{Return} $\bar p$.}
\end{algorithmic}
\end{algorithm*}
\FloatBarrier
\subsubsection{Pre-Promotion Record Checks}
\label{app:memco-promotion-eligibility}

The Wilson test is applied only to records that first pass two deterministic
checks. Let $A_{\mathrm{str}}(\bar p)$ denote structural completeness and
$A_{\mathrm{info}}(\bar p)$ denote the low-information check. A record proceeds
to Wilson scoring only when $A_{\mathrm{str}}=A_{\mathrm{info}}=1$; all
shareable record types then use the same Wilson formula and threshold
$\lambda$.

\begin{table*}[htbp]
\centering
\small
\begin{tabularx}{\linewidth}{@{}p{0.27\linewidth}X@{}}
\toprule
\textbf{Record relation} & \textbf{Required transferable content}\\
\midrule
Workflow & A nonempty applicability condition or task role, an ordered action
or phase sequence, and an outcome.\\
Precondition & A condition, an action relation, and evidence of compatible
progress.\\
Repair / failure relation & A triggering failure or no-progress condition and
the corresponding repair or avoidance relation.\\
Actionable knowledge & A usable condition--action--outcome relation that can be
bound to the receiving decision context.\\
\bottomrule
\end{tabularx}
\caption{Structural completeness checks applied before Wilson promotion. The
checks retain only records with sufficient transferable content; incomplete,
source-bound, or uninformative records remain local.}
\label{tab:memco-promotion-eligibility}
\end{table*}

The low-information check removes records whose abstract form contains no
usable condition, relation, or outcome beyond generic interaction advice. These
checks operate before evidence calibration and do not change the Wilson score
or threshold. Source provenance and retrieval relevance are retained for later
aggregation or decision-time use rather than used as alternative promotion
thresholds.

\subsubsection{Global Exact-Key Merge of Promoted Records}
\label{app:memco-merge-example}

Global construction does not use embedding similarity or an LLM to decide
whether two records are equivalent. It first constructs the applicable abstract
form for each selected $p\in\mathcal P_i^l\cup\mathcal E_i^l$ and then
groups the resulting pattern and actionable representations by the key in
Eq.~\ref{eq:memco-abstract-key}. Evidence counts do not
participate in identifier matching. The Wilson implementation rebuilds each
snapshot from all current local memories, so occurrence statistics and source
sets are recomputed from those local records rather than cumulatively added to
an earlier global report. Episode reduction is completed before global
aggregation, and Eq.~\ref{eq:promotion-count-aggregation} then sums the
resulting episode-reduced signed counts for records with the same exact
abstract key. Occurrence counts are therefore never substituted for the
promotion counts. The accompanying diagnostic is aggregated as
\begin{equation}
n^0(\kappa)=
\sum_{\substack{i\in[N],\,p\in\mathcal P_i^l\cup\mathcal E_i^l:
\kappa_g(\operatorname{Abstract}(p))=\kappa}}
n_i^0(p),
\label{eq:memco-global-stall-aggregation}
\end{equation}
where $\kappa$ is the global record key. Only
$n^+(\kappa)+n^-(\kappa)$ enters the Wilson sample size. This is the same
episode-reduction-then-aggregation order specified by
Eq.~\ref{eq:promotion-count-aggregation}.

After aggregation, each record is evaluated by the structural,
low-information, and Wilson gates. Because every rebuild starts from an empty
global memory, a record that no longer passes is absent from the new snapshot;
every retained global record therefore satisfies the promotion rule under the
current local evidence.

The exact-record-key rule prevents overly broad merging. A
filling workflow and a cleaning workflow may both mention a water source, but
they are not the same pattern entry. Tables~\ref{tab:memco-positive-merge}
and~\ref{tab:memco-negative-merge} retain concrete and global examples for
readability and illustrate normalized workflow-identifier matching rather
than surface semantic overlap.

\begin{table}[htbp]
\centering
\begin{tabularx}{\linewidth}{@{}>{\raggedright\arraybackslash}p{0.21\linewidth}>{\raggedright\arraybackslash}X>{\raggedright\arraybackslash}X@{}}
\toprule
\textbf{Field} & \textbf{Kitchen record} & \textbf{Bathroom record}\\
\midrule
Concrete entities
& cloth, kitchen sink, countertop.
& soap, bathroom basin, vanity.\\
Record type
& workflow.
& workflow.\\
Condition after abstraction
& cleanable target\_object.
& cleanable target\_object.\\
Action sequence after abstraction
& find water source; clean target\_object; place in goal\_destination.
& find water source; clean target\_object; place in goal\_destination.\\
Outcome after abstraction
& stored clean target\_object.
& stored clean target\_object.\\
\bottomrule
\end{tabularx}
\par\smallskip
\noindent\textbf{Merge decision.}
Merged as one promoted workflow record in $\mathcal M^g$. The two
$\mathcal P_i^l$ records have the same $\kappa_g$ after induction-time role
normalization. Snapshot construction
recomputes the source set and evidence totals from all current local records.
\caption{Positive merge example. Different local objects can map to the same
global workflow when the record fields match exactly.}
\label{tab:memco-positive-merge}
\label{tab:merge-positive-example}
\end{table}

\begin{table}[htbp]
\centering
\begin{tabularx}{\linewidth}{@{}>{\raggedright\arraybackslash}p{0.21\linewidth}>{\raggedright\arraybackslash}X>{\raggedright\arraybackslash}X@{}}
\toprule
\textbf{Field} & \textbf{Filling workflow} & \textbf{Cleaning workflow}\\
\midrule
Condition
& fillable target\_object.
& cleanable target\_object.\\
Action sequence
& find water source; fill target\_object.
& find water source; clean target\_object.\\
Outcome
& filled target\_object.
& clean target\_object.\\
\bottomrule
\end{tabularx}
\par\smallskip
\noindent\textbf{Merge decision.}
Not merged. The entries share the phrase \textit{water source}, but their
normalized workflow identifiers differ because their action sequences and
outcomes differ.
\caption{Negative merge example. Surface overlap is insufficient for merging.}
\label{tab:memco-negative-merge}
\label{tab:merge-negative-example}
\end{table}

Algorithm~\ref{alg:memco-global-merge-details} implements
$\operatorname{GlobalMerge}$ in Eq.~\ref{eq:global-exact-merge}.
Evidence is aggregated before promotion; this update stores the current
totals under the appropriate global record key. In contrast,
$\operatorname{LocalMerge}$ acts on graph elements
(Algorithm~\ref{alg:memco-local-graph-merge-details}).
\begin{algorithm*}[!htbp]
\caption{$\operatorname{GlobalMerge}$: Exact-Key Update of Promoted Global Records}
\label{alg:memco-global-merge-details}
\small
\begin{algorithmic}[1]
\STATE {\textbf{Input:} Global memory $\mathcal M^g$ and
abstract aggregates $\bar p$ that pass the pre-promotion checks and satisfy
$W(\bar p)\geq\lambda$, carrying the
aggregated evidence from Eq.~\ref{eq:promotion-count-aggregation}}
\STATE {\textbf{Return:} Updated global memory $\mathcal M^g$}
\FOR{each promoted aggregate $\bar p$}
    \STATE {Compute its global record key $\kappa_g(\bar p)$ using
    Eq.~\ref{eq:memco-abstract-key}.}
    \IF{no stored record has key $\kappa_g(\bar p)$}
        \STATE {Insert $\bar p$, its current source set, promotion
        score, and aggregated evidence
        into $\mathcal M^g$.}
    \ELSE
        \STATE {Keep one record for the matching global key and
        replace its mutable fields with the current source set, promotion
        score, occurrence statistics, and episode-reduced counts
        $n^+(\bar p)$, $n^-(\bar p)$, and $n^0(\bar p)$.}
    \ENDIF
\ENDFOR
\STATE {\textbf{Return} $\mathcal M^g$.}
\end{algorithmic}
\end{algorithm*}

\paragraph{Snapshot realization.}
Algorithm~\ref{alg:promotion-merge} constructs the global snapshot from the
current local memories by reducing episode evidence, grouping records by the
exact global key $\kappa_g$, and applying the pre-promotion and Wilson checks.
A complete rebuilt snapshot is then made available to subsequent decisions.
The acquisition and evaluation lifecycle is summarized in
Appendix~\ref{app:memco-evaluation-scope}.

\FloatBarrier

\subsection{Query-Guided Memory Support Details}
\label{app:memco-retrieval-details}

\subsubsection{Local Graph Query Anchoring}
\label{app:memco-local-graph-retrieval}

The local graph is queried as a bounded structural index. State anchors are
matched to the current query using normalized state content together with
location, visible or held entities, task progress, and goal-related cues when
available. From each retained anchor, retrieval follows only a bounded local
neighborhood of the observed state--action--feedback--next-state transition.
In the serialized graph, this is the corresponding
state--action--next-state backbone together with its attached feedback fields.
The resulting transition entries summarize concrete local evidence that can
support the current decision; retrieval does not search arbitrary paths or
construct a new long workflow from the merged graph.

Graph-derived entries and records from $\mathcal P_i^l$, $\mathcal E_i^l$,
and $\mathcal R_i^l$ enter the same bounded local selection stage. Their graph
provenance can strengthen a record when the retrieved state, action, or outcome
matches the current query, while incompatible or stale context lowers its
priority. All such entries remain subject to the common state, progress, goal,
and action-compatibility checks below.

\subsubsection{Memory Entry Construction, Matching, and Ranking}
\label{app:memco-relevance-computation}

At each decision step, \modelname{} constructs the query defined in
Section~\ref{sec:3.5} from the current goal, observation, available actions,
and interaction history. The same retrieval function $\mathbb R$ is called
independently on local and global memory according to
Eq.~\ref{eq:two-streams}. Local retrieval draws from the local records and
bounded graph neighborhoods described above, whereas global retrieval draws
from promoted abstract records in $\mathcal M^g$.

Candidate entries are evaluated through the three complementary relevance
facets in Eq.~\ref{eq:memco-relevance}:
\begin{equation*}
r(m,\boldsymbol q_t)
=r_x(m,\boldsymbol q_t)+r_p(m,\boldsymbol q_t)+r_g(m,\boldsymbol q_t).
\end{equation*}
Here $r_x$ measures consistency with the current state and available actions,
$r_p$ measures compatibility with task progress, and $r_g$ measures alignment
with the goal and its functional roles. These facets score the same memory
entry from complementary views; they are not independent pass/fail gates.
Each facet is computed by deterministic comparisons over the normalized
fields available to the entry: compatible matches contribute positively,
explicit mismatches reduce the score, and unavailable optional fields
contribute no update. The reported outer weights of the three facets are one.
These comparisons use the common normalized representation rather than a
benchmark-specific retrieval objective.

\begin{center}
{\small
\setlength{\tabcolsep}{4pt}
\renewcommand{\arraystretch}{1.12}
\begin{tabularx}{\linewidth}{@{}>{\raggedright\arraybackslash}p{0.12\linewidth}>{\raggedright\arraybackslash}p{0.20\linewidth}>{\raggedright\arraybackslash}X@{}}
\toprule
\textbf{Facet} & \textbf{Main emphasis} & \textbf{Examples of matching evidence}\\
\midrule
$r_x$ & Current state and action context
& State or location consistency, visible or held entities, applicable
conditions, and compatibility with currently available actions.\\[4pt]
$r_p$ & Task and execution progress
& Task stage, completed or remaining subgoals, failure context, and whether the
stored procedure is appropriate at the current decision phase.\\[4pt]
$r_g$ & Goal and functional roles
& Alignment between stored object, destination, tool, or action roles and the
current task goal.\\
\bottomrule
\end{tabularx}}
\end{center}

Applicability constraints are handled before ranking: an entry is removed when
a required condition contradicts the current state or when a required action
cannot be matched to $\mathcal A_t$. These hard compatibility checks are
distinct from the soft preferences used to rank the remaining entries.

The joint relevance score is then combined with stored evidence, current
action binding, source-specific context, and soft conflict penalties to obtain
the full rank used for bounded selection:
\begin{equation*}
s_{\mathrm{rank}}(m,\boldsymbol q_t)
=r(m,\boldsymbol q_t)+\Delta_{\mathrm{ev}}(m)
+\Delta_{\mathrm{bind}}(m,\boldsymbol q_t)
+\Delta_{\mathrm{src}}(m,\boldsymbol q_t)
-p_{\mathrm{soft}}(m,\boldsymbol q_t).
\end{equation*}
Here $\Delta_{\mathrm{ev}}$ summarizes the occurrence-level evidence retained
with the entry, $\Delta_{\mathrm{bind}}$ reflects compatibility between the
stored operation or semantic roles and the currently available actions,
$\Delta_{\mathrm{src}}$ incorporates source-specific supporting context
(such as query-matched local graph context for local entries or retained
applicability and role context for global entries), and
$p_{\mathrm{soft}}$ collects non-disqualifying conflicts such as stale or
weakly mismatched guidance. These terms are deterministic adjustments over
the same normalized fields used by the common retrieval interface.
Occurrence-level statistics can contribute to retrieval ranking, whereas the
episode-reduced counts $n^+,n^-$ are reserved for Wilson promotion. The Wilson
score therefore controls membership in global memory rather than being added
directly to the online relevance score.

\subsubsection{Bounded Selection}
\label{app:memco-selection-details}
Algorithm~\ref{alg:memco-bounded-selection} combines source-aware pool
construction, compatibility filtering, joint relevance, and bounded selection.
The caps $k^l$ and $k^g$ bound the retrieved sets. Before rank-based filling,
selection preserves at most one strongest supporting match for each relevance
facet, when such a match is available, so that one dominant score does not
remove complementary state-, progress-, or goal-related support. A supporting
facet match is one with $r_d(m,\boldsymbol q_t)>0$ for
$d\in\{x,p,g\}$. Remaining positions are filled by decreasing full rank.
Composition may subsequently reject or deduplicate selected entries after
current-environment grounding.

\begin{algorithm*}[!htbp]
\caption{$\mathbb R$: Source-Aware Retrieval, Filtering, and Bounded Selection}
\label{alg:memco-bounded-selection}
\footnotesize
\begin{algorithmic}[1]
\STATE {\textbf{Input:} Query $\boldsymbol q_t$, one memory
$\mathcal M\in\{\mathcal M_i^l,\mathcal M^g\}$, and a nonnegative cap $k$}
\STATE {\textbf{Return:} Selected memory entries $\mathcal S$}
\STATE Initialize the retrieval pool $\mathcal B\leftarrow\emptyset$.
\IF{$\mathcal M=\mathcal M_i^l=(G_i^l,\mathcal P_i^l,
\mathcal E_i^l,\mathcal R_i^l)$}
    \STATE Add records from $\mathcal P_i^l$, $\mathcal E_i^l$, and
    $\mathcal R_i^l$ to $\mathcal B$.
    \STATE Add bounded transition entries from query-matched neighborhoods
    of $G_i^l$, preserving the paper-level state--action--feedback--next-state
    semantics and the attached serialized feedback.
\ELSE
    \STATE Add the promoted abstract records in $\mathcal M^g$ to
    $\mathcal B$.
\ENDIF
\STATE Initialize the eligible scored pool $\mathcal V\leftarrow\emptyset$.
\FOR{each memory entry $m\in\mathcal B$}
    \STATE Compute $r_x(m,\boldsymbol q_t)$,
    $r_p(m,\boldsymbol q_t)$, and $r_g(m,\boldsymbol q_t)$.
    \IF{$m$ contradicts a required current condition or a required action
    cannot be matched to $\mathcal A_t$}
        \STATE Exclude $m$ from $\mathcal V$.
    \ELSE
        \STATE Compute the full rank from joint relevance, stored evidence,
        current action binding, source context, and soft conflicts.
        \STATE Add $m$ and its facet scores, full rank, provenance, and
        available action/graph evidence to $\mathcal V$.
    \ENDIF
\ENDFOR
\STATE Deduplicate equivalent entries in $\mathcal V$, keeping the
higher-ranked applicable entry, and sort by decreasing full rank.
\IF{$|\mathcal V|\leq k$}
    \STATE \textbf{Return} $\mathcal V$.
\ENDIF
\STATE Initialize $\mathcal S\leftarrow[\,]$.
\FOR{each relevance facet $d\in\{x,p,g\}$}
    \IF{$|\mathcal S|<k$ and an unselected entry satisfies
    $r_d(m,\boldsymbol q_t)>0$}
        \STATE Append the unselected entry with the largest $r_d$,
        breaking ties by full rank.
    \ENDIF
\ENDFOR
\FOR{each remaining entry in decreasing full-rank order}
    \IF{$|\mathcal S|<k$}
        \STATE Append the entry to $\mathcal S$ without duplication.
    \ENDIF
\ENDFOR
\STATE Retain provenance, facet scores, full rank, and action/graph evidence
for composition.
\STATE \textbf{Return} $\mathcal S$.
\end{algorithmic}
\end{algorithm*}

\FloatBarrier

\subsubsection{Composition and Rendering}
\label{app:memco-composition-details}
Retrieval and composition serve different roles. Retrieval selects relevant
entries within each memory source, whereas composition determines how the
selected local and global entries should be used in the current decision. With
$\mathcal H=\{\mathrm{ground},\mathrm{workflow},\mathrm{reflection}\}$,
grounding uses current-state facts and graph relations; workflow support uses
plans, procedures, and preconditions; and reflection uses failure patterns,
repairs, and blocked actions. Entries whose content is incompatible with the
current role or decision context are removed.

For a global record, composition first binds its abstract object, destination,
tool, and action roles to the current query and available actions. A global
record is not emitted when a required binding cannot be established. Local
state constraints and concrete local grounding take precedence over an
incompatible abstract recommendation. The remaining support is deduplicated,
organized by role, and kept within a bounded context budget before rendering.
This stage therefore converts shared knowledge into current-environment support
rather than directly copying a stored global record into the prompt.

The renderer serializes the admitted support into the three memory-bearing
fields in Eq.~\ref{eq:adaptive-context}. It keeps only information needed for
the current decision and uses the current bindings for grounded global records
rather than source-environment anchors. The full prompt template is shown in
Appendix~\ref{app:memco-prompts}.

\begin{algorithm*}[!htbp]
\caption{$\operatorname{Compose}$: Query-Conditioned Local--Global Support Composition}
\label{alg:compose}
\small
\begin{algorithmic}[1]
\STATE \textbf{Input:} Selected memory entries $\mathcal S_t^l,\mathcal S_t^g$,
query $\boldsymbol q_t$, actions $\mathcal A_t$, roles
$\mathcal H=\{\mathrm{ground},\mathrm{workflow},\mathrm{reflection}\}$,
and the configured context budget
\STATE \textbf{Return:} Role-specific supports
$\{\mathcal U_t^h\}_{h\in\mathcal H}$
\STATE Initialize $\mathcal U_t^h\leftarrow[\,]$ for every
$h\in\mathcal H$.
\FOR{each memory entry $m\in\mathcal S_t^l\cup\mathcal S_t^g$ in retained
full-rank order}
    \STATE Preserve the source, full rank, facet scores, and available
    action/graph evidence attached to $m$.
    \IF{$m$ is an abstract global record}
        \STATE Bind its object, destination, tool, and action roles against
        $\boldsymbol q_t$ and $\mathcal A_t$.
        \IF{a required role cannot be grounded or its required action is
        unavailable}
            \STATE Exclude $m$ from the current composition.
        \ENDIF
    \ENDIF
    \IF{$m$ remains eligible after binding and compatibility checks}
        \STATE Determine the compatible role set
        $\mathcal H_m\subseteq\mathcal H$ from the record type and its current
        contextual trigger.
        \STATE Assign current-state facts and graph relations to
        $\mathrm{ground}$; plans, procedures, and preconditions to
        $\mathrm{workflow}$; and failures, repairs, and blocked actions to
        $\mathrm{reflection}$ whenever the corresponding role is in
        $\mathcal H_m$.
    \ENDIF
\ENDFOR
\STATE Remove duplicate support items within each role while retaining the
higher-ranked applicable item and its provenance.
\STATE Resolve action conflicts by discarding unavailable actions and by
letting current-state constraints and concrete local grounding override an
incompatible abstract recommendation.
\STATE Apply the fixed context budget to the remaining support in
deterministic full-rank order.
\STATE Retain the admitted memory-bearing support in the corresponding
$\mathcal U_t^h$ in deterministic rank order.
\STATE \textbf{Return} $\{\mathcal U_t^h\}_{h\in\mathcal H}$.
\end{algorithmic}
\end{algorithm*}

\begin{algorithm*}[!htbp]
\caption{$\operatorname{Render}$: Role-Specific Memory Context Rendering}
\label{alg:render}
\small
\begin{algorithmic}[1]
\STATE \textbf{Input:} Query $\boldsymbol q_t$ and composed supports
$\{\mathcal U_t^h\}_{h\in\mathcal H}$
\STATE \textbf{Return:} Memory context
$\mathcal C_t=[c_t^{\mathrm{ground}},c_t^{\mathrm{workflow}},
c_t^{\mathrm{reflection}}]$
\FOR{each support role
$h\in(\mathrm{ground},\mathrm{workflow},\mathrm{reflection})$}
    \STATE Order $\mathcal U_t^h$ by the deterministic composition order.
    \STATE Convert each selected memory entry into a concise instruction
    using only fields relevant to $\boldsymbol q_t$ and role $h$.
    \STATE For grounded global records, emit the current concrete bindings
    rather than removed scene-specific anchors.
    \STATE Deduplicate equivalent instructions and retain the highest-priority
    applicable instruction for each emitted memory line.
    \STATE Store the resulting text in $c_t^h$; set $c_t^h$ to the empty
    string when $\mathcal U_t^h$ is empty after composition.
\ENDFOR
\STATE If all three rendered fields are empty, return an empty memory
context.
\STATE Set $\mathcal C_t\leftarrow
[c_t^{\mathrm{ground}},c_t^{\mathrm{workflow}},
c_t^{\mathrm{reflection}}]$.
\STATE \textbf{Return} $\mathcal C_t$.
\end{algorithmic}
\end{algorithm*}

\FloatBarrier

\subsubsection{Concrete Retrieval Example}
\label{app:memco-retrieval-example}

Suppose the current task is to clean a bar of soap and place it in a vanity.
The current observation mentions a bathroom basin, and the environment exposes
an admissible action for cleaning the soap with the basin. The local record
in this illustrative example is actionable knowledge in $\mathcal E_i^l$;
the shared record is a promoted workflow from $\mathcal M^g$.
They support the same decision in different ways: the local relation grounds
an executable cleaning action, while the global workflow supplies the
clean-before-place ordering. The global roles are bound to the current soap,
basin, and vanity before use. The rejected kitchen-specific instruction illustrates a grounding conflict,
not an additional observed record in the global snapshot.

\begin{table}[htbp]
\centering
\small
\begin{tabularx}{\linewidth}{@{}>{\raggedright\arraybackslash}p{0.24\linewidth}>{\raggedright\arraybackslash}X@{}}
\toprule
\textbf{Field} & \textbf{Value}\\
\midrule
Current goal & Clean soap and put it in the vanity.\\
Current observation & Soap and a basin are visible in the bathroom.\\
Admissible action & \texttt{clean soap with basin}.\\
\midrule
Local type and record
& Actionable knowledge, $\mathcal E_i^l$: the basin can clean the soap;
the action relation matches the currently admissible cleaning action.\\
Local use
& Current-state and action compatibility support $c_t^{\mathrm{ground}}$.\\
\midrule
Global type and record
& Promoted workflow, $\mathcal M^g$: for a cleanable target,
find a water source; clean \texttt{target\_object}; place it in
\texttt{goal\_destination}. Outcome: a stored clean target.\\
Current role bindings
& \texttt{target\_object} $\mapsto$ soap; water source $\mapsto$ basin;
\texttt{goal\_destination} $\mapsto$ vanity.\\
Global use
& Goal-role and task-progress compatibility support
$c_t^{\mathrm{workflow}}$; the ordering does not replace local grounding.\\
\midrule
Rejected memory entry & \texttt{go to kitchen sink}.\\
Rejection reason
& The concrete kitchen location cannot be bound to the current bathroom
observation and available actions; the instruction is not used.\\
\bottomrule
\end{tabularx}
\caption{Schematic typed retrieval and composition example. Local actionable
knowledge supplies executable grounding, while a global workflow supplies
action order. Source, record type, and support role are shown separately.}
\label{tab:memco-retrieval-example}
\label{tab:retrieval-example}
\end{table}

\FloatBarrier
\subsection{Reported Configuration and Scope of the Details}
\label{app:memco-configuration}
The reported configuration uses $\alpha=0.05$, $\lambda=0.35$, retrieval caps
$k^l=k^g=5$, and unit outer weights for the three relevance facets. These are
the paper-level parameters governing promotion and bounded retrieval. Across
the evaluated benchmarks, the same paper-level pipeline is used for local
memory construction, abstraction and promotion, query-guided retrieval, and
composition. Environment interfaces provide the normalized state, action,
feedback, goal, progress, and supported task-role fields consumed by this
pipeline; benchmark-specific serialization details are not additional
components of the method. The paper-level memory context is the three-role
$\mathcal C_t$ in Eq.~\ref{eq:adaptive-context}; surrounding task-wrapper
context shown in Appendix~\ref{app:memco-prompts} is kept separate from this
local/global retrieval interface.

\FloatBarrier
\section{Extended Experimental Results}
\label{appendix:extended-results}

This section provides the numerical results underlying the mechanism
and sensitivity analyses in
Figs.~\ref{fig:memory-cost-budget-sensitivity},
\ref{fig:lambda-component-ablation}, and
\ref{fig:relevance-retrieval-ablation}.
Table~\ref{tab:12} reports the Wilson-threshold sweep,
Table~\ref{tab:component-ablationapp} combines the memory-scope,
promotion, composition, and memory-cost results,
Table~\ref{tab:app-budget-sensitivity} reports interaction-budget
sensitivity, and Table~\ref{tab:app-relevance-retrieval} gives the
absolute relevance-scoring and retrieval results.

Unless otherwise specified, the component, promotion, and retrieval
variants reuse the same frozen local-memory snapshot and differ only in
the evaluated component. These are controlled within-snapshot
comparisons rather than independent repetitions of the main experiments.
The interaction-budget sweep follows the budget-specific evaluation
protocol used in Fig.~\ref{fig:memory-cost-budget-sensitivity}(d).
Accuracy and success rates are reported in percent, and Steps denotes the
average number of environment interactions. Familiar and Held-out refer
to ALFWorld's official \texttt{valid\_seen} and
\texttt{valid\_unseen} layout splits, respectively. Collectively, the
controlled studies examine the three mechanisms emphasized by the method:
complementary local/global memory scope, evidence-aware global promotion, and
query-guided use of memory for the current decision.

\subsection{Experimental Reproduction Boundaries}
\label{app:memco-experimental-boundaries}

The main results and the controlled component studies serve different purposes.
Table~\ref{tab:memory-comparison-simplified} reports mean $\pm$ standard
deviation over three runs. The component, promotion, and retrieval studies in
this appendix instead reuse a fixed local-memory snapshot so that the evaluated
mechanism can be changed without recollecting experience. For threshold
sensitivity, only the promoted global snapshot is reconstructed from the same
local evidence; for retrieval and composition studies, the underlying memories
remain fixed.

The main experiments use the shared setting $\lambda=0.35$ across benchmarks,
while Table~\ref{tab:12} reports a controlled sensitivity sweep around this
operating point. The cross-setting significance analysis in
Appendix~\ref{appendix:statistical-significance} uses the displayed main-result
means as matched backbone--task blocks; it does not treat the repeated runs as
additional independent blocks.

\subsection{Sensitivity to the Wilson Promotion Threshold}
\label{app:memco-threshold-results}
Table~\ref{tab:12} reconstructs global memory at each threshold
while reusing the same local evidence. Increasing $\lambda$ from $0.25$ to
$0.50$ reduces the global bank from 735 to 294 records. Over this range, familiar-layout
accuracy remains between $96.10\%$ and $97.40\%$, and held-out-layout accuracy between
$97.50\%$ and $97.86\%$. The reported setting $\lambda=0.35$ uses 426 records,
$42.0\%$ fewer than $\lambda=0.25$, while attaining the highest familiar-layout accuracy
and tying the highest held-out-layout accuracy. The lower threshold nevertheless has
the fewest steps on both splits. The table therefore characterizes the
accuracy--compactness--efficiency trade-off rather than establishing that one
threshold optimizes every metric.

\begin{table*}[!htbp]
  \centering
  
  \small
  \caption{Sensitivity to the Wilson promotion threshold on
  ALFWorld with Qwen3-32B.
  All thresholds use the same frozen local-memory snapshot; only global
  memory is reconstructed.}
  \label{tab:12}
  \renewcommand{\arraystretch}{1}
  \setlength{\tabcolsep}{4pt}
  \begin{tabular}{@{}cc
  cccc@{}}
    \toprule
    \boldmath$\lambda$
      & \shortstack{\textbf{Global}\\\textbf{records}}
      & \shortstack{\textbf{Familiar}\\\textbf{Acc. (\%) $\uparrow$}}
      & \shortstack{\textbf{Held-out}\\\textbf{Acc. (\%) $\uparrow$}}
      & \shortstack{\textbf{Familiar}\\\textbf{Steps $\downarrow$}}
      & \shortstack{\textbf{Held-out}\\\textbf{Steps $\downarrow$}} \\
    \midrule
0.25 & 735 & \underline{96.43} & \textbf{97.86}
     & \textbf{9.63} & \textbf{11.34} \\
0.35 & 426 & \textbf{97.40} & \textbf{97.86}
     & \underline{9.74} & \underline{11.42} \\
0.50 & 294 & 96.10 & \underline{97.50}
     & 9.85 & 11.49 \\
    \bottomrule
  \end{tabular}
\end{table*}

\FloatBarrier
\subsection{Memory Scope, Promotion, and Query-Guided Composition}
\label{app:memco-component-results}
Table~\ref{tab:component-ablationapp} separates the roles of the two memory
sources, the promotion criterion, and the way retrieved records are used.
Local-only retains strong familiar-layout accuracy but has a larger familiar--held-out layout gap;
Global-only improves held-out-layout accuracy from $89.64\%$ to $97.14\%$ relative to
Local-only. The complete system improves on both isolated sources and uses
fewer steps on both splits. Simply concatenating local and global records
does not recover this benefit: relative to fixed L5+G5, the full method gains
10.71 and 16.07 percentage points on familiar- and held-out-layout accuracy, respectively.

The empirical-promotion row changes the admission score from $W(\bar p)$
to $\rho^+(\bar p)$ at the same $\lambda=0.35$, leaving the structural
eligibility checks and query-guided retrieval policy unchanged.
Wilson promotion improves accuracy by 2.27 and 1.43 percentage points and
reduces steps by 0.82 and 0.54 on the familiar-layout and held-out-layout splits, respectively.
Together with Figure~\ref{fig:memory-cost-budget-sensitivity}, this comparison
separates the quality of the promoted bank from merely retrieving more
records. These within-snapshot results support the contribution of each
component without attributing every change to a single underlying cause.

\begin{table*}[h]
  \centering
  
  \caption{\textbf{Ablation of memory organization, global promotion,
  and query-guided composition on ALFWorld with Qwen3-32B.}
  All variants use the same frozen local-memory snapshot.
  ``Global records'' denotes the size of promoted global memory.
  ``Avg.\ \#Mem'' and ``Avg.\ Mem.\ Tok.'' denote the average number of
  entries and memory-prompt tokens injected per retrieval call,
  respectively, averaged over the Familiar and Held-out splits.
  A dash indicates that no global memory is used.}
  \label{tab:component-ablationapp}

  \scriptsize
  \setlength{\tabcolsep}{3pt}
  \renewcommand{\arraystretch}{1.08}

  \resizebox{0.99\textwidth}{!}{
  \begin{tabular}{@{}lrrrrrrr@{}}
    \toprule
    \multirow{2}{*}{\textbf{Setting}}
      & \multirow{2}{*}{\textbf{Global records}}
      & \multicolumn{2}{c}{\textbf{Accuracy (\%) $\uparrow$}}
      & \multicolumn{2}{c}{\textbf{Steps $\downarrow$}}
      & \multirow{2}{*}{\textbf{Avg.\ \#Mem $\downarrow$}}
      & \multirow{2}{*}{\textbf{Avg.\ Mem.\ Tok. $\downarrow$}} \\
    \cmidrule(lr){3-4}
    \cmidrule(lr){5-6}
      & & \textbf{Familiar} & \textbf{Held-out}
      & \textbf{Familiar} & \textbf{Held-out}
      & & \\
    \midrule

    \multicolumn{8}{@{}l}{\textit{Memory scope and composition policy}} \\
    \addlinespace[1pt]

    Local-only
      & -- & 96.43 & 89.64 & 11.81 & 14.87 & 0.70 & 72.96 \\

    Global-only
      & 426 & 96.10 & 97.14 & 10.31 & 11.57 & 1.70 & 124.21 \\

    Local+Global, fixed concat (L3+G3)
      & 426 & 83.12 & 79.29 & 16.12 & 17.38 & 6.00 & 146.58 \\

    Local+Global, fixed concat (L5+G5)
      & 426 & 86.69 & 81.79 & 15.84 & 16.32 & 9.99 & 207.56 \\

    \addlinespace[2pt]
    \multicolumn{8}{@{}l}{\textit{Global-memory promotion criterion}} \\
    \addlinespace[1pt]

    Empirical support ratio $\rho^{+}$
      & 736 & 95.13 & 96.43 & 10.56 & 11.96 & 1.73 & 124.56 \\

    \midrule

    \textbf{MemCo (Wilson + query-guided)}
      & 426 & 97.40 & 97.86 & 9.74 & 11.42 & 1.70 & 123.23 \\

    \bottomrule
  \end{tabular}}
\end{table*}

\FloatBarrier

\subsection{Interaction-Budget Sensitivity}
\label{app:budget-sensitivity}

Table~\ref{tab:app-budget-sensitivity} reports the exact success rates
underlying Fig.~\ref{fig:memory-cost-budget-sensitivity}(d).
This experiment varies only the maximum interaction budget under its
budget-specific evaluation protocol.

\begin{table}[h]
  \centering
  \caption{\textbf{Sensitivity to the maximum interaction budget on
  ALFWorld with Qwen3-32B.}
  Success rates are reported for the Familiar and Held-out layout splits.}
  \label{tab:app-budget-sensitivity}

  \small
  \setlength{\tabcolsep}{6pt}
  \renewcommand{\arraystretch}{1.08}

  \begin{tabular}{llrr}
    \toprule
    \textbf{Method}
      & \textbf{Budget}
      & \textbf{Familiar}
      & \textbf{Held-out} \\
    \midrule

    \multirow{3}{*}{ICL}
      & 10 & 36.36 & 28.57 \\
      & 30 & 79.72 & 79.69 \\
      & 50 & 92.21 & 95.71 \\

    \addlinespace[2pt]

    \multirow{3}{*}{MemCo}
      & 10 & 65.58 & 51.43 \\
      & 30 & 96.10 & 97.50 \\
      & 50 & 96.43 & 100.00 \\

    \bottomrule
  \end{tabular}
\end{table}

\subsection{Relevance Scoring and Retrieval}
\label{app:relevance-retrieval-results}

Table~\ref{tab:app-relevance-retrieval} reports the absolute results
underlying Fig.~\ref{fig:relevance-retrieval-ablation}.
The main-text figure visualizes signed changes from the full MemCo row,
whereas the table reports the corresponding absolute accuracy and
trajectory length.

\begin{table*}[h]
  \centering
  \caption{\textbf{Absolute results for relevance-scoring and retrieval
  ablations on ALFWorld with Qwen3-32B.}
  In the component rows, $-r_i$ removes term $r_i$, whereas an isolated
  $r_i$ retains only that term. The all-zero control is included for
  completeness but omitted from the main-text visualization.}
  \label{tab:app-relevance-retrieval}

  \small
  \setlength{\tabcolsep}{7pt}
  \renewcommand{\arraystretch}{1.08}

  \begin{tabular}{@{}lrrrr@{}}
    \toprule
    \multirow{2}{*}{\textbf{Setting}}
      & \multicolumn{2}{c}{\textbf{Accuracy (\%) $\uparrow$}}
      & \multicolumn{2}{c}{\textbf{Steps $\downarrow$}} \\
    \cmidrule(lr){2-3}
    \cmidrule(lr){4-5}
      & \textbf{Familiar}
      & \textbf{Held-out}
      & \textbf{Familiar}
      & \textbf{Held-out} \\
    \midrule

    \multicolumn{5}{@{}l}{\textit{Reference configuration}} \\
    \textbf{Full MemCo} $(r_x+r_p+r_g)$
      & 97.40 & 97.86 & 9.74 & 11.42 \\

    \addlinespace[2pt]
    \multicolumn{5}{@{}l}{\textit{Remove one relevance term}} \\
    $-r_x$ (retain $r_p+r_g$)
      & 96.75 & 97.14 & 10.37 & 11.92 \\
    $-r_p$ (retain $r_x+r_g$)
      & 96.43 & 97.14 & 10.42 & 11.69 \\
    $-r_g$ (retain $r_x+r_p$)
      & 96.75 & 97.14 & 10.42 & 11.92 \\

    \addlinespace[2pt]
    \multicolumn{5}{@{}l}{\textit{Single-term relevance}} \\
    $r_x$ only
      & 96.43 & 97.50 & 10.30 & 11.74 \\
    $r_p$ only
      & 97.08 & 97.50 & 10.36 & 11.98 \\
    $r_g$ only
      & 96.75 & 95.71 & 10.39 & 11.95 \\

    \addlinespace[2pt]
    \multicolumn{5}{@{}l}{\textit{Additional control}} \\
    $r=0$
      & 96.43 & 97.14 & 10.37 & 11.96 \\

    \addlinespace[2pt]
    \multicolumn{5}{@{}l}{\textit{Alternative retrieval strategy}} \\
    Cosine similarity
      & 96.10 & 97.14 & 10.50 & 11.80 \\
    BM25
      & 95.78 & 97.50 & 10.53 & 11.75 \\

    \bottomrule
  \end{tabular}
\end{table*}

\FloatBarrier
\section{Theoretical Analysis}
\label{app:memco-theory}

This section analyzes two properties of \modelname's global promotion mechanism.
We first prove the finite-evidence correction stated in
Theorem~\ref{thm:wilson-error-bound} and derive sharper asymptotic and
threshold implications. We then characterize the stability of the
episode-level evidence aggregation used before Wilson promotion.
Throughout this section, $n^+$ and $n^-$ denote the episode-reduced
promotion counts defined in Appendix~\ref{app:memco-evidence-counts}, rather than the
occurrence-level statistics used for retrieval. Wilson promotion uses only
these signed episode-reduced counts through
$s=n^++n^-$ and $\rho^+=n^+/s$ when $s>0$; neither $c^\omega$ nor the
diagnostic $n^0$ enters the Wilson score.

\subsection{Wilson Finite-Evidence Correction}
\label{app:wilson-analysis}
\begin{proof}[Proof of Theorem~\ref{thm:wilson-error-bound}]
For brevity, let
$s=s(\bar p)$, $\rho=\rho^+(\bar p)$, and $W=W(\bar p)$.
The Wilson lower bound in Eq.~\ref{eq:wilson} is equivalently the
smaller root of
\begin{equation}
(s+z^2)W^2-(2s\rho+z^2)W+s\rho^2=0.
\label{app_eq:wilson-quadratic}
\end{equation}
Indeed, solving Eq.~\ref{app_eq:wilson-quadratic} for $W$ recovers
exactly the closed form in Eq.~\ref{eq:wilson}.

Rearranging Eq.~\ref{app_eq:wilson-quadratic} gives
\begin{equation}
s(\rho-W)^2
=
z^2W(1-W).
\label{app_eq:wilson-score-identity}
\end{equation}
Since $W$ is the lower endpoint and therefore $W\leq\rho$,
\begin{equation}
\epsilon_W(\bar p)
=
\rho^+(\bar p)-W(\bar p)
=
z\sqrt{
\frac{W(\bar p)(1-W(\bar p))}
{s(\bar p)}
}.
\label{app_eq:wilson-identity}
\end{equation}
For $\rho^+(\bar p)=0$, Eq.~\ref{eq:wilson} gives
$W(\bar p)=0$, so the same identity holds trivially.

Finally, since $0\leq W(\bar p)\leq1$ and
$W(\bar p)(1-W(\bar p))\leq1/4$,
\begin{equation}
0\leq\epsilon_W(\bar p)
\leq
\frac{z}{2\sqrt{s(\bar p)}}.
\label{app_eq:wilson-bound}
\end{equation}
The bound is independent of $\rho^+(\bar p)$, establishing the uniform
$O(s(\bar p)^{-1/2})$ correction stated in
Theorem~\ref{thm:wilson-error-bound}.
\end{proof}

\begin{corollary}[Exact promotion threshold and required evidence]
\label{app_cor:wilson-threshold}
For $s=s(\bar p)>0$, $z>0$, and $0<\lambda<1$, the Wilson gate satisfies
\begin{equation}
W(\bar p)\geq\lambda
\quad\Longleftrightarrow\quad
\rho^+(\bar p)\geq
\lambda+z\sqrt{\frac{\lambda(1-\lambda)}{s(\bar p)}}.
\label{app_eq:wilson-threshold}
\end{equation}
Equivalently, for integer signed evidence counts,
\begin{equation}
n^+(\bar p)\geq
\left\lceil s(\bar p)\lambda+
z\sqrt{s(\bar p)\lambda(1-\lambda)}\right\rceil.
\label{app_eq:wilson-count-threshold}
\end{equation}
For a fixed empirical rate $\rho^+(\bar p)>\lambda$, this is equivalent to
\begin{equation}
s(\bar p)\geq
\frac{z^2\lambda(1-\lambda)}{(\rho^+(\bar p)-\lambda)^2}.
\label{app_eq:wilson-sample-threshold}
\end{equation}
No finite positive evidence size passes when $\rho^+(\bar p)\leq\lambda$.
At the all-supporting boundary $n^-(\bar p)=0$, the minimum integer
evidence size is
\begin{equation}
s_{\min}=\left\lceil\frac{z^2\lambda}{1-\lambda}\right\rceil.
\label{app_eq:wilson-all-supporting-minimum}
\end{equation}
Below this size, no supporting rate can pass the Wilson gate. These
conditions concern the Wilson gate only; structural and information
eligibility are still required for global promotion.
\end{corollary}

\begin{proof}
Write $\rho=\rho^+(\bar p)$ and define
$Q_\rho(w)=s(\rho-w)^2-z^2w(1-w)$.
By Eq.~\ref{app_eq:wilson-quadratic}, its smaller root is $W$.
For $0<\rho<1$, $Q_\rho(\rho)<0$, so $\rho$ lies strictly between
the two roots. Consequently, when $\lambda\leq\rho$, the condition
$W\geq\lambda$ is equivalent to $Q_\rho(\lambda)\geq0$.
Taking the nonnegative square root gives
$\rho-\lambda\geq z\sqrt{\lambda(1-\lambda)/s}$.
When $\lambda>\rho$, both sides of the claimed equivalence are false
because $W\leq\rho$. The boundary $\rho=0$ gives $W=0$; at $\rho=1$,
$W=s/(s+z^2)$ gives the same equivalence directly.
Multiplying by $s$ and using integer $n^+$ proves
Eq.~\ref{app_eq:wilson-count-threshold}; rearranging for $\rho>\lambda$
proves Eq.~\ref{app_eq:wilson-sample-threshold}.
Finally, the right-hand side of Eq.~\ref{app_eq:wilson-threshold} must
not exceed one. This is equivalent to $s\geq z^2\lambda/(1-\lambda)$,
which is attained when all votes are supporting.
\end{proof}

\paragraph{Pointwise refinement.}
The uniform bound in Theorem~\ref{thm:wilson-error-bound} controls the
largest possible finite-evidence correction over supporting rates.
The dependence on a fixed supporting rate can be characterized more
precisely.

Substituting Eq.~\ref{eq:wilson} into
$\epsilon_W(\bar p)=\rho^+(\bar p)-W(\bar p)$ gives the exact closed form
\begin{equation}
\begin{aligned}
\epsilon_W(\bar p)
={}&
\frac{1}{1+z^2/s(\bar p)}
\Biggl[
\frac{z^2}{s(\bar p)}
\left(\rho^+(\bar p)-\frac{1}{2}\right)\\
&\qquad
+z
\left(
\frac{\rho^+(\bar p)(1-\rho^+(\bar p))}
{s(\bar p)}
+
\frac{z^2}{4s(\bar p)^2}
\right)^{1/2}
\Biggr].
\end{aligned}
\label{app_eq:wilson-closed-form}
\end{equation}

For a fixed $0<\rho^+(\bar p)<1$ and increasing $s(\bar p)$,
\begin{equation}
\begin{aligned}
&
\left[
\frac{\rho^+(\bar p)(1-\rho^+(\bar p))}
{s(\bar p)}
+
\frac{z^2}{4s(\bar p)^2}
\right]^{1/2}\\
&\qquad =
\left[
\frac{\rho^+(\bar p)(1-\rho^+(\bar p))}
{s(\bar p)}
\right]^{1/2}
+
O\!\left(s(\bar p)^{-3/2}\right).
\end{aligned}
\label{app_eq:wilson-sqrt-asymptotic}
\end{equation}
Substituting Eq.~\ref{app_eq:wilson-sqrt-asymptotic} into
Eq.~\ref{app_eq:wilson-closed-form}, together with
\[
\left(1+\frac{z^2}{s(\bar p)}\right)^{-1}
=
1+O\!\left(s(\bar p)^{-1}\right),
\]
gives
\begin{equation}
\epsilon_W(\bar p)
=
z\sqrt{
\frac{\rho^+(\bar p)(1-\rho^+(\bar p))}
{s(\bar p)}
}
+
O\!\left(s(\bar p)^{-1}\right).
\label{app_eq:wilson-pointwise}
\end{equation}

At the all-supporting boundary $\rho^+(\bar p)=1$,
Eq.~\ref{eq:wilson} reduces to
\begin{equation}
W(\bar p)
=
\frac{s(\bar p)}{s(\bar p)+z^2},
\qquad
\epsilon_W(\bar p)
=
1-W(\bar p)
=
\frac{z^2}{s(\bar p)+z^2}.
\label{app_eq:wilson-boundary}
\end{equation}
Thus, for fixed $0<\rho^+(\bar p)<1$, the leading correction is
$O(s(\bar p)^{-1/2})$, whereas at the all-supporting boundary it
decreases as $O(s(\bar p)^{-1})$. The
$O(s(\bar p)^{-1/2})$ statement in
Theorem~\ref{thm:wilson-error-bound} is therefore a uniform bound rather
than a claim that every fixed supporting-rate regime has the same
asymptotic behavior.
\paragraph{Interpretation and assumptions.}
Theorem~\ref{thm:wilson-error-bound} establishes an explicit
finite-evidence correction relative to the empirical supporting rate,
while Corollary~\ref{app_cor:wilson-threshold} gives its direct promotion
interpretation: under a fixed threshold, records supported by less
evidence require stronger empirical support.

The identities and bounds above are algebraic properties of the Wilson
score and do not themselves require independent observations.
The nominal one-sided confidence interpretation additionally relies on
the corresponding binomial sampling assumption. Wilson coverage is an
approximation under that model, not an exact finite-sample coverage guarantee. Episode-level reduction prevents repeated annotations within
one episode from directly multiplying the promotion evidence, but does
not establish independence or identical distribution across episodes.

\subsection{Stability of Episode-Level Evidence Aggregation}
\label{app:memco-aggregation-proof}

The preceding analysis characterizes how Wilson promotion treats a fixed pair
of aggregated signed counts. We next consider the stability of the procedure
that constructs those counts.

\begin{proposition}[Order and repetition invariance of episode-reduced aggregation]
\label{prop:aggregation}
Fix the current set of local records, their abstract keys, and the
deterministic pre-promotion checks. Apply the episode reduction in
Appendix~\ref{app:memco-evidence-counts} separately to each record to obtain
$n_i^+(p)$, $n_i^-(p)$, and $n_i^0(p)$, and then aggregate these counts by exact
abstract key as in Eq.~\ref{eq:promotion-count-aggregation}. The resulting
$(n^+,n^-,n^0)$, Wilson score, and promotion decision are invariant to the
processing order of this fixed input. Repeating an annotation already present
for the same record and episode does not change its episode-reduced
contribution.
\end{proposition}

\begin{proof}
For a fixed record and episode, the relation-aware reducer depends on the set
of support, contradiction, and no-progress annotations, not on their
processing order. Repeating an annotation that is already present therefore
does not change the reduced signed count or its diagnostic stall marker.
The subsequent aggregation in
Eq.~\ref{eq:promotion-count-aggregation} is a sum over the fixed
episode-reduced counts associated with each exact abstract key, and is
therefore associative and commutative. Since the abstract key and the
pre-promotion checks are also fixed, the Wilson score and promotion decision
are unchanged.
\end{proof}

\paragraph{Scope.}
Proposition~\ref{prop:aggregation} concerns a fixed set of current local
records. It does not make stale or out-of-order replacement safe in an
asynchronous implementation; a streaming realization must separately ensure
that the retained source state is current. The evaluated snapshot-rebuild
specialization is described in Appendix~\ref{app:memco-merge-example}.
Different intermediate snapshots may affect subsequent trajectories, so
order-dependent acquisition does not constitute fixed input evidence.

Adding previously absent contradictory evidence or a no-progress marker can
legitimately change an episode-reduced count. The repetition guarantee does not
apply to occurrence-level retrieval counters. Finally, neither exact-key
aggregation nor episode reduction establishes independence or identical
distribution of the aggregated evidence; the nominal statistical
interpretation of the Wilson bound therefore remains subject to the
assumptions stated in Appendix~\ref{app:memco-theory}.

\section{Memory Types and Local--Global Examples}
\label{sec:memco_module_examples}

This section complements the implementation details with a unified account of
what each memory component stores, what can be shared, and how the resulting
records are used. The four pattern types belong to $\mathcal P_i^l$; together
with $G_i^l$, $\mathcal E_i^l$, and $\mathcal R_i^l$, they constitute the local
memory in Eq.~\ref{eq:local-memory}. Global memory stores promoted abstract
records from the shareable components, not a duplicate of the entire local
memory.

\subsection{Type Definitions and Local--Global Correspondence}
\label{app:memco-type-correspondence}
\label{sec:abs_exam}

Table~\ref{tab:memco-memory-types} covers every component and pattern subtype.
The examples are schematic adaptations of the household and PDDL examples
used in this appendix, not new benchmark rollouts or literal serialized dumps.
Role-level examples assume that a concrete entity has a supported task-role
mapping; operational conditions are retained. ``Shareable'' in the sharing
column means eligible for consideration, not automatically admitted to global
memory. Concrete PDDL memory snapshots are shown separately in
Appendix~\ref{appendix:pddl_vis}.

\begin{table*}[!htbp]
\centering
\begingroup
\fontsize{8.5}{10.2}\selectfont
\setlength{\tabcolsep}{3pt}
\renewcommand{\arraystretch}{1.12}
\begin{tabularx}{\linewidth}{@{}>{\raggedright\arraybackslash}p{0.13\linewidth}>{\raggedright\arraybackslash}p{0.205\linewidth}>{\raggedright\arraybackslash}p{0.225\linewidth}>{\raggedright\arraybackslash}X>{\raggedright\arraybackslash}p{0.10\linewidth}@{}}
\toprule
\textbf{Component / type} & \textbf{Definition} & \textbf{Local example}
& \textbf{Possible global counterpart} & \textbf{Sharing}\\
\midrule
Graph $G_i^l$
& Paper-level state--action--feedback transition structure with transition
statistics; the serialized graph stores feedback on the corresponding
state--action--next-state transition, while completed episode identifiers are
retained at store level.
& Soap is dirty at the basin $\to$ clean soap with basin $\to$ progress $\to$ soap is clean.
& No shared graph. Separately extracted patterns or actionable relations may be considered.
& Local only.\\[5pt]
Workflow in $\mathcal P_i^l$
& Ordered procedure spanning distinct progress phases, with conditions and an outcome.
& Find a basin; clean the soap; place the clean soap in the vanity.
& For a cleanable target: find water source; clean target; place in goal destination. Outcome: stored clean target.
& Shareable\\[5pt]
Precondition in $\mathcal P_i^l$
& A temporal transition whose observed condition and action are associated with successful progress.
& While carrying dirty soap at a usable basin, cleaning the soap advances the cleaning phase.
& With a dirty target held at a usable cleaning tool, cleaning advances the same phase.
& Shareable\\[5pt]
Repair in $\mathcal P_i^l$
& Successful action that resolves a failure identified in the preceding observation.
& After an \texttt{invalid\_action} observation, \texttt{check valid actions} resolves the recorded failure.
& After \texttt{invalid\_action}, \texttt{check valid actions} restores progress. This instance-free record is unchanged by abstraction.
& Shareable\\[5pt]
Failure in $\mathcal P_i^l$
& Conditioned action--failure relation, usable to avoid repeating a failed action.
& In a recorded local state, \texttt{pickup b1} leads to \texttt{invalid\_action}.
& An eligible anti-pattern can share a conditioned no-progress relation.
The raw failure example itself remains local unless it meets the
type-specific eligibility checks.
& Shareable\\[5pt]
Actionable knowledge $\mathcal E_i^l$
& State--action--outcome relation that provides usable action or object--tool knowledge.
& In the applicable soap/basin state, \texttt{clean soap with basin} produces clean soap.
& Under the retained conditions, a compatible cleaning tool can clean a cleanable target. Bind both roles before use.
& Shareable\\[5pt]
Cases / reflections $\mathcal R_i^l$
& Concrete episode examples, failure diagnoses, and recovery context.
& The observation, failed \texttt{pickup b1}, diagnosis, and subsequent recovery in that episode.
& No shared copy of the case or reflection. A separately extracted repair/failure relation is a different record.
& Local only.\\
\bottomrule
\end{tabularx}
\endgroup
\caption{Memory definitions, schematic examples, and sharing scope.
The global column shows possible abstract records, not guaranteed promotions.
Every shared record remains subject to type-specific eligibility,
aggregated evidence, and the Wilson threshold. Applicability conditions and
current action binding are still required when a promoted record is used.}
\label{tab:memco-memory-types}
\end{table*}

\FloatBarrier

\subsection{From Local Records to Promoted Records}
\label{sec:merge_exam}

The global column in Table~\ref{tab:memco-memory-types} describes possible
abstract counterparts, not automatic promotions or interchangeable copies of
local records. In particular, a supported failure pattern remains a
conditioned failure relation rather than becoming a successful workflow.
Table~\ref{tab:memco-abstraction-example} gives the field-level abstraction;
Tables~\ref{tab:memco-positive-merge} and~\ref{tab:memco-negative-merge}
illustrate exact-key matching. Evidence aggregation and admission are specified
once in Algorithm~\ref{alg:promotion-merge}. Every admitted record must still
be compatible with the receiving agent's current state and action bindings.
\subsection{Memory Source, Record Type, and Support Role}
\label{sec:retrieval_exam}

Three distinctions are important. \emph{Source} specifies whether an item comes
from $\mathcal M_i^l$ or $\mathcal M^g$. \emph{Record type} specifies the stored
content, such as a workflow, repair, or actionable relation. \emph{Support role}
specifies its query-conditioned use in
$\mathcal H=\{\mathrm{ground},\mathrm{workflow},\mathrm{reflection}\}$.
These are not three alternative classifications of the memory banks.

In Table~\ref{tab:memco-retrieval-example}, a local actionable record supplies
grounding and a global workflow supplies action order. A local workflow can
also support procedural guidance, while an applicable global record can
support grounding after role binding. Local or promoted repair/failure records
may supply reflection support when their conditions are triggered; this does
not require sharing the concrete cases in $\mathcal R_i^l$. Selection and use
remain governed by the current state, task phase, goal, and admissible actions,
rather than by a fixed one-to-one assignment from local/global scope to a
support role. This separates reusable experience from the local context needed
to execute it without treating all cross-agent experience as directly
applicable.
\FloatBarrier
\section{Benchmark Domain Construction and Evaluation Scope}
\label{appendix:benchmark}

\subsection{ALFWorld Multi-Domain Environment}
\label{app:memco-alfworld-split}
To evaluate collaboration across environments, ALFWorld tasks are organized
into four functional scene domains: Kitchen, Living Room, Bedroom, and
Bathroom. We sample 100 training tasks from each domain for experience
acquisition. Following the official protocol, evaluation uses both the
\emph{Familiar Layout} (\texttt{valid\_seen}) and \emph{Held-out Layout}
(\texttt{valid\_unseen}) splits. The resulting training/Familiar/Held-out
counts are Kitchen ($100/20/20$), Living Room ($100/17/11$), Bedroom
($100/20/20$), and Bathroom ($100/20/19$).

Following the problem formulation, an \emph{unseen environment} is defined
with respect to the evaluated agent: it is an environment in which that agent
did not acquire experience before evaluation. This agent-specific notion and
the official held-out-layout split provide complementary axes of
generalization. An environment may be unseen to one agent while experience
from other agents is available through global memory; the Held-out Layout split
additionally changes layout familiarity.

\subsection{PDDL Planning and Domain Partitioning}
\label{app:memco-pddl-split}
PDDL tasks are organized into the four planning domains Gripper, Barman,
Blockworld, and Tyreworld. The source sets contain 20, 20, 10, and 10 tasks,
respectively, and evaluation uses a unified test set of 60 tasks. Local memory
is constructed independently for each source domain before cross-agent global
memory is formed.

\subsection{FEVER Domain Partitioning}
\label{app:memco-fever-split}
For FEVER, we construct two semantic source domains, Film/TV and Music, with
100 training examples in each domain. As noted in the AI Use Statement,
GPT-4o-mini is used only to assist the semantic domain labeling used for this
partition construction; it is not an evaluation-time memory component.
Evaluation uses 60 disjoint examples, 30 from each domain, with a balanced
distribution over SUPPORTS, REFUTES, and NEI labels. These partitions provide
distinct sources of interaction experience while keeping the verification task
and action interface unchanged.

\subsection{ScienceWorld Task-Family Partitioning}
\label{app:memco-scienceworld-split}
ScienceWorld tasks are organized by their official task families while
preserving the official train/test variation split. We sample up to 10 training
variations per task template for local experience acquisition and up to 2 test
variations for evaluation, yielding 60 held-out evaluation variations in the
merged test set. Each local domain is evaluated with the same local-plus-global
memory protocol.

\subsection{Acquisition, Evaluation, and Controlled Comparisons}
\label{app:memco-evaluation-scope}
During acquisition, agents interact with their designated source environments,
update their own local memories, and contribute selected shareable records to
global promotion. Global memory is constructed only from the completed local
stores. During evaluation, local and global memories are frozen: the evaluated
agent refreshes its query and retrieved support at every decision step, but the
test trajectory is not written back into either memory.

For each source-domain local memory, evaluation is performed on the common
merged test set together with the shared global memory. Consequently, the same
evaluated agent encounters both tasks aligned with its own acquisition domain
and tasks from environments in which it did not acquire local experience. This
operationalizes the agent-specific unseen-environment setting in the problem
formulation while allowing the shared global memory to supply experience
contributed by other agents. In ALFWorld, Familiar and Held-out additionally
measure whether the layout itself was seen in the benchmark training split.

The controlled ablations in Appendix~\ref{appendix:extended-results} reuse the
same frozen local-memory snapshot. Threshold sensitivity reconstructs only the
global memory from the same evidence, while the other component studies change
the evaluated memory scope, promotion rule, retrieval, or composition policy
without recollecting source experience.

\FloatBarrier
\section{\modelname{} Per-Step Prompt Template}
\label{app:memco-prompts}

The displayed \texttt{Local memory}, \texttt{Global memory}, and
\texttt{Failure memory} lines are rendering labels for the grounding,
workflow, and reflection fields of Eq.~\ref{eq:adaptive-context}, respectively;
they do not define three independently stored memory banks. Record provenance
is retained during composition, and an applicable global record may also
contribute to grounding after current-state role binding
(Appendix~\ref{sec:retrieval_exam}).
The displayed prompt also contains reference-case, trajectory, and insight
slots outside the Memory Decision Summary, but the formatter emits an optional
slot only when its list is nonempty. In the implemented \modelname payload,
\texttt{reference\_cases}, \texttt{insights}, \texttt{action\_constraints},
and \texttt{repair\_hints} are empty. The \modelname summary is supplied as at most
one \texttt{planner\_notes} entry after query-guided composition and
rendering. \texttt{execution\_patterns} is an optional wrapper slot: it is
empty for ALFWorld, PDDL, and FEVER, while in ScienceWorld it may contain at
most two score-improving fragments retrieved from committed local histories.
Separately, the task wrapper supplies benchmark-authored few-shot reference
cases from the task configuration and currently truncates that list to its
first case. These surrounding wrapper slots are outside the three-field
Memory Decision Summary in Eq.~\ref{eq:adaptive-context}.

\begin{tcolorbox}[skillbox,title={\modelname{} Per-Step Prompt Template}]
\small
\ttfamily
\textbf{System Prompt}\\
{[ENVIRONMENT-SPECIFIC SYSTEM PROMPT]}\\
For example, the system prompt specifies the legal action space and output format for the current benchmark: ALFWorld household actions, PDDL operators, FEVER Search/Lookup/Finish actions, or ScienceWorld valid commands.\\[6pt]

\textbf{User Prompt}\\
\#\# Successful Examples (Reference Cases; optional, task-configured)\\
Below are some examples of similar tasks that were successfully completed. Please use these as references to guide your thinking and approach to the current task:\\[2pt]
[FEW-SHOT REFERENCE CASES]\\[6pt]

\#\# Your Own Past Successes (Execution Patterns)\\
Here are examples of successful execution processes you've previously used on similar tasks. Pay special attention to the step-by-step procedures and strategies, especially when encountering obstacles:\\[2pt]
[RETRIEVED SUCCESSFUL TRAJECTORIES]\\[6pt]

\#\# Key Insights from Related Tasks (empty in the \modelname payload)\\
The following are insights gathered during the execution of similar tasks. You may refer to them during your task execution to improve problem-solving accuracy.\\[2pt]
[RETRIEVED RULES / INSIGHTS]\\
{[PLANNER NOTES: AT MOST ONE \modelname SUMMARY]}\\
{[ACTION CONSTRAINTS: EMPTY FOR \modelname]}\\
{[REPAIR HINTS: EMPTY FOR \modelname]}\\[6pt]

\#\# How to Use Your Intermediate Findings\\
{[ENVIRONMENT-SPECIFIC MEMORY-USAGE INSTRUCTIONS]}\\
For example, treat retrieved memories as planning hints rather than exact action sources; prefer currently valid actions; avoid copying stale object names, operator arguments, evidence labels, or unavailable commands; and do not repeat actions that failed or made no progress.\\[6pt]

\#\#\# \modelname{} Memory Decision Summary\\
Current phase: [PHASE\_LABEL]\\
Current state: [STATE\_SUMMARY]\\
Local memory: [LOCAL\_MEMORY]\\
Global memory: [GLOBAL\_MEMORY]\\
Failure memory: [FAILURE\_MEMORY]\\
Next priority: [NEXT\_PRIORITY]\\
Confidence / caveat: [CONFIDENCE\_CAVEAT]\\[6pt]

\#\# Your Turn: Take Action!\\
Use the above examples and insights as a foundation, and now work on the following task:\\[2pt]
[TASK DESCRIPTION / QUERY]\\[4pt]

[CURRENT OBSERVATION]\\
{[CURRENT VALID ACTIONS, IF PROVIDED]}\\
{[RECENT ACTION-OBSERVATION HISTORY]}\\[6pt]

Output exactly one valid next action.
\end{tcolorbox}

\FloatBarrier

\section{Cross-Setting Statistical Significance Analysis}
\label{appendix:statistical-significance}

We conduct a complementary non-parametric analysis to assess whether the performance improvements of \modelname{} are consistent across backbone and task settings. Following standard practice for comparing multiple learning methods across matched evaluation settings \cite{demvsar2006statistical}, each backbone--task configuration is treated as one evaluation block. The inputs are the method-wise mean results in Table~\ref{tab:memory-comparison-simplified}. Table~\ref{tab:cross-setting-ranks} summarizes relative standing, Table~\ref{tab:friedman-tests} tests the across-method omnibus null, and Table~\ref{tab:wilcoxon-holm-results} reports the paired comparisons against each baseline. These tables answer different questions and are not additional independent experimental runs.

\subsection{Matched Evaluation Blocks and Average Ranks}
\label{app:memco-statistical-blocks}
\paragraph{Evaluation blocks.}
The analysis contains $N_{\mathrm{blk}}=12$ matched configurations, corresponding to three
LLM backbones and four task settings: FEVER, PDDL, ALFWorld Familiar
Layout, and ALFWorld Held-out Layout. Six methods are compared in every block:
ICL, A-Mem, MemSkill, MemRL, G-Memory, and \modelname{}, giving $K_{\mathrm{meth}}=6$.

Accuracy and interaction steps are analyzed separately. This avoids treating
the two metrics obtained from the same evaluation as independent blocks.
For accuracy, larger values indicate better performance; for interaction
steps, smaller values indicate better performance. Within each block, the
best method receives rank 1, and tied methods receive their average rank.

\paragraph{Average Ranks.}
Let $\operatorname{rank}_{ij}^{(\mu)}$ denote the rank of method $j$ in block $i$ for metric
$\mu\in\{\mathrm{Acc},\mathrm{Steps}\}$. The average rank is

\begin{equation}
\overline{R}_j^{(\mu)}
=
\frac{1}{N_{\mathrm{blk}}}
\sum_{i=1}^{N_{\mathrm{blk}}} \operatorname{rank}_{ij}^{(\mu)}.
\end{equation}

Table~\ref{tab:cross-setting-ranks} reports the resulting average ranks.
\modelname{} obtains the best average rank for both accuracy and interaction
steps. We additionally report the average over all 24 metric-specific
slices as a descriptive summary; this combined value is not used as the
sample size of the inferential tests. In particular, the combined column is
the arithmetic mean of the two metric-specific average ranks; it is not a
new set of 24 independent evaluation blocks.

\begin{table}[h]
\centering
\caption{
Average ranks across the matched backbone--task configurations.
Lower values are better. The final column is descriptive only; statistical
tests are conducted separately for accuracy and steps.
}
\label{tab:cross-setting-ranks}
\small
\setlength{\tabcolsep}{5pt}
\begin{tabular}{lccc}
\toprule
\textbf{Method}
& \textbf{Accuracy}
& \textbf{Steps}
& \textbf{Combined} \\
\midrule
ICL
& 3.750 & 4.333 & 4.042 \\
A-Mem
& 5.333 & 3.833 & 4.583 \\
MemSkill
& 3.583 & 4.000 & 3.792 \\
MemRL
& 4.125 & 4.000 & 4.063 \\
G-Memory
& 2.833 & 3.083 & 2.958 \\
\textbf{\modelname{}}
& \textbf{1.375}
& \textbf{1.750}
& \textbf{1.563} \\
\bottomrule
\end{tabular}
\end{table}

\FloatBarrier
\subsection{Omnibus Tests Across Methods}
\label{app:memco-statistical-omnibus}
For each metric, the Friedman statistic is

\begin{equation}
\chi_{F,\mu}^{2}
=
\frac{12N_{\mathrm{blk}}}{K_{\mathrm{meth}}(K_{\mathrm{meth}}+1)}
\left[
\sum_{j=1}^{K_{\mathrm{meth}}}
\left(\overline{R}_j^{(\mu)}\right)^2
-
\frac{K_{\mathrm{meth}}(K_{\mathrm{meth}}+1)^2}{4}
\right].
\end{equation}

When tied ranks occur, we apply the standard correction

\begin{equation}
C_\mu
=
1-
\frac{
\sum_{i=1}^{N_{\mathrm{blk}}}\sum_g
\left(t_{ig}^{3}-t_{ig}\right)
}{
N_{\mathrm{blk}}(K_{\mathrm{meth}}^3-K_{\mathrm{meth}})
},
\qquad
\widetilde{\chi}_{F,\mu}^{2}
=
\frac{\chi_{F,\mu}^{2}}{C_\mu},
\end{equation}

where $t_{ig}$ denotes the number of tied methods in tie group $g$ of
block $i$. Accuracy contains three pairwise ties, resulting in
$C_{\mathrm{Acc}}=0.9929$. No ties occur for interaction steps, so
$C_{\mathrm{Steps}}=1$.

Because the classical Friedman statistic can be conservative, we use the
Iman--Davenport correction:

\begin{equation}
F_{F,\mu}
=
\frac{
(N_{\mathrm{blk}}-1)\widetilde{\chi}_{F,\mu}^{2}
}{
N_{\mathrm{blk}}(K_{\mathrm{meth}}-1)-\widetilde{\chi}_{F,\mu}^{2}
}.
\end{equation}

Under the omnibus null hypothesis, the Iman--Davenport statistic is
compared with an approximate $F$ reference distribution with $K_{\mathrm{meth}}-1=5$ and
$(K_{\mathrm{meth}}-1)(N_{\mathrm{blk}}-1)=55$ degrees of freedom.

As reported in Table~\ref{tab:friedman-tests}, the omnibus null hypothesis
is rejected for both metrics under this testing model. This omnibus result
indicates an across-method difference, but does not establish which individual
baseline differs from \modelname{}; those comparisons are reported next.

\begin{table}[!htbp]
\centering
\caption{
Friedman tests with the Iman--Davenport correction.
The reported Friedman statistics include the tie correction where
applicable.
}
\label{tab:friedman-tests}
\small
\setlength{\tabcolsep}{5pt}
\begin{tabular}{lcccc}
\toprule
\textbf{Metric}
& \boldmath$\widetilde{\chi}_F^2$
& \boldmath$F_F$
& \textbf{df}
& \boldmath$p$ \\
\midrule
Accuracy
& 30.324
& 11.240
& $(5,55)$
& $1.71\times10^{-7}$ \\
Steps
& 15.571
& 3.855
& $(5,55)$
& $4.58\times10^{-3}$ \\
\bottomrule
\end{tabular}
\end{table}

\FloatBarrier
\subsection{Paired Comparisons, Effect Sizes, and Interpretation}
\label{app:memco-statistical-paired}
After rejecting the omnibus null hypotheses, we compare \modelname{} with
each baseline using exact two-sided Wilcoxon signed-rank tests. Unlike a
comparison based only on average ranks, this paired test uses the direction
and magnitude of the performance differences within each matched
configuration.

For baseline $b$, we define the accuracy difference as

\begin{equation}
\Delta_{i,b}^{\mathrm{Acc}}
=
\mathrm{Acc}_{i,\modelname{}}
-
\mathrm{Acc}_{i,b},
\end{equation}

and the interaction-step difference as

\begin{equation}
\Delta_{i,b}^{\mathrm{Steps}}
=
\mathrm{Steps}_{i,b}
-
\mathrm{Steps}_{i,\modelname{}}.
\end{equation}

Thus, positive differences favor \modelname{} for both metrics. Zero
differences are excluded from the signed-rank statistic. The remaining
absolute differences are ranked, with average ranks assigned to ties.
Let $T_{\mathrm{W}}^{+}$ and $T_{\mathrm{W}}^{-}$ be the sums of ranks associated with positive and
negative differences, respectively. The two-sided test statistic is

\begin{equation}
T_{\mathrm{W}}=\min\{T_{\mathrm{W}}^{+},T_{\mathrm{W}}^{-}\}.
\end{equation}

We calculate the conditional two-sided $p$-values by enumerating all sign
assignments to the observed nonzero absolute-difference ranks under the paired
sign-symmetry null. Here $T_{\mathrm{W}}$ denotes the signed-rank statistic;
$W(\bar p)$ remains the Wilson promotion score. The enumeration is exact for this sign-assignment
model, including averaged tied ranks; it does not itself establish that the
evaluation blocks are independent.

Because the analysis contains five control-versus-baseline comparisons for
each of two metrics, we jointly correct all $M=10$ post-hoc $p$-values using
Holm's step-down procedure \cite{holm1979simple}. Given the ordered raw
$p$-values
$p_{(1)}\leq\cdots\leq p_{(M)}$, the adjusted values are

\begin{equation}
\widetilde{p}_{(i)}
=
\max_{1\leq j\leq i}
\left\{
\min\left[1,(M-j+1)p_{(j)}\right]
\right\}.
\end{equation}

This correction controls the family-wise error rate across all reported
method--metric comparisons.

\paragraph{Effect Size.}
In addition to statistical significance, we report the matched-pairs
rank-biserial correlation

\begin{equation}
r_{\mathrm{rb}}
=
\frac{T_{\mathrm{W}}^{+}-T_{\mathrm{W}}^{-}}
{T_{\mathrm{W}}^{+}+T_{\mathrm{W}}^{-}}.
\end{equation}

Positive values indicate an advantage for \modelname{}, with values closer to
1 indicating greater cross-configuration consistency.

Table~\ref{tab:wilcoxon-holm-results} reports the complete results.
The Win/Loss/Tie column counts the configurations in which \modelname{} is
better, worse, or tied with the corresponding baseline. For example, the
G-Memory comparisons give 10 wins and 2 losses for each metric, with positive
rank-biserial effects and Holm-adjusted $p=0.0146$. The win counts summarize
direction, whereas the effect sizes and signed-rank statistics also depend on
the ranked magnitudes of the paired differences.

\begin{table*}[!htbp]
\centering
\caption{
Exact two-sided Wilcoxon signed-rank comparisons between \modelname{} and
each baseline across the 12 matched backbone--task configurations.
Holm correction is applied jointly to all ten comparisons.
All differences are defined such that positive values favor \modelname{}.
}
\label{tab:wilcoxon-holm-results}
\small
\setlength{\tabcolsep}{5pt}
\begin{tabular}{llccccc}
\toprule
\textbf{Metric}
& \textbf{Baseline}
& \textbf{Win/Loss/Tie}
& \boldmath$T_{\mathrm{W}}$
& \boldmath$r_{\mathrm{rb}}$
& \textbf{Exact $p$}
& \textbf{Holm-adjusted $p$} \\
\midrule

\multirow{5}{*}{Accuracy}
& ICL
& 11/1/0
& 1
& 0.974
& $9.77\times10^{-4}$
& $8.79\times10^{-3}$ \\

& A-Mem
& 12/0/0
& 0
& 1.000
& $4.88\times10^{-4}$
& $4.88\times10^{-3}$ \\

& MemSkill
& 11/1/0
& 2
& 0.949
& $1.46\times10^{-3}$
& $1.03\times10^{-2}$ \\

& MemRL
& 11/0/1
& 0
& 1.000
& $9.77\times10^{-4}$
& $8.79\times10^{-3}$ \\

& G-Memory
& 10/2/0
& 6
& 0.846
& $6.84\times10^{-3}$
& $1.46\times10^{-2}$ \\

\midrule

\multirow{5}{*}{Steps}
& ICL
& 10/2/0
& 4
& 0.897
& $3.42\times10^{-3}$
& $1.46\times10^{-2}$ \\

& A-Mem
& 11/1/0
& 3
& 0.923
& $2.44\times10^{-3}$
& $1.46\times10^{-2}$ \\

& MemSkill
& 11/1/0
& 3
& 0.923
& $2.44\times10^{-3}$
& $1.46\times10^{-2}$ \\

& MemRL
& 9/3/0
& 6
& 0.846
& $6.84\times10^{-3}$
& $1.46\times10^{-2}$ \\

& G-Memory
& 10/2/0
& 5
& 0.872
& $4.88\times10^{-3}$
& $1.46\times10^{-2}$ \\

\bottomrule
\end{tabular}
\end{table*}

\paragraph{Summary.}
Under the stated block-level testing assumptions, all ten paired comparisons
have Holm-adjusted $p<0.05$.
 In particular, the improvements
over G-Memory remain significant for both accuracy
($p_{\mathrm{Holm}}=0.0146$) and interaction steps
($p_{\mathrm{Holm}}=0.0146$). The corresponding rank-biserial correlations
are 0.846 and 0.872, respectively, indicating that the differences are not
only statistically significant but also directionally consistent across the
evaluated configurations.

The unit of analysis in these tests is the matched backbone--task
configuration. The two metrics are tested separately, and the analysis does
not treat the three repeated runs or the two metrics from one configuration
as additional independent evaluation blocks. Accordingly, these results
should be interpreted as evidence of consistent performance differences
across the evaluated configurations rather than as a claim about all
possible models and environments. Configurations reuse backbones and task
families, so dependence between blocks may limit that inferential
interpretation; the reported $p$-values are conditional on the block-level
assumptions and are not a substitute for independent task- or seed-level
replication. The average ranks and paired differences remain descriptive
summaries of the evaluated configurations.

\FloatBarrier

\section{Case Study: Adaptive Memory Support and Cross-Environment Generalization}

\label{appendix:case-study}

\begin{figure*}[t]
  \centering
  \includegraphics[width=1\linewidth]{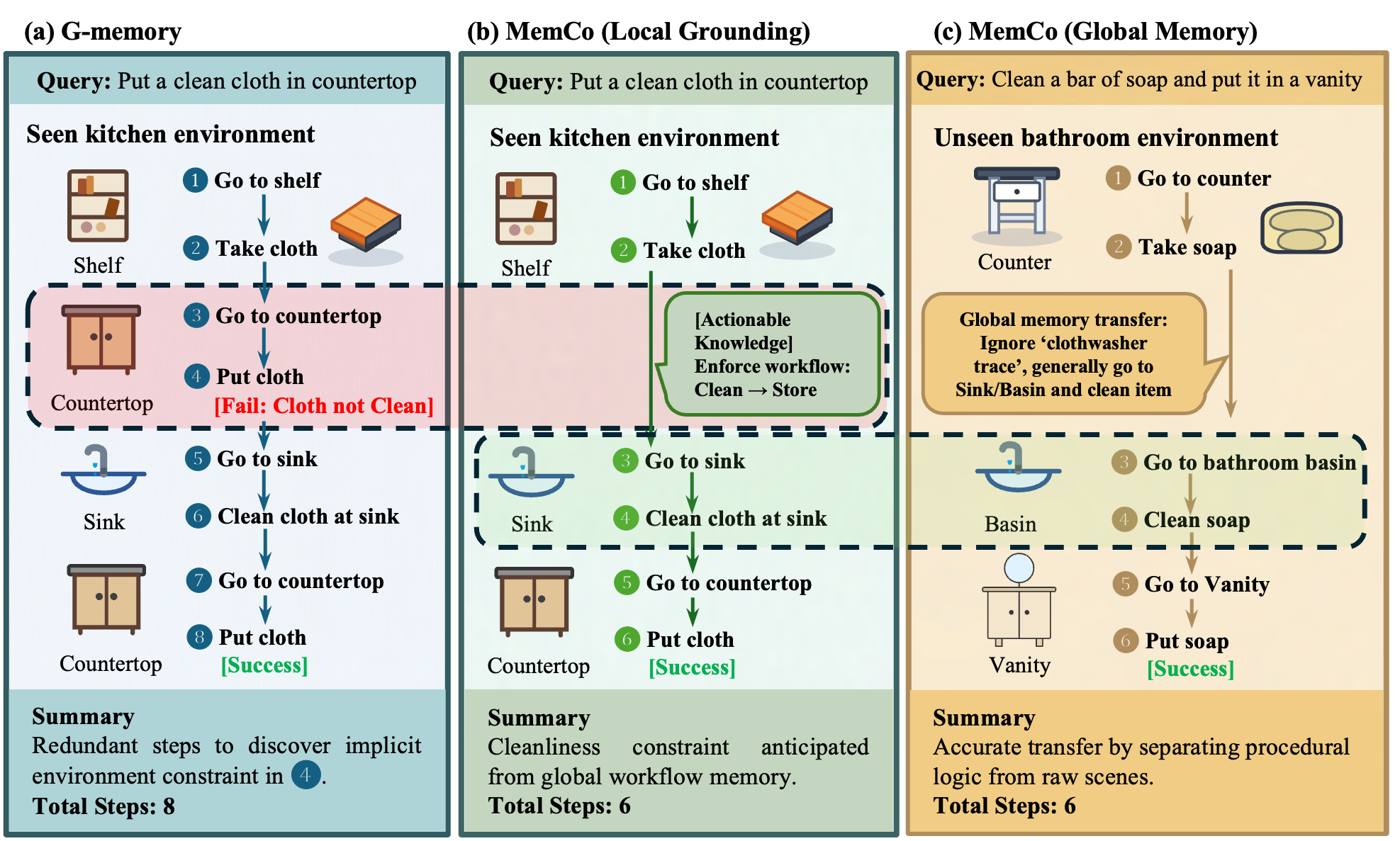}
  \caption{\textbf{Case study showcasing adaptive memory support and cross-environment generalization.} (a) In the shown G-Memory trajectory, the agent attempts placement before satisfying the cleanliness requirement. (b) \modelname{} uses local actionable knowledge to avoid the invalid placement in the same environment. (c) In an unseen environment, \modelname{} combines a reusable global workflow with current-environment grounding.}
  \label{fig:casestudy}
  \vspace{-4mm}
\end{figure*}

To qualitatively evaluate \modelname, we analyze the trajectories across heterogeneous settings to investigate how local situational grounding and global procedural workflows interact dynamically online.

\paragraph{Exploration Efficiency.}
As shown in Fig.~\ref{fig:casestudy}(a) and~\ref{fig:casestudy}(b), when facing the same task ``Put a clean cloth in countertop'', the displayed G-Memory trajectory attempts to place the cloth before satisfying the cleanliness requirement, resulting in redundant interactions. In contrast, \modelname{} retrieves decision-relevant support before placement: the reusable clean-before-store ordering is combined with local actionable knowledge for the current state. The resulting trajectory avoids the invalid move and completes the task with fewer interaction steps.

\paragraph{Unseen-Environment Generalization.}
The comparison between Fig.~\ref{fig:casestudy}(b) and~\ref{fig:casestudy}(c)
illustrates how shared experience supports an agent in an unseen environment.
The task changes from cleaning a cloth in a kitchen to ``Clean a bar of soap
and put it in a vanity'' in a bathroom. A kitchen-specific instruction can
be irrelevant in the bathroom, whereas the clean-before-place ordering can
remain useful under the corresponding applicability conditions.

As illustrated in Fig.~\ref{fig:casestudy}(c), query-guided selection and
composition exclude incompatible kitchen-specific guidance, including the
irrelevant ``clothwasher trace''. The shared workflow supplies the action
order; the current observation and available actions ground its roles in the
bathroom basin and vanity, with applicable local evidence providing further
context. This decision-time grounding does not require writing new
environment-specific records into the frozen evaluation memory. The shown
trajectory completes the task in 6 steps, illustrating how shared procedural
guidance can be transferred to an environment unseen by the evaluated agent
while current-environment grounding determines how that guidance is executed.

\section{Visualizations of PDDL Memory Structures}
\label{appendix:pddl_vis}

To present a clear and concrete view of our memory repository, we visualize the actual global and local memory structures generated during the PDDL symbolic planning tasks. Because these structures are very wide, each visualization and its caption are rotated together and placed on a standalone figure page to use the page height and improve readability. 

\subsection{Global Memory Structure}
\label{app:memco-global-visualization}
Figure~\ref{fig:pddl_global} shows the global memory snapshot assembled
from contributions in gripper, blockworld, barman, and tyreworld. The sharing
interface in Algorithm~\ref{alg:promotion-merge} receives selected reusable
records and their evidence rather than complete local trajectories or graphs.
The visualization contains promoted workflows and precondition relations from
different planning domains, illustrating how recurrent local structures are
represented in the shared global bank. The stored global records remain
conditional and are checked against the current planning state before use.

\begin{figure*}[p]
\centering
\rotatebox{90}{%
\begin{minipage}{0.90\textheight}
\centering
\includegraphics[width=\linewidth,keepaspectratio]{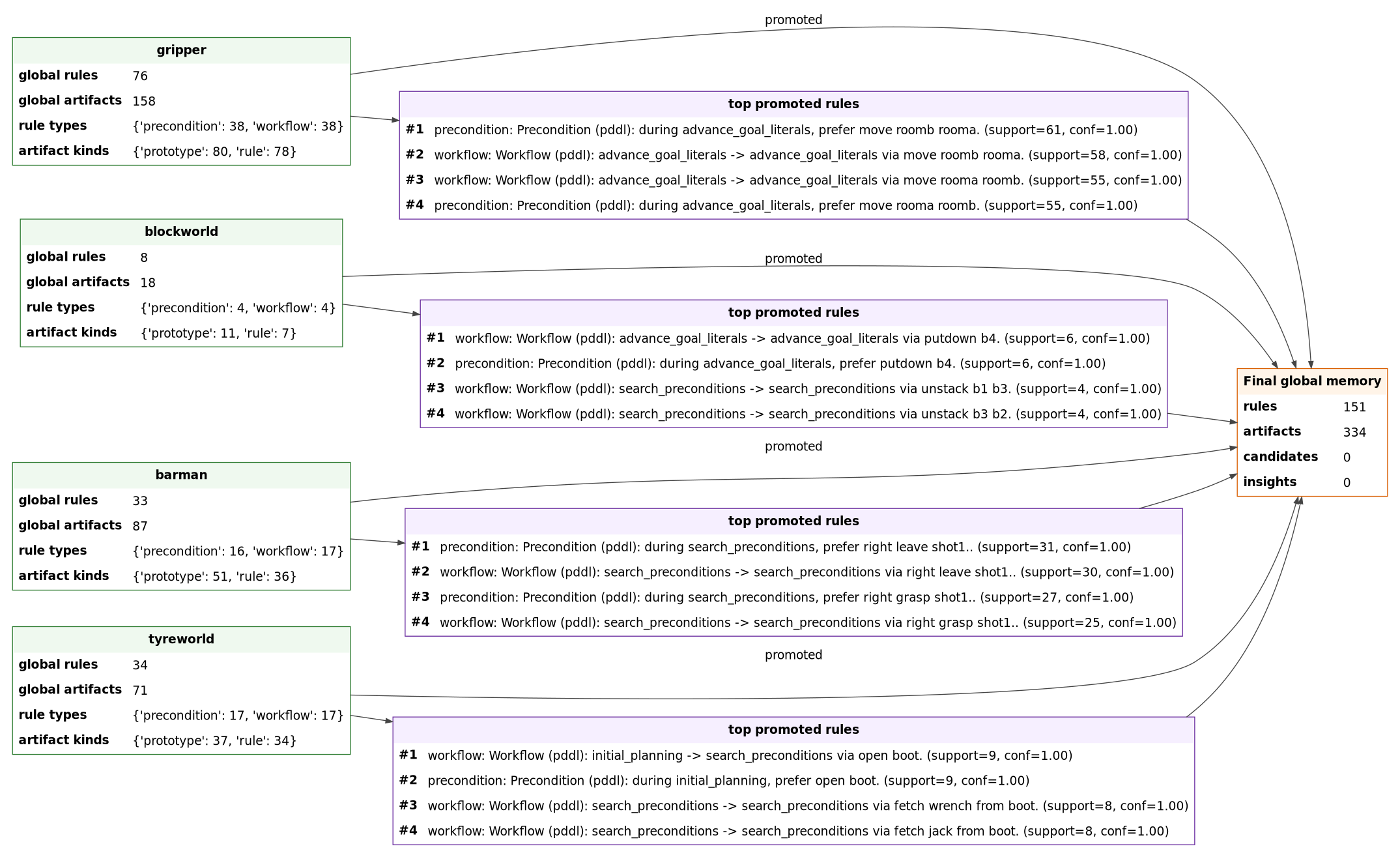}
\caption{\textbf{Visualization of the final PDDL global memory structure.}
The centralized server aggregates top-promoted rules and preconditions across
heterogeneous planning domains into a structured global bank.}
\label{fig:pddl_global}
\end{minipage}%
}
\end{figure*}
\clearpage
\subsection{Local Memory Graphs}
\label{app:memco-local-visualization}
Figures~\ref{fig:pddl_local_tyre} through \ref{fig:pddl_local_block}
display local-memory snapshots across the four PDDL domains, including state
and interaction structure, outcome-annotated records, and sample artifacts.
The connected experience and concrete cases remain local; selected patterns
and actionable relations may be shared through the separate sharing interface.

Figure~\ref{fig:pddl_local_block} provides examples of the stored local
relations, including goal-advancing conditions, failure-triggered repairs, and
failure-avoidance guidance. These entries illustrate the complementary roles
of structured patterns, actionable knowledge, and concrete local context. The
visualization labels reflect the stored instances; the paper-level memory
components remain the $G$--$\mathcal P$--$\mathcal E$--$\mathcal R$
decomposition in Eq.~\ref{eq:local-memory}.

\begin{figure*}[p]
\centering
\rotatebox{90}{%
\begin{minipage}{0.90\textheight}
\centering
\includegraphics[width=\linewidth,keepaspectratio]{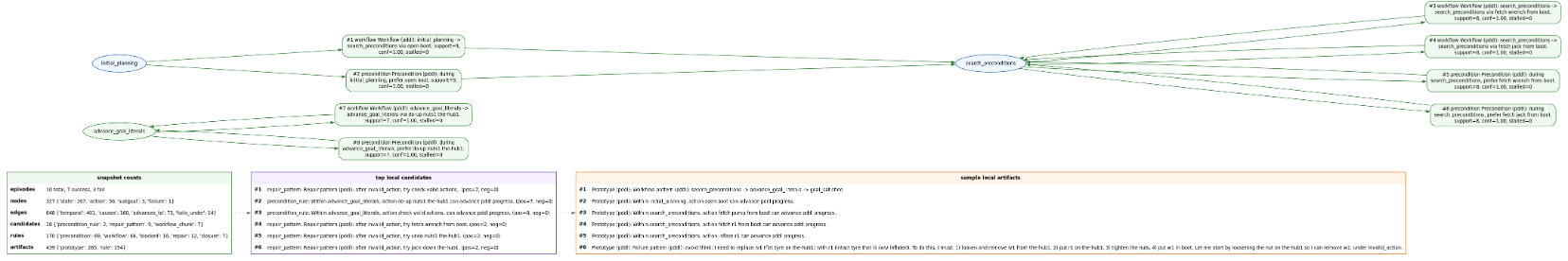}
\caption{\textbf{Local memory structure for the tyreworld domain.} The graph captures real-time planning states and local records for vehicle repair configurations.}
\label{fig:pddl_local_tyre}
\end{minipage}%
}
\end{figure*}
\clearpage

\begin{figure*}[p]
\centering
\rotatebox{90}{%
\begin{minipage}{0.90\textheight}
\centering
\includegraphics[width=\linewidth,keepaspectratio]{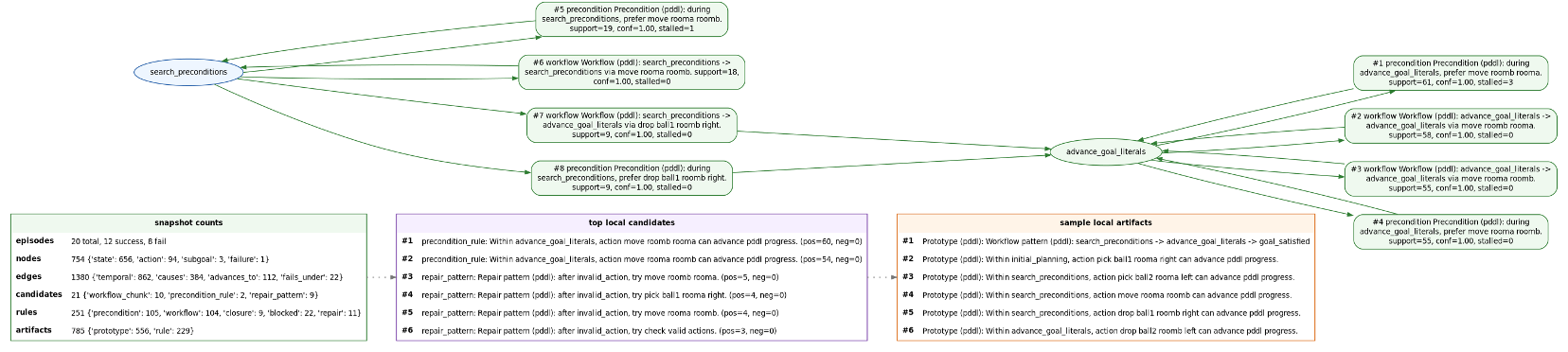}
\caption{\textbf{Local memory structure for the gripper domain.} The graph visualizes execution trajectories and goal-advancing action links for robotic manipulation tasks.}
\label{fig:pddl_local_grip}
\end{minipage}%
}
\end{figure*}
\clearpage
\begin{figure*}[p]
\centering
\rotatebox{90}{%
\begin{minipage}{0.90\textheight}
\centering
\includegraphics[width=\linewidth,keepaspectratio]{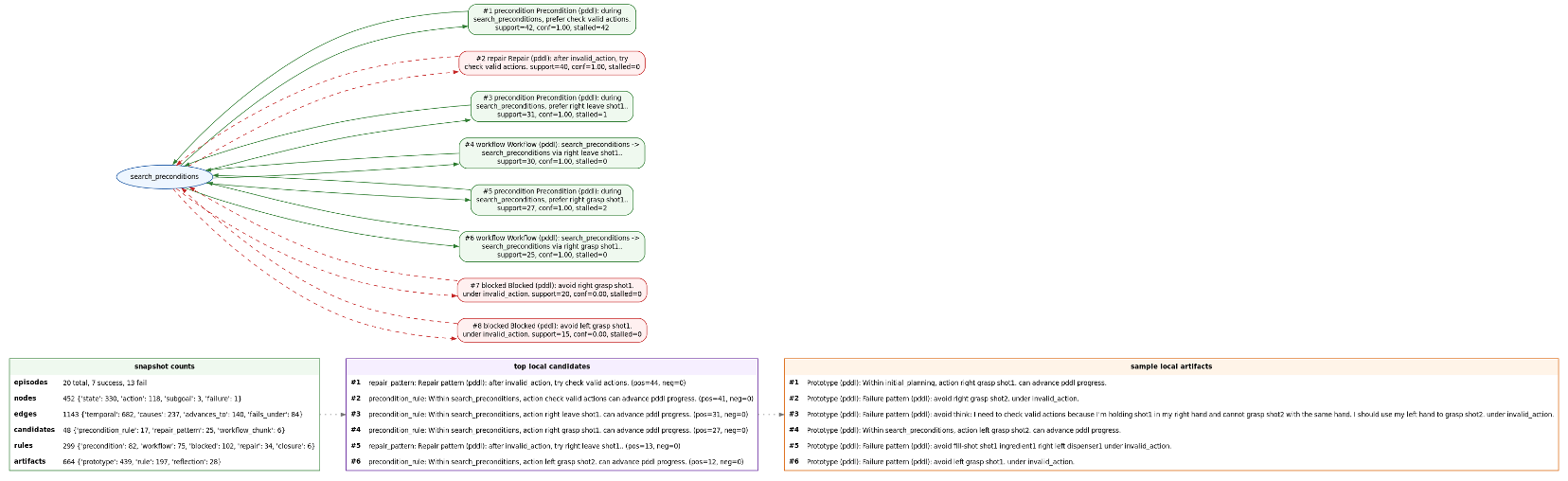}
\caption{\textbf{Local memory structure for the barman domain.} The snapshot logs successful action combinations and filters out invalid interaction paths in a dense action space.}
\label{fig:pddl_local_barman}
\end{minipage}%
}
\end{figure*}
\clearpage

\begin{figure*}[p]
\centering
\rotatebox{90}{%
\begin{minipage}{0.90\textheight}
\centering
\includegraphics[width=\linewidth,keepaspectratio]{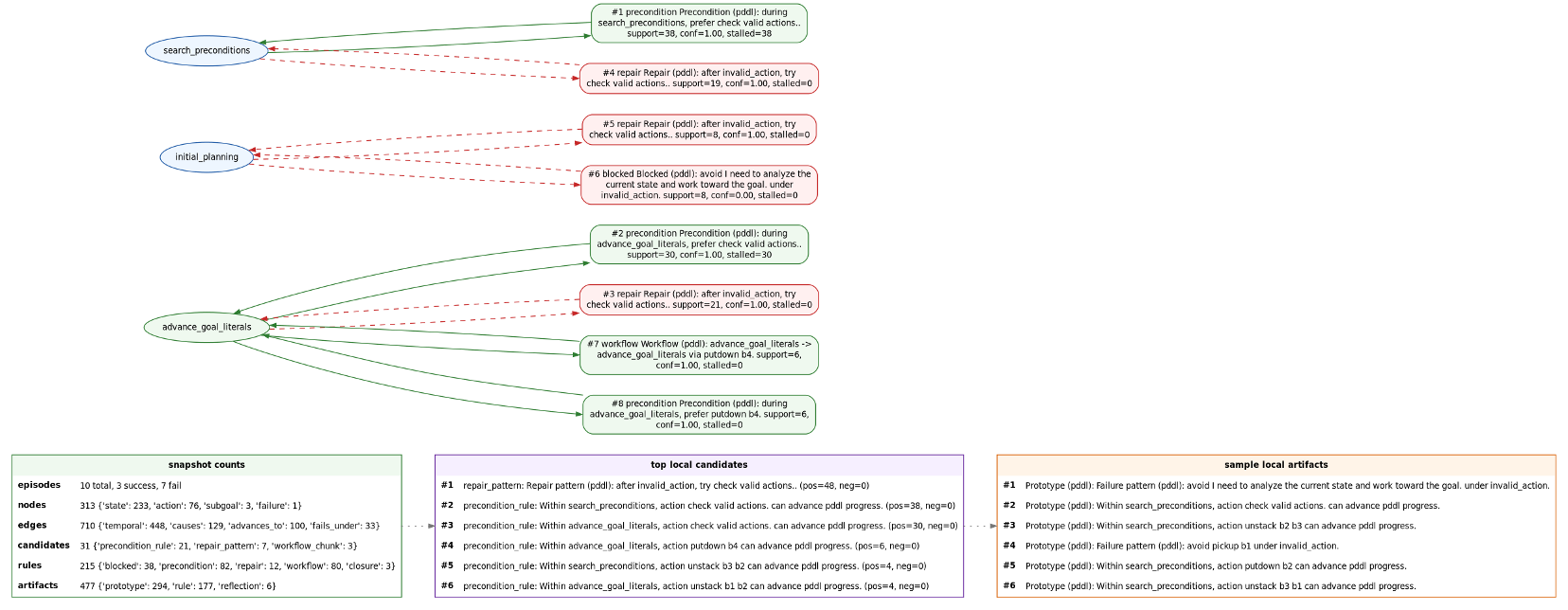}
\caption{\textbf{Local memory structure for the blockworld domain.} The network represents temporal sequence steps and block rearrangement heuristics.}
\label{fig:pddl_local_block}
\end{minipage}%
}
\end{figure*}
\clearpage

\FloatBarrier
\stopcontents[memcoappendix]

\end{document}